\documentclass[twoside,11pt]{article}

\usepackage{jmlr2e}

\usepackage{etoc}
\etocdepthtag.toc{mtchapter}
\etocsettagdepth{mtchapter}{subsubsection}
\etocsettagdepth{mtappendix}{none}

\let\proof\relax

\usepackage{xcolor}
\usepackage[utf8]{inputenc} 
\usepackage[T1]{fontenc}    
\usepackage{hyperref}       
\usepackage{url}            
\usepackage{booktabs}       
\usepackage{amsfonts}       
\usepackage{nicefrac}       
\usepackage{microtype}      
\usepackage{amsthm,amsmath}
\usepackage{comment}
\usepackage{enumerate}
\usepackage{enumitem}
\usepackage{graphicx}
\usepackage{float}
\usepackage{wrapfig}
\usepackage{mathabx}
\usepackage{subcaption}
\usepackage{enumitem}
\usepackage{pifont}
\usepackage{multirow, varwidth}
\usepackage[ruled,vlined]{algorithm2e}
\usepackage{colortbl}
\usepackage{nccmath}
\usepackage{bm}
\makeatletter
\newcommand{\doublewidetilde}[1]{{%
  \mathpalette\double@widetilde{#1}%
}}
\newcommand{\double@widetilde}[2]{%
  \sbox\z@{$\m@th#1\widetilde{#2}$}%
  \ht\z@=.9\ht\z@
  \widetilde{\box\z@}%
}
\makeatother

\definecolor{myred}{HTML}{FFE5D1}
\definecolor{myblue}{HTML}{D4f2f2}

\newcommand{\dE}{{\mathbb{E}}}

\newcommand{\wR}{\widetilde{R}}
\newcommand{\wF}{\widetilde{F}}
\newcommand{\wQ}{\widetilde{Q}}
\newcommand{\wt}{\widetilde{t}}
\newcommand{\wJ}{\widetilde{J}}

\newcommand{\var}{\textup{var}}

\newcommand{\cM}{\mathcal{M}}
\newcommand{\cV}{\mathcal{V}}

\newtheorem*{rep@theorem}{\rep@title}
\newcommand{\newreptheorem}[2]{%
\newenvironment{rep#1}[1]{%
 \def\rep@title{#2 \ref{##1}}%
 \begin{rep@theorem}}%
 {\end{rep@theorem}}}
\newreptheorem{lemma}{Lemma'}
\newreptheorem{definition}{Definition'}
\newreptheorem{proposition}{Proposition'}
\newreptheorem{theorem}{Theorem'}

\makeatletter
\def\renewtheorem#1{%
  \expandafter\let\csname#1\endcsname\relax
  \expandafter\let\csname c@#1\endcsname\relax
  \gdef\renewtheorem@envname{#1}
  \renewtheorem@secpar
}
\def\renewtheorem@secpar{\@ifnextchar[{\renewtheorem@numberedlike}{\renewtheorem@nonumberedlike}}
\def\renewtheorem@numberedlike[#1]#2{\newtheorem{\renewtheorem@envname}[#1]{#2}}
\def\renewtheorem@nonumberedlike#1{  
\def\renewtheorem@caption{#1}
\edef\renewtheorem@nowithin{\noexpand\newtheorem{\renewtheorem@envname}{\renewtheorem@caption}}
\renewtheorem@thirdpar
}
\def\renewtheorem@thirdpar{\@ifnextchar[{\renewtheorem@within}{\renewtheorem@nowithin}}
\def\renewtheorem@within[#1]{\renewtheorem@nowithin[#1]}
\makeatother
\renewtheorem{definition}{Definition}
\newtheorem{assumption}{Assumption}
\renewtheorem{proposition}{Proposition}
\renewtheorem{theorem}{Theorem}
\newtheorem*{theorem*}{Theorem}
\renewtheorem{lemma}{Lemma}
\newtheorem*{lemma*}{Lemma}
\renewtheorem{corollary}[theorem]{Corollary}
\newtheorem{approximation}{Approximation}

\newcommand\new[1]{{\color{black}{#1}}}

\firstpageno{1}

\usepackage{scrextend}
\makeatletter
\def\maketitle{\par
 \begingroup
   \def\thefootnote{\fnsymbol{footnote}}
   \def\@makefnmark{\hbox to 0pt{$^{\@thefnmark}$\hss}}
   \deffootnote[1.7em]{1.6em}{2em}{$^\thefootnotemark$}
   \@maketitle \@thanks
 \endgroup
\setcounter{footnote}{0}
 \let\maketitle\relax \let\@maketitle\relax
 \gdef\@thanks{}\gdef\@author{}\gdef\@title{}\let\thanks\relax}
\makeatother

\usepackage{lastpage}
\jmlrheading{24}{2023}{1-\pageref{LastPage}}{9/21; Revised
5/23}{5/23}{21-1095}{Tian Li$^*$, Ahmad Beirami$^*$, Maziar Sanjabi, Virginia Smith}
\ShortHeadings{On  Tilted Losses in Machine Learning}{Li$^*$, Beirami$^*$, Sanjabi, Smith}

\begin{document}

\title{On Tilted Losses in Machine Learning: \\ Theory and Applications}

\author{\name Tian Li\thanks{Equal contribution.} \email tianli@cmu.edu \\
      \addr Computer Science Department \\
      Carnegie Mellon University \\
      Pittsburgh, PA 15213, USA
      \AND
      \name Ahmad Beirami\footnotemark[1]\hspace{0.07in}\thanks{Work done at Meta AI.} \email beirami@google.com \\
      \addr Google Research \\
      New York, NY 10011, USA
      \AND Maziar Sanjabi \email maziars@fb.com \\
      \addr Meta AI \\
      Menlo Park, CA 94025, USA 
      \AND Virginia Smith \email smithv@cmu.edu \\
      \addr Machine Learning Department \\
      Carnegie Mellon University \\
      Pittsburgh, PA 15213, USA\vspace{-.1in}}


\editor{Zaid Harchaoui}

\maketitle

\begin{abstract}%
Exponential tilting is a technique commonly used in fields such as statistics, probability, information theory, and optimization to create parametric distribution shifts. 
Despite its prevalence in related fields, tilting has not seen widespread use in machine learning. In this work, we aim to bridge this gap by exploring the use of tilting in risk minimization. We study a simple extension to ERM---tilted empirical risk minimization (TERM)---which uses exponential tilting to flexibly tune the impact of individual losses. The resulting framework has several useful properties: We show that TERM can increase or decrease the influence of outliers, respectively, to enable fairness or robustness; has variance-reduction properties that can benefit generalization; and can be viewed as a smooth approximation to {the tail probability of losses}. Our work makes  connections between TERM and related objectives, such as Value-at-Risk, Conditional Value-at-Risk, and distributionally robust optimization (DRO).  We develop batch and stochastic first-order optimization methods for solving TERM, provide convergence guarantees for the solvers, and show that the framework can be efficiently solved relative to common alternatives. Finally, we demonstrate that TERM can be used for a multitude of applications in machine learning, such as enforcing fairness between subgroups, mitigating the effect of outliers, and handling class imbalance. Despite the straightforward modification TERM makes to traditional ERM objectives, we find that the framework can consistently outperform ERM and deliver competitive performance with state-of-the-art, problem-specific approaches. 
\end{abstract}

\vspace{.1in}


\section{Introduction}\label{sec:introduction}
\label{sec:intro}
 
\begingroup
\setlength{\thinmuskip}{0mu}
\setlength{\medmuskip}{0.75mu}
\setlength{\thickmuskip}{0.75mu}
Many statistical estimation procedures rely on the concept of empirical risk minimization (ERM), in which the parameter of interest, 
$\theta \in \Theta \subseteq \mathbb{R}^d$, is estimated by minimizing an average loss over the data $\{x_1, \, \dots, \, x_N\}$:
\endgroup

\vspace{-0.4em}
\begin{equation}\label{eq: ERM}
    \overline{R}(\theta) :=  \frac{1}{N} \sum_{i \in [N]}f(x_i; \theta) \, .
\end{equation} 
\vspace{-0.4em}

Although ERM is widely used in machine learning, it is known to perform poorly in situations where average performance is not an appropriate surrogate for the problem of interest. Significant research has thus been devoted to developing alternatives to traditional ERM for diverse applications, such as learning in the presence of noisy/corrupted data~\citep{khetan2017learning,jiang2018mentornet}, performing classification with imbalanced data~\citep{lin2017focal,malisiewicz2011ensemble}, ensuring that subgroups within a population are treated fairly~\citep{hashimoto2018fairness,samadi2018price}, or developing solutions with favorable out-of-sample performance~\citep{duchi2019variance}.

In this paper, we suggest that deficiencies in ERM can be flexibly addressed via a unified framework, {\em tilted empirical risk minimization (TERM)}. TERM encompasses a family of objectives, parameterized by a  real-valued hyperparameter, $t$. For $t \in \mathbb{R}^{\setminus 0},$ the $t$-tilted loss (TERM objective) is given by: 

\vspace{-0.4em}
\begin{equation}
    \label{eq: TERM}
    \wR(t ;\theta) := \frac{1}{t} \log\bigg( \frac{1}{N}\sum_{i \in [N]}  e^{t f(x_i; \theta)} \bigg) \, .
\end{equation}
\vspace{-0.4em}

TERM generalizes ERM as the $0$-tilted loss recovers the average loss, i.e., $\wR(0, \theta){=} \overline{R}(\theta)$.\footnote{$\wR(0;\theta)$ is defined in~\eqref{eq:def-R0} via the continuous extension of $R(t; \theta)$.} It also recovers other popular alternatives such as the max-loss ($t{\to}{+}\infty$) and min-loss ($t{\to}{-}\infty$) (Lemma~\ref{lemma: special_cases}). As we discuss below, although tilted risk minimization is not widely used in machine learning, variants of tilting have been extensively studied in related fields including statistics, applied probability, optimization, and information theory. 

\subsection{Perspectives on Exponential Tilting} \label{sec:intro:background}
We begin by defining \textit{exponential tilting} and discussing uses of tilting in various fields. Let $\mathcal{P} := \{ p_\theta\}$ be a set of parametric distributions.
For any $x \in \mathcal{X}$, 
we let  $f(x; \theta)$ be the information of $x$ under $\theta$, which is defined as \citep{cover1991information}:
\begin{equation}
\label{eq:information}
    f(x; \theta) := -\log p_\theta(x).
\end{equation}
Further assume that $X$ is a random variable drawn from distribution $p(\cdot),$ which is not necessarily matched to $\cal P$, i.e., the model family may be misspecified.
The cumulant generating function of the information random variable, $f(X; \theta),$ can be stated  as~\cite[Section 2.2]{dembo-zeitouni}:
\begin{equation}
    \Lambda_X(t; \theta) := \log \left( \dE\left[e^{tf(X; \theta)}\right]\right) = \log \sum_x p(x) p_\theta(x)^{-t},
\end{equation}
where in this paper $\dE[ \cdot ]$ denotes expectation with respect to the true distribution $p$ unless otherwise stated.
This expectation is commonly referred to as an \textit{exponential tilt} of the information density, and can induce parametric distribution shifts that have varied applications in probability, statistics, and information theory. In particular, it is noteworthy that if $\cal P$ is an exponential family of distributions parameterized by $\theta,$ then the tilted distribution $p_{\theta}(x)^t$ (when normalized by $\int_{\cal X} p_{\theta}(x)^t dx$) also belongs to the same exponential family. 
Further, given samples $\{x_i\}_{i \in [N]},$ the empirical cumulant generating function is defined as: 
\begin{equation}
    \widetilde{\Lambda}(t; \theta) := \log \left( \frac{1}{N}\sum_{i \in [N]}\left\{e^{tf(x_i; \theta)}\right\}\right).
    \label{eq:cf}
\end{equation}
It is thus evident that TERM~\eqref{eq: TERM} can be viewed as an appropriately scaled variant of the empirical cumulant generating function in~\eqref{eq:cf}. Although tilting of this form has been used in a number of related disciplines, uses of exponential tilting in machine learning are relatively unexplored. We provide several perspectives on exponential tilting from other fields below.

\paragraph{Statistics.} 
Exponential tilting is well-known as a distribution shifting technique in statistics, where the main idea is to draw samples from an exponentially tilted version of the original distribution to improve the convergence properties of statistical estimation, especially when the distribution of interest belongs to an exponential family, such as Gaussian or multinomial.
Common use cases include rejection sampling, rare-event simulation, saddle-point approximation~\cite[p. 156]{butler2007saddlepoint}, and importance sampling~\citep{siegmund1976importance}. 

\paragraph{Applied probability.} 
In large deviations theory, exponential tilting lies at the heart of deriving concentration bounds. For example, Chernoff bounds apply Markov's inequality to $e^{tX},$ which results in a parametric set of bounds by using  exponential tilts of various orders. The bound may then be further optimized on the real tilt value to derive the tightest possible bound~\citep{dembo-zeitouni}.

\paragraph{Information theory.} While source coding limits and channel capacity are characterized by Shannon entropy and Shannon mutual information (which are simple averages over the information~\eqref{eq:information})~\citep{cover1991information}, there are other elements of information theory that are not characterized by the average, such as error exponents in channel decoding~\citep{Gallager-book}, probability of error in list decoding~\citep{Merhav-list}, and computational cost in sequential decoding~\citep{massey94, Arikan-guesswork}. These fundamental elements of information theory are asymptotically determined by a non-zero tilted cumulant generating function of the information random variable~\eqref{eq:information} (see  \citep{beirami2018characterization} for further discussion). 

\paragraph{Optimization.} 
Exponential tilting has also appeared as a minimax smoothing approach in optimization~\citep{kort1972new,pee2011solving,liu2019deep}.
{Such smooth approximations to the max often appear through LogSumExp functions, with applications in geometric programming~\citep[][Sec.~9.7]{calafiore2014optimization}, and boosting~\citep{mason1999boosting,shen2010dual}.} We discuss min-max objectives and the connections with TERM in several subsequent sections of the paper.

\paragraph{\textit{Machine learning.}}
Despite the rich history of tilted objectives in related fields, they have not seen widespread use in ML beyond limited applications such as robust regression~\citep{wang2013robust} and sequential decision making~\citep{howard1972risk,borkar2002q}. 
In this work, we argue that tilting is a critical yet undervalued  tool in machine learning. We demonstrate the effectiveness of tilting by (i) rigorously studying properties of the TERM objective, and (ii) exploring its utility for a wide range of ML applications. Surprisingly, we find that this simple extension to ERM can match or exceed state-of-the-art performance from highly tuned, bespoke solutions to common ML problems,  from learning with noisy data to ensuring fair performance between subgroups. We highlight several motivating applications of TERM below and provide an outline of the remainder of the paper in Section~\ref{sec:contributions}.


\begin{figure}[t!]
    \centering
    \includegraphics[width=1.0\textwidth]{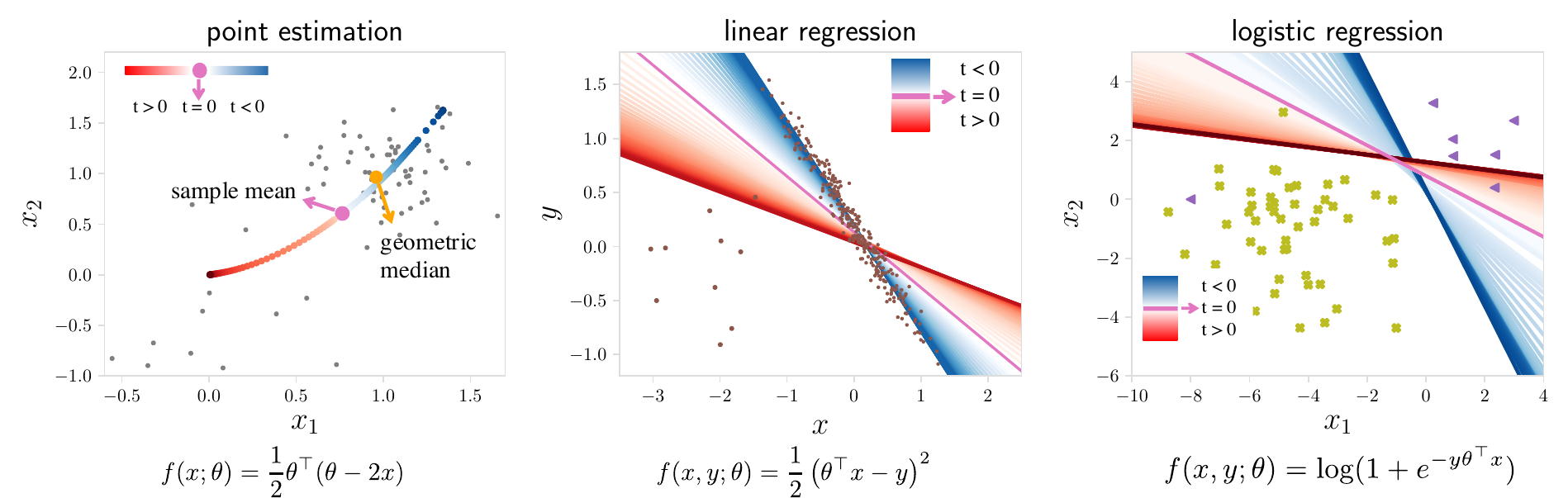}
    \caption{Toy examples illustrating TERM as a function of $t$: (a) finding  a point estimate from a set of 2D samples, (b) linear regression with outliers, and (c) logistic regression with imbalanced classes. While positive values of $t$ magnify outliers, negative values suppress them. Setting $t{=}0$ recovers the original ERM objective~\eqref{eq: ERM}.
    }
    \label{fig:example}
\end{figure}

\subsection{Motivating Examples} 

To motivate how the TERM objective~\eqref{eq: TERM} may be used in machine learning, we provide several running examples below, which are illustrated in Figure~\ref{fig:example}.

\textit{(a) Point estimation:}  As a first example, consider determining a point estimate from a set of samples that contain some outliers. We plot an example 2D dataset in Figure~\ref{fig:example}a, with data centered at (1,1). Using traditional ERM (i.e., TERM with $t=0$) recovers the \textit{sample mean}, which can be biased towards outlier data. 
By setting $t<0$, TERM can suppress outliers by reducing the relative impact of the largest losses (i.e., points that are far from the estimate) in~\eqref{eq: TERM}. A specific value of $t<0$ can in fact approximately recover the geometric median, as the objective in~\eqref{eq: TERM} can be viewed as approximately optimizing specific loss quantiles (a connection which we make explicit in Section~\ref{sec:term_properties}). In contrast, if these `outlier' points are important to estimate, setting $t>0$ will push the solution towards a  point that aims to minimize variance, as we prove  in Section~\ref{sec:term_properties}, Theorem~\ref{thm: variance-reduction}.

\textit{(b) Linear regression:} A similar interpretation holds for the case of linear regression (Figure 2b). As $t \to -\infty,$ 
TERM finds a line of best while ignoring outliers. 
However, this solution may not be preferred if we have reason to believe that these `outliers' should not be ignored. As $t \to +\infty,$ TERM recovers the min-max solution, which aims to minimize the worst loss, thus ensuring the model is a reasonable fit for \textit{all} samples (at the expense of possibly being a worse fit for many). Similar criteria have been used, e.g., in defining notions of fairness~\citep{hashimoto2018fairness,samadi2018price}. We explore several use-cases involving robust regression and fairness in more detail in Section~\ref{sec:experiments}.

\textit{(c) Logistic regression:} Finally, we consider a binary classification problem using logistic regression (Figure 2c). For $t \in \mathbb{R}$,
the TERM solution varies from the nearest cluster center ($t{\to} {-}\infty$), to the logistic regression classifier ($t{=}0$),  towards a classifier that magnifies the misclassified data ($t{\to}{+}\infty$).  
We note that it is common to modify logistic regression classifiers by adjusting the decision threshold from $0.5$, which is equivalent to moving the intercept of the decision boundary.
This is fundamentally different than what is offered by TERM (where the slope is changing). As we show in Section~\ref{sec:experiments}, this added flexibility affords TERM with competitive performance on a number of classification problems, such as those involving noisy data, class imbalance, or a combination of the two.

\subsection{Contributions}
\label{sec:contributions}

In this work, we explore the use of tilting in machine learning through TERM, a simple, unified framework that can flexibly address various challenges with empirical risk minimization. We first analyze the objective and its solutions, showcasing the behavior of TERM with varying tilt parameters $t$ (Section~\ref{sec:term_properties}). We {also} establish connections between TERM and related approaches such as distributionally robust optimization in Section~\ref{sec:other objectives}.

We rigorously analyze the relations between TERM and other risks (e.g, Value-at-Risk {(VaR)}, Conditional Value-at-Risk {(CVaR)}, and Entropic Value-at-Risk {(EVaR)}) in Section~\ref{sec:var}. 
In particular, we introduce a new risk measure based on TERM, called Tilted Value-at-Risk (TiVaR), to approximate VaR. We show that TiVaR {can provide a better approximation of VaR than CVaR in certain regimes,} and improves upon EVaR in all regimes.

We develop efficient first-order batch and stochastic methods for solving TERM, both for hierarchical and non-hierarchical cases (Section~\ref{sec:solver} and~\ref{sec:term-extended}). We provide convergence rates scaling with the hyperparameter $t$ on both convex and non-convex problems for both batch and stochastic algorithms. Our solvers run within 2--3$\times$ wall-clock time compared with that of ERM in all explored case studies.

Finally, we show via numerous case studies that TERM is competitive with  existing, problem-specific state-of-the-art solutions (Section~\ref{sec:experiments}). We also extend TERM to handle compound issues, such as the simultaneous existence of noisy samples and imbalanced classes (Section~\ref{sec:term-extended}). 
Our results demonstrate the effectiveness and versatility of tilted objectives in machine learning.

We note that the material in this paper was presented in part at ICLR 2021 (\hyperlink{cite.TERM}{Li and Beirami et al., 2021}). Compared to this earlier work, the current manuscript provides additional historical background of tilting (Section~\ref{sec:intro}), establishes  stronger and novel relationships between tilted losses and other risk-averse objectives in the literature (Section~\ref{sec:other objectives} and Section~\ref{sec:var}), provides convergence guarantees for our stochastic solver of TERM (Section~\ref{sec:solver}), offers comprehensive details on applications of the framework in practice, and considers new applications of TERM to meta-learning and heteroskedastic deep learning (Section~\ref{sec:experiments}). 

\paragraph{Outline.} This paper is organized as follows. We discuss  general properties and interpretations of TERM in Section~\ref{sec:term_properties}. We connect TERM with other prior risk measures in Section~\ref{sec:other objectives} and propose a new risk motivated by TERM in Section~\ref{sec:var}. In Section~\ref{sec:solver}, we develop both batch and stochastic algorithms for optimizing TERM and provide convergence guarantees for them. We extend TERM to hierarchical multi-objective tilting in Section~\ref{sec:term-extended} and demonstrate the flexibility and competitive performance of the TERM framework via real-world applications in Section~\ref{sec:experiments}. We discuss related work in Section~\ref{sec:background:related_work} and conclude the paper with Section~\ref{sec:conclusion}.

\section{TERM: Properties and Interpretations}\label{sec:term_properties}

To better understand the performance of the $t$-tilted losses in~\eqref{eq: TERM}, in this section we provide several interpretations of the TERM solutions, leaving the full proofs to the appendix. We make no distributional assumptions on the data, and study properties of TERM under the assumption that the loss function forms a generalized linear model, e.g., $L_2$ loss and logistic loss.
However, we also obtain favorable empirical results using TERM with other objectives such as PCA and deep neural networks in Section~\ref{sec:experiments}, motivating the extension of this part of our theory beyond GLMs in future work.

\subsection{Assumptions} \label{app:assumptions}

We first provide notation and assumptions that are used throughout our theoretical analyses. The results in this paper are derived under one of the following {three} nested assumptions {(the assumptions become progressively more restrictive, i.e., $3 \to 2 \to 1$)}:
\begin{assumption}[{Continuous differentiability}]
\label{assump:smoothness}
For $i \in [N],$ the loss function $f(x_i;\theta)$ belongs to the  differentiability class $C^1$ (i.e., continuously differentiable) with respect to $\theta \in \Theta \subseteq \mathbb{R}^d.$ 
\end{assumption}

\begin{assumption}[{Smoothness and} strong convexity condition]\label{assump: regularity}
Assume that Assumption~\ref{assump:smoothness} is satisfied. In addition, for any $i\in [N]$, $f(x_i; \theta)$ belongs to differentiability class $C^2$ (i.e., twice differentiable with continuous Hessian) with respect to $\theta$.
We further assume that there exist $\beta_{\min}, \beta_{\max} \in \mathbb{R}^{>0}$ such that for $i \in [N]$ and any $\theta \in \Theta \subseteq \mathbb{R}^d,$
\begin{equation}
\label{eq: assump1-strong-convex}
    \beta_{\min} \mathbf{I}\preceq  \nabla^2_{\theta \theta^\top} f(x_i;\theta) \preceq \beta_{\max}\mathbf{I},
\end{equation}    
where $\mathbf{I}$ is the identity matrix of appropriate size (in this case $d\times d$), and there does {\bf not} exist any $\theta \in \Theta,$ such that $\nabla_\theta f(x_i; \theta) = 0$ for all $i\in [N].$
\end{assumption}

\begin{assumption}[Generalized linear model condition~\citep{wainwright2008graphical}]\label{assump: expnential_family}
Assume that Assumption~\ref{assump: regularity} is satisfied.
Further, assume that the loss function $f(x; \theta)$ is given by
\begin{equation}
    f(x; \theta) = A(\theta) - \theta^\top T(x),
\label{eq: loss-exponential}
\end{equation}
where $A(\cdot)$ is a convex function such that there exists $\beta_{\max}$ where for any $\theta \in \Theta \subseteq \mathbb{R}^d,$
\begin{equation}
   \beta_{\min} \mathbf{I}\preceq \nabla^2_{\theta \theta^\top} A(\theta)  \preceq \beta_{\max} \mathbf{I}\, ,
\end{equation}   
and
\begin{equation}
    \sum_{i \in [N]} T(x_i) T(x_i)^\top \succ 0. 
\end{equation}
\end{assumption}
This set of assumptions become the most restrictive with 
Assumption~\ref{assump: expnential_family}, which essentially requires 
that the loss be the negative log-likelihood of an exponential family. While the assumption is stated using the natural parameter of an exponential family for ease of presentation, the results hold for any bijective and smooth reparameterization of the exponential family.
For example, Assumption~\ref{assump: expnential_family}
is  satisfied by the commonly used $L_2$ loss for regression and logistic loss for classification (see toy examples (b) and (c) in Figure~\ref{fig:example}). {$\sum_{i \in [N]} T(x_i) T(x_i)^\top \succ 0$ assumes a reasonable regularity on the dataset $\{x_i\}_{i \in [N]}$. For instance, in the case of linear regression ($T(x_i)=x_i \in \mathbb{R}^d$), it reduces to the standard regularity assumption $XX^{T} \succ 0$ (where $X :=[x_1,\cdots, x_N] \in \mathbb{R}^{d \times N}$).}  While Assumption~\ref{assump: expnential_family} is not satisfied when we use neural network function approximators in Section~\ref{sec:experiments}, we observe favorable numerical results motivating the extension of these results beyond the cases that are theoretically studied in this paper.

In the sequel, many of the results are concerned with characterizing the $t$-tilted solutions defined as the parametric set of solutions of $t$-tiled losses by sweeping  $t \in \mathbb{R}$, 
\begin{equation}
    \label{eq: opt_obj}
    \breve{\theta}(t) \in \arg\min_{\theta \in \Theta} \wR(t ;\theta),
\end{equation}
where $\Theta \subseteq \mathbb{R}^d$ is an open subset of $\mathbb{R}^d.$ Further, let the optimal tilted objective be defined as
\begin{equation} \label{def:wF}
    \widetilde{F}(t):= \wR(t; \breve{\theta}(t)).
\end{equation} 
We state a final assumption, on $\breve{\theta}(t)$, below.
\begin{assumption}[Strict saddle property~(Definition 4 in \citet{ge2015escaping})]\label{assump:strict-saddle}
We assume that the set $ \arg\min_{\theta \in \Theta} \wR(t ;\theta)$ is non-empty for all $t \in \mathbb{R}$.
Further, we assume that for all $t \in \mathbb{R},$ $\wR(t; \theta)$ is a ``strict saddle'' as a function of $\theta$, i.e., for all local minima,  $\nabla^2_{\theta \theta^\top}\wR(t; \theta){\succ}0$, and for all other stationary solutions, $\lambda_{\min}(\nabla^2_{\theta \theta^\top}\wR(t; \theta))< 0$, where $\lambda_{\min}(\cdot)$ is the minimum eigenvalue of the matrix.
\end{assumption}
We use the strict saddle property in order to reason about the properties of the $t$-tilted solutions.
In particular, since we are solely interested in the local minima of $\wR(t; \theta),$ the strict saddle property implies that for every $\breve{\theta}(t) \in \arg\min_{\theta \in \Theta} \wR(t; \theta),$ for a sufficiently small $r$, for all $\theta \in \mathcal{B}(\breve{\theta}(t), r),$
\begin{equation}
    \nabla^2_{\theta \theta^\top} \wR(t; \theta) \succ 0,
\end{equation}
where $\mathcal{B}(\breve{\theta}(t), r)$ denotes a $d$-ball of radius $r$ around $\breve{\theta}(t).$
We will show later in Section~\ref{sec:property:general} that the strict saddle property is readily verified for $t \in \mathbb{R}^{>0}$ under Assumption~\ref{assump: regularity}, {and we need Assumption~\ref{assump:strict-saddle} to be able to reason about $t \in \mathbb{R}^{<0}$.} 

\subsection{General Properties of TERM} \label{sec:property:general}
We begin by noting several general properties of the TERM objective~\eqref{eq: TERM}. In particular: (i) $\wR(t;\theta)$ is $L$-Lipschitz continuous in $\theta$ if $f(x;\theta)$ is $L$-Lipschitz (Lemma~\ref{lem: lipschitz}); (ii) If $f(x; \theta)$ is strongly convex, the $t$-tilted loss is strongly convex for  $t>0$ (Lemma~\ref{lem: Hessian}); and (iii) 
Given a smooth $f(x; \theta)$, the $t$-tilted loss is smooth for all finite $t$ (Lemma~\ref{lemma:TERM-smoothness}). We state these properties more formally below.

\begin{lemma}[Lipschitzness of $\wR(t;\theta)$]\label{lem: lipschitz}
For any $t\in \mathbb{R}$ and $\theta \in \Theta$, if for $i\in [N]$, $f(x_i;\theta)$ is $L$-Lipschitz conditnuous in $\theta$, then $\wR(t;\theta)$ is $L$-Lipschitz in $\theta$.
\end{lemma}


\begin{lemma}[Tilted Hessian and strong convexity for $t \in \mathbb{R}^{>0}$]\label{lem: Hessian}
 Under Assumption~\ref{assump: regularity}, for any $t\in \mathbb{R},$ 
\begin{align}
        \nabla^2_{\theta \theta^\top}\wR(t ;\theta) &= 
         {\frac{\color{black}t}{N}} \sum_{i \in [N]} ( \nabla_{\theta}f(x_i;\theta) - \nabla_\theta \wR(t ;\theta))(\nabla_{\theta}f(x_i;\theta) - \nabla_\theta \wR(t ;\theta))^\top  e^{t (f(x_i; \theta) - \wR(t; \theta))}\label{eq:smoothness-23}\\
        & \quad + {\frac{1}{N}} \sum_{i \in [N]} \nabla^2_{\theta\theta^\top}f(x_i;\theta) e^{t ( f(x_i; \theta) - \wR(t; \theta))}.\label{eq:smoothness-24}
\end{align}
In particular, for all $\theta \in \Theta$ and all $t \in \mathbb{R}^{>0},$ the $t$-tilted objective is strongly convex. That is
\begin{equation}
     \nabla^2_{\theta \theta^\top}\wR(t ;\theta) \succ \beta_{\min} \mathbf{I}.
\end{equation}
\end{lemma}

Lemma~\ref{lem: lipschitz} and~\ref{lem: Hessian} are proved in Appendix~\ref{app:properties}. Lemma~\ref{lem: Hessian} also implies that under Assumption~\ref{assump: regularity}, the strict saddle assumption (Assumption~\ref{assump:strict-saddle}) is readily verified.

\begin{lemma}[Smoothness of $\wR(t; \theta)$] 
\label{lemma:TERM-smoothness}
For any $t \in \mathbb{R}$, let $\beta(t)$ be the smoothness parameter of twice differentiable $\wR(t;\theta)$:
\begin{equation}
    \beta(t) :=  \lambda_{\max} \left( \nabla^2_{\theta\theta^\top} \wR(t; \theta)\right),
\end{equation}    
where $\nabla^2_{\theta\theta^\top} \wR(t; \theta)$ is the Hessian of $\wR(t; \theta)$ at $\theta$ and $\lambda_{\max}(\cdot)$ denotes the largest eigenvalue. 
Under Assumption~\ref{assump: regularity}, 
for any $t\in \mathbb{R},$ $\wR(t; \theta)$ is a  $\beta(t)$-smooth function of $\theta$. 
Further, for $t\in \mathbb{R}^{\leq 0},$\footnote{{$\mathbb{R}^{\leq 0}$ denotes the set of  non-positive real numbers.}} 
\begin{equation}
     \beta(t) <  \beta_{\max}, 
\end{equation}
where $\beta_{\max}$ is defined in Assumption~\ref{assump: regularity}.
For $t \in \mathbb{R}^{>0},$ 
\begin{align}
0 <\lim_{t \to +\infty} \frac{\beta(t)}{t} &< +\infty.
\end{align}
\end{lemma}
Lemma~\ref{lemma:TERM-smoothness} (proved in Appendix~\ref{app:general-TERM-properties}) indicates that $t$-tilted losses are $\beta(t)$-smooth for all $t$. $\beta(t)$ is bounded for all negative  $t$ and moderately positive $t$, whereas it scales linearly with $t$ as $t \to +\infty$, which has been previously studied in the context of exponential smoothing of the max~\citep{kort1972new, pee2011solving}. This can also be observed visually via the toy example in Figure~\ref{fig:toy1}.

As discussed in Section~\ref{sec:introduction}, TERM can recover traditional ERM ($t{=}0$), the max-loss ($t{\to}{+}\infty$), and the min-loss ($t{\to}{-}\infty$). We formally state this in Lemma~\ref{lemma: special_cases} below.

\begin{figure}[t!]
    \centering
    \hspace{-.14in}
    \includegraphics[trim=0 0 0 5,clip, width=0.55\textwidth]{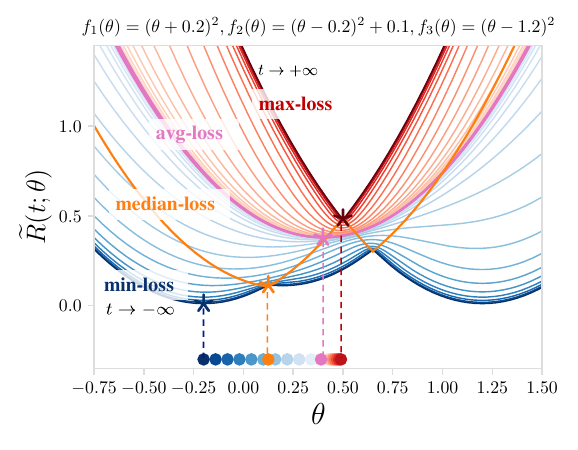}
    \vspace{-5mm}
    \caption{TERM objectives for a squared loss with three samples ($N$=3). As $t$ moves from $-\infty$ to $+\infty$, $t$-tilted losses recover  min-loss, avg-loss, and  max-loss. TERM is smooth for all finite $t$ and convex for positive $t$.}
    \label{fig:toy1}
\end{figure}

\begin{lemma}
\label{lemma: special_cases}
Under Assumption~\ref{assump:smoothness},
\begin{align}
\wR(-\infty; \theta) &:= \lim_{t \to -\infty}\wR(t ;\theta) = \widecheck{R}( \theta),\\
\wR(0; \theta) &:= \lim_{t \to 0}\wR(t ;\theta) = \overline{R}( \theta), \label{eq:def-R0} \\
\wR(+\infty; \theta) &:= \lim_{t \to +\infty}\wR(t ;\theta) = \widehat{R}( \theta),
\end{align}
where $\widehat{R}(\theta)$ is the max-loss and $\widecheck{R}(\theta)$ is the min-loss\footnote{{When the argument of the max-loss or the min-loss is not unique, for the purpose of differentiating the loss function, we define $\widehat{R}(\theta)$ as the average of the individual losses that achieve the maximum, and $\widecheck{R}(\theta)$ as the average of the individual losses that achieve the minimum.}}:
\begin{equation}
    \widehat{R}(\theta) := \max_{i \in [N]} f(x_i; \theta), \quad\quad\quad\quad\quad     \widecheck{R}(\theta) := \min_{i \in [N]} f(x_i; \theta). 
\label{eq:best-worst-risk}
\end{equation}
\end{lemma}

Note that Lemma~\ref{lemma: special_cases} has been studied or observed before in the entropic risk literature~\citep[e.g.,][]{ahmadi2012entropic}, as well as other contexts~\citep{cohen2014simnets}.
This lemma also implies that $ \breve{\theta}(0)$ is the ERM solution, $\breve{\theta}(+\infty)$ is the min-max solution, and $\breve{\theta}(-\infty)$ is the min-min solution. In other words, a benefit of TERM is that it offers a continuum of solutions between the min and max losses.

Providing a smooth trade-off between these specific losses can be beneficial for a number of practical use-cases---both in terms of the resulting solution and the difficulty of solving the problem itself. We empirically demonstrate the benefits of such a trade-off in Section~\ref{sec:experiments}. We also visualize the solutions to TERM for a toy problem in Figure~\ref{fig:toy1}, which allows us to illustrate several special cases of the general framework. Interestingly, we additionally show that the TERM solution can be viewed as a smooth approximation to {the \textit{tail probability of losses}}, which effectively minimizes quantiles of losses such as the median loss (Section~\ref{sec:var}). In Figure~\ref{fig:toy1}, it is clear to see why this may be beneficial, as the median loss (orange) can be highly non-smooth in practice. In Theorem~\ref{thm: obj-increasing} and~\ref{thm: opt-obj-increasing} below, we formally characterize how tilted objectives change as a function of values $t$ (proofs provided in Appendix~\ref{app:properties}).

\begin{theorem}[Tilted objective is increasing with $t$]
\label{thm: obj-increasing}
Under Assumption~\ref{assump: expnential_family}, for all $t \in \mathbb{R},$ and all $\theta \in \Theta,$
\begin{equation}
    \frac{\partial}{\partial t} \wR(t; \theta) \geq 0.
\end{equation}
\end{theorem}

\begin{theorem}[Optimal tilted objective is increasing with $t$] 
\label{thm: opt-obj-increasing}
Under Assumption~\ref{assump: expnential_family}, for all $t \in \mathbb{R},$ and all $\theta \in \Theta,$
\begin{equation}
    \frac{\partial}{\partial t} \wF(t) = \frac{\partial}{\partial t} \wR(t; \breve{\theta}(t)) \geq 0.
\end{equation}
\end{theorem}
{
Recall that TERM as $t \to -\infty$ and $t \to \infty$ corresponds to min-loss and max-loss, respectively. 
We  discuss in Section~\ref{sec:var:approximate} that solving TERM with any $t \in \mathbb{R}$ can indeed be viewed as approximately minimizing the $k$-th smallest loss ($k \in [N]$) among all $N$ individual losses. As we increase $k$ from $1$ to $N,$ the corresponding value of $t$ sweeps in $(-\infty, \infty)$.
Theorem~\ref{thm: opt-obj-increasing} hence roughly states that the optimal $k$-th smallest loss is non-decreasing with $k$, which is intuitively expected.} 

We next provide two interesting interpretations of the TERM framework to further understand its behavior. 

\subsection{Interpretation 1: Re-Weighting Samples to Magnify/Suppress Outliers} \label{sec:property:reweighting}

\begin{figure}[h]
    \centering
    \begin{subfigure}{0.32\textwidth}
        \centering
        \includegraphics[width=\textwidth]{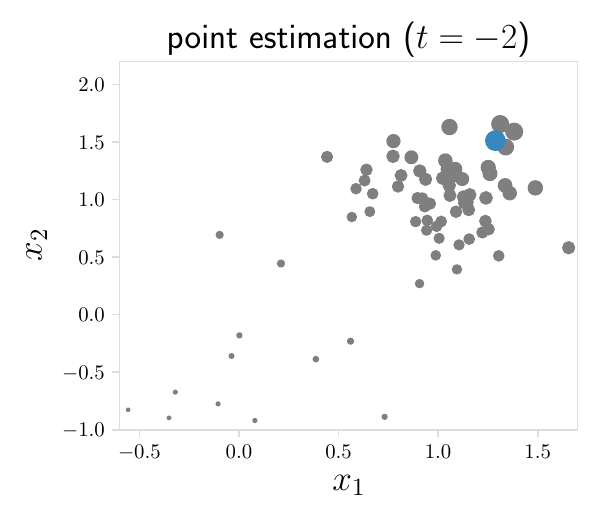}
    \end{subfigure}
    \begin{subfigure}{0.32\textwidth}
        \centering
        \includegraphics[width=\textwidth]{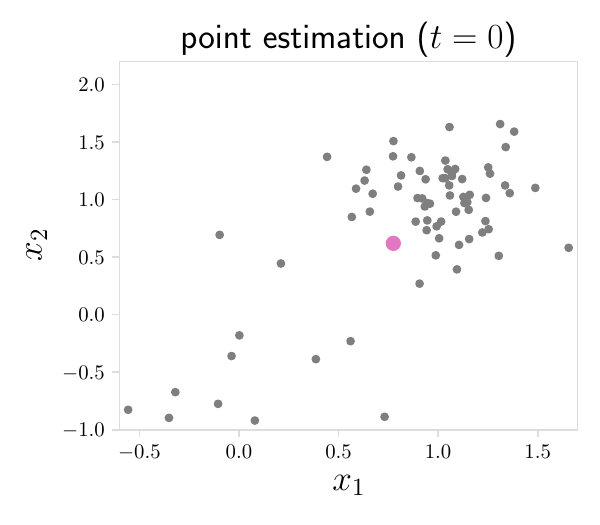}
    \end{subfigure}
    \begin{subfigure}{0.32\textwidth}
        \centering
        \includegraphics[width=\textwidth]{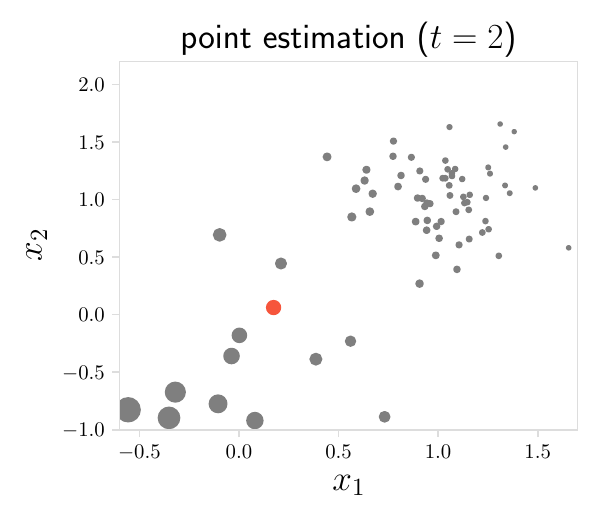}
    \end{subfigure}
    \begin{subfigure}{0.32\textwidth}
        \centering
        \includegraphics[width=\textwidth]{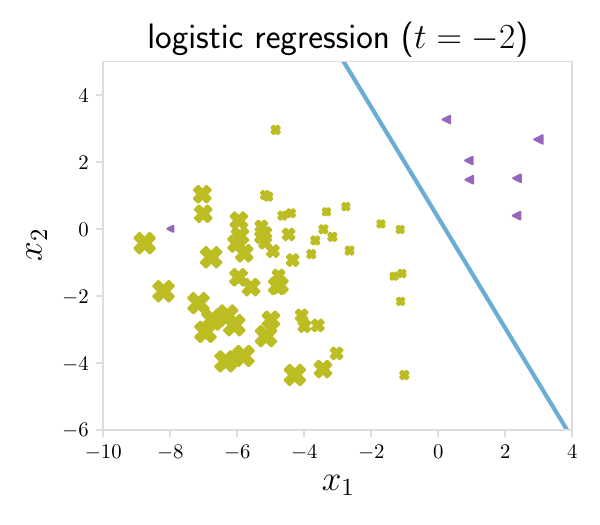}
    \end{subfigure}
    \begin{subfigure}{0.32\textwidth}
        \centering
        \includegraphics[width=\textwidth]{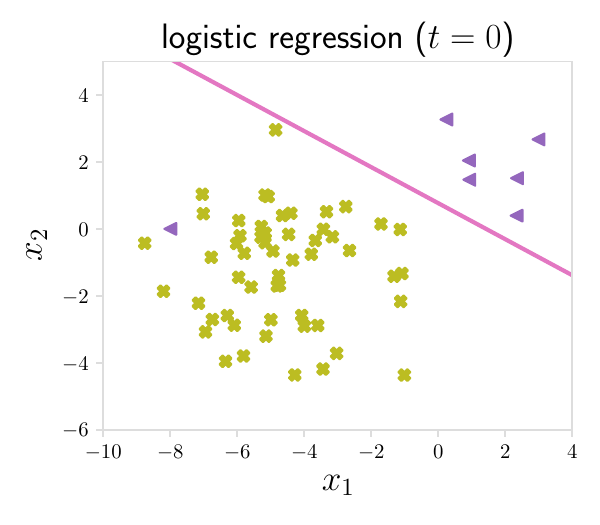}
    \end{subfigure}
    \begin{subfigure}{0.32\textwidth}
        \centering
        \includegraphics[width=\textwidth]{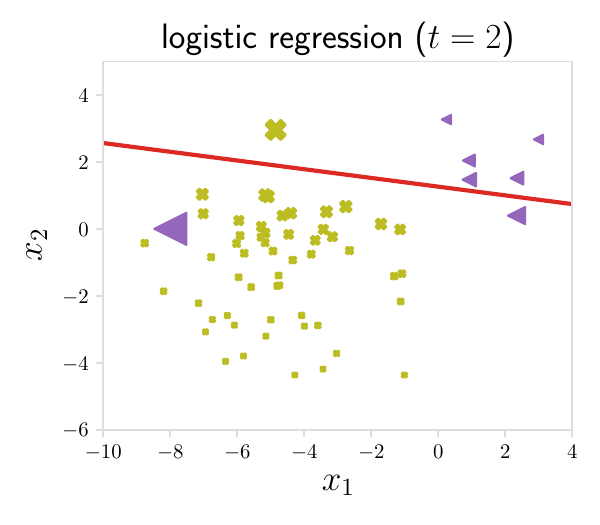}
    \end{subfigure}
    \caption{We visualize the size of the samples using their gradient weights. Negative $t$'s  ($t=-2$ on the left) focus on the inlier samples (suppressing outliers), while positive $t$'s ($t=2$ on the right) magnify the outlier samples. 
    }
    \label{fig:weights}
\end{figure}

As discussed via the toy examples in Section 1,  TERM  can be tuned (using $t$) to magnify or suppress the influence of outliers. We make this notion rigorous by exploring the \textit{gradient} of the $t$-tilted loss in order to reason about the solutions to the objective  defined in~\eqref{eq: TERM}.
\begin{lemma}[Tilted gradient]
\label{lemma:TERM-gradient}
For a smooth loss function $f(x; \theta)$,
\begin{equation}
\setlength{\thinmuskip}{0mu}
\setlength{\medmuskip}{0.75mu}
\setlength{\thickmuskip}{0.75mu}
\label{eq: tilted_grad}
    \nabla_\theta \wR(t ;\theta){=}\sum_{i \in [N]} w_i(t; \theta) \nabla_\theta f(x_i;\theta), 
\end{equation}
where {tilted weights are given by}
\begin{equation}
w_i (t; \theta) := \frac{e^{t f(x_i; \theta)}}{\sum_{j \in [N]} e^{t f(x_j; \theta)}} = \frac{1}{N}e^{t(f(x_i; \theta) - \wR(t; \theta))}. \label{eq: w_i}
\end{equation}
\end{lemma}

\begin{proof}
Under Assumption~\ref{assump:smoothness}, we have:
\begin{align}
        \nabla_\theta \wR(t ;\theta) &= \nabla_\theta \left\{\frac{1}{t} \log\left( \frac{1}{N}\sum_{i \in [N]}  e^{t f(x_i; \theta)} \right)\right\} = \frac{\sum_{i \in [N]}\nabla_{\theta}f(x_i;\theta)e^{t f(x_i; \theta)}}{\sum_{i \in [N]} e^{t f(x_i; \theta)}} \, .
\end{align}
\end{proof}
Lemma~\ref{lemma:TERM-gradient} provides the gradient of the tilted objective, which has been studied previously in the context of exponential smoothing (see~\citet[Proposition 2.1]{pee2011solving}). From this, we can observe that the tilted gradient is a weighted average of the gradients of the original individual losses, where each data point is weighted exponentially proportional to the value of its loss.  
Note that $t=0$ recovers the uniform weighting associated with ERM, i.e., $w_i(t; \theta) = 1/N$. For positive $t,$ this has the effect of  {\it magnifying} the outliers---samples with large losses---by assigning more weight to them, and for negative $t,$ it {\it suppresses} the outliers by assigning less weight to them (Figure~\ref{fig:weights}).

{Generalizing the notion of tilted gradients (weighted average of individual gradients), we define tilted empirical mean over any $N$-vector $\mathbf{u} \in \mathbf{R}^N$ below, which will be used throughout the paper.}
{\begin{definition}[Tilted empirical mean and variance]
\label{def: tilted_mean}
For $\mathbf{u} \in \mathbf{R}^N$, let  weighted empirical mean with weights $\mathbf{w} \in \Delta^N$ (where $\Delta^N$ stands for $N$ dimensional simplex) be defined as
\begin{align}
     \widehat{E}_\mathbf{w}(\mathbf{u}) &: = \sum_{i \in [N]} w_i u_i.
\end{align}
Tilted empirical mean is weighted empirical mean with tilted weights, i.e.,
\begin{align}
     \widehat{E}_{\mathbf{w}(t; \theta)}(\mathbf{u}) &: = \sum_{i \in [N]} w_i(t; \theta) u_i, \label{eq: tilted_mean} \\
    \widehat{E}_{\mathbf{w}(t; \breve{\theta}(t))}(\mathbf{u}) &: = \sum_{i \in [N]} w_i(t; \breve{\theta}(t)) u_i,  ~~ \widehat{E}_t : = \widehat{E}_{\mathbf{w}(t; \breve{\theta}(t))}(\mathbf{u}), \label{eq: t_tilted_mean}
\end{align}
where $w_i(t; \theta)$ is defined in Eq.~\eqref{eq: w_i}, and $\breve{\theta}(t)$ is defined in Eq.~\eqref{eq: opt_obj}. We also refer to $\widehat{E}_t$ as the ``$t$-tilted empirical mean''. Similarly, tilted empirical variance is defined as
\begin{align}
\widehat{\var}_{\mathbf{w}(t; \theta)}(\mathbf{u}) & : =     \widehat{E}_{\mathbf{w}(t; \theta(t))}\left(u_i-\widehat{E}_{\mathbf{w}(t; \theta(t))}\left(\mathbf{u}\right)\right)^2, \label{eq: tilted_variance} \\ 
\widehat{\var}_{\mathbf{w}(t; \breve{\theta}(t))}(\mathbf{u}) &:= \widehat{E}_t (u_i - \widehat{E}_t\left(\mathbf{u})\right)^2, ~~\widehat{\var}_{t} :=  \widehat{\var}_{\mathbf{w}(t; \breve{\theta}(t))}(\mathbf{u}), \label{eq: t_tilted_variance}
\end{align}
and we refer to $\widehat{\var}_{t}$ as the ``$t$-tilted empirical variance''.
\end{definition}
As discussed before, the full gradient of TERM is tilted empirical mean of individual gradients $\{\nabla_{\theta} f(x_i;\theta)\}_{i \in [N]}$ with weights proportional to $e^{tf(x_i;\theta)}$. In the next section as well as Appendix~\ref{app: general-TERM-GLM-solution}, we will prove other properties of TERM using tilted empirical mean and variance defined here.}

\subsection{Interpretation 2: Empirical Bias/Variance Trade-off}
\label{sec:property:variance}


{Another key property of the TERM solutions is that for any $t \in \mathbb{R}$, $t$-tilted empirical variance of the losses across all samples will decrease if we increase $t$ by a small amount of value. We formally stated this in Theorem~\ref{thm: variance-reduction}}.

{
\begin{theorem}[Variance reduction]
Let $\mathbf{f}(\theta):= (f(x_1; \theta)), \ldots, f(x_N; \theta))$. 
Then, under Assumption~\ref{assump: expnential_family} and Assumption~\ref{assump:strict-saddle}, for any $t \in \mathbb{R}$,
\begin{equation}
 \left. \frac{\partial}{\partial t} \left\{  \widehat{\var}_\tau(\mathbf{f}(\breve{\theta}(t)))\right\} \right|_{t = \tau} <0.
\end{equation}
\label{thm: variance-reduction}
\vspace{-1em}
\end{theorem}
Note that $\widehat{\var}_\tau$ is $\tau$-tilted empirical variance defined in Eq.~\eqref{eq: t_tilted_variance}. Hence, for any $t$, the $t$-tilted empirical variance among $N$ losses will decrease if we increase $t$ by a small value. When $\tau=0$, $\widehat{\var}_\tau$ reduces to standard empirical variance. In particular, Theorem~\ref{thm: variance-reduction} states that the empirical variance of the loss vector decreases if $t$ is chosen to be a small positive value.} 
Therefore, it is possible to trade off between optimizing the average loss vs. reducing variance, allowing the solutions to potentially  achieve a better bias-variance trade-off for generalization~\citep{maurer2009empirical,bennett1962probability,hoeffding1994probability}.  {At a high level, this property is consistent with and extends the approximation of TERM mentioned by~\citet[Section V.A]{liu2019deep}, which approximates TERM as the empirical risk regularized with  variance of the loss at $t=0$.}
We {rely on} this property to achieve better generalization in classification in Section~\ref{sec:experiments}. 


In addition to empirical variance across all losses, there are other related distribution uniformity measures. In Theorem~\ref{thm: uniform_gradient_weights} below, we also prove that entropy of the weight distribution at solution $\breve{\theta}(t)$ tilted by $\tau$ close to $t$ is increasing with $t$, which indicates that larger $t$'s encourages more uniform solutions measured via entropy.




\begin{theorem}[Gradient weights become more uniform by increasing $t$]\label{thm: uniform_gradient_weights}
Under Assumption~\ref{assump: expnential_family} and Assumption~\ref{assump:strict-saddle},
for any $t \in \mathbb{R}^{>0},$
\begin{equation}
    \left. \frac{\partial}{\partial t} H({\bf w} (\tau; \breve{\theta}(t))) \right|_{\tau=t} > 0,
\end{equation}
where $H(\cdot)$ denotes the Shannon entropy function measured in nats,
\begin{equation}
H\left({\bf w}(t; \theta)\right) := -  \sum_{i \in [N]} w_i(t; \theta) \log w_i(t; \theta ).
\end{equation}
\label{theorem:weight-uniformity}
\vspace{-0.5em}
\end{theorem}
Full proofs of the theorems presented in this section can be found in Appendix~\ref{app: general-TERM-GLM-solution}. In the next section, we connect TERM to other objectives. Note that the results in all subsequent sections do not require the GLMs assumption, unless stated otherwise.

\section{Connections to Other Risk Measures} \label{sec:other objectives}

In this section (and subsequently in Section~\ref{sec:var}) we explore TERM by comparing, contrasting, and drawing connections between  TERM and other common risk measures.  To do so, we first introduce a distributional version of TERM, which is closely related to entropic risk (measure) in previous literature~\citep{ahmadi2012entropic,follmer2004stochastic}.  Entropic risk, denoted as $R_X(t; \theta)$, can be viewed as the scaled cumulant generating function of $f(X; \theta)$, i.e., 

\begin{equation}
    R_X(t; \theta) := \frac{1}{t}\Lambda_X(t; \theta) = \frac{1}{t}\log \left( \dE\left[e^{tf(X; \theta)}\right]\right) = \frac{1}{t} \log \sum_x p(x) p_\theta(x)^{-t}. \label{eq:entropic_risk}
\end{equation}
{We note that entropic risk is usually defined over $t \in \mathbb{R}^{>0}$ in the literature~\citep{follmer2004stochastic}. In Eq.~\eqref{eq:entropic_risk} above, we naturally extend its definition to support $t \in \mathbb{R}$.} The TERM objective $\widetilde{R}(t;\theta)$ is {the} empirical version of entropic risk $R_{X}(t;\theta)$ ($t \in \mathbb{R}$). {One of the contributions of this work can be viewed as providing an operational meaning to the value of the (empirical) entropic risk and rigorously investigating its properties for $t \in \mathbb{R}^{<0}$.} In the next sections (Section~\ref{sec:background_renyi_ce}--Section~\ref{sec:background:dro}), we characterize various relations between tilted risks (TERM or entropic risk)  and other common risk measures, both in terms of the empirical variants (involving TERM) and distributional forms (involving entropic risk).

\subsection{TERM and R\'enyi Cross Entropy} \label{sec:background_renyi_ce}
We begin by demonstrating that TERM can be viewed as form of R\'enyi cross entropy minimization, which helps to explain the uniformity properties of TERM discussed in Section~\ref{sec:property:variance}. Consider the cross entropy between $p$ and $p_\theta$ defined by
\begin{equation}
\label{eq: CE-risk}
    H(p \| p_\theta) := \dE\left[f(X; \theta) \right] = \sum_x p(x) \log \left(\frac{1}{p_\theta(x)}\right).
\end{equation}
Hence, minimizing $\dE \left[f(X; \theta) \right]$ is equivalent to minimizing the cross entropy between the true distribution and the postulated distribution. The empirical variant of~\eqref{eq: CE-risk} would be empirical risk minimization~\eqref{eq: ERM}.

For $\rho \in \mathbb{R}^{>0},$ let R\'enyi cross entropy of order $\rho$ between $p$ and $q$ be defined as:\footnote{$H_1$ is defined via continuous extension.}
\begin{equation}
    H_\rho(p \| q ) := \frac{1}{1-\rho} \log \left(\sum_{x} p(x) q(x)^{\rho-1} \right).
\end{equation}
R\'enyi cross entropy can be viewed as a natural extension of cross entropy, and in fact it recovers cross entropy for $\rho =  1,$ i.e., $H_1(p\|q) = H(p\|q).$ R\'enyi cross-entropy can also be viewed as a natural extension of R\'enyi entropy, which it recovers when $p=q,$ i.e., $H_\rho(p\|p) = H_\rho(p),$ where 
R\'enyi entropy of order $\rho$ is defined as
\begin{equation}
    H_\rho(p) := \frac{1}{1-\rho} \log \left(\sum_{x}  p(x)^{\rho} \right).
\end{equation}

It is straightforward to see that the entropic risk can be expressed in terms of R\'enyi cross entropy:
\begin{equation}
    R_X(t; \theta) = H_{1-t}(p\|p_\theta).
\end{equation}
Equivalently, in the empirical world, TERM can be expressed as:
\begin{equation}
    \wR(t; \theta) = H_{1-t} (\mathbf{u} \| \mathbf{w}(1; \theta)),
\end{equation}
where $\mathbf{u}$ denotes the uniform $N$-vector and
$\mathbf{w}(1; \theta) := \left(w_1(1; \theta), \ldots, w_{n}(1; \theta)\right)$ with $w_i(1;\theta)$  defined in Eq.~\eqref{eq: w_i}, and for any two $N$-vectors $\mathbf{p}$ and $\mathbf{q},$
\begin{equation}
    H_\rho(\mathbf{p}\|\mathbf{q}) := \frac{1}{1-\rho} \log \left( \sum_{i \in N}p_i q_i^{\rho-1}\right).
\end{equation}
In other words, if we treat the loss $f(x_i; \theta)$ as log-likelihood of the sample $x_i$ under $p_\theta,$ this implies that TERM is the R\'enyi entropy of order $(1-t)$ between the uniform vector and the normalized likelihood vector of all samples, $\mathbf{w}(1; \theta)$. Hence, minimizing over $\theta$ is encouraging the \textit{uniformity} of $\mathbf{w}(1;\theta)$ in the sense of the R\'enyi cross entropy with the uniform vector.

\subsection{TERM as a Regularizer to Empirical Risk} \label{sec:background_CE}
TERM can also be interpreted as a form of regularization in traditional ERM. 
{We first note that by Taylor series expansion at $t=0$, TERM can be approximately decomposed into empirical risk regularized by $t$ times the empirical variance of the loss, for small $t$~\citep[Section V.A]{liu2019deep}. Here, we provide an exact interpretation of TERM as regularized ERM for all $t$.}
We first look at the distributional case, i.e.,  relating $R_X(t; \theta)$ to cross entropy as follows.
\begin{lemma}
\label{lem: entropic_risk_is_CE+KL}
The entropic risk of order $t$ can be stated as:
\begin{equation}
   R_X(t; \theta) =   H(p \| p_\theta) +  \frac{1}{t}D(p \| T(p, p_\theta, -t)),
\end{equation}
where $D$ denotes KL divergence between two distributions and $T(p, p_\theta, -t)$ is 
a mismatched tilted distribution  defined as~\cite[Definition 1]{salamatian2019mismatched}
\begin{equation}
    T(p, p_\theta, -t)(x) := \frac{p(x) p_\theta(x)^{-t}}{\sum_u p(u) p_\theta(u)^{-t}}.
\end{equation}
\end{lemma}
\begin{proof}
Consider the following equation:
\begin{equation}
    \sum_x p(x) \log\left( \frac{p(x)}{T(p, p_\theta, -t) (x)} \right)  = -t \sum_x p(x) \log \frac{1}{p_\theta(x)} + \log \left(\sum_x p(x) p_\theta(x)^{-t}\right),
\end{equation}
which directly implies the desired identity.
\end{proof}
In other words, entropic risk of order $t$ is equivalent to the cross entropy risk  regularized via a tilted mismatched distribution. 
Let $ \mathbf{w}(t; \theta) := \left(w_1(t; \theta), \ldots, w_{n}(t; \theta)\right)$ denote the tilted weight vector of the $n$ samples.
 Our next result is an empirical variant of Lemma~\ref{lem: entropic_risk_is_CE+KL}.
\begin{lemma}
TERM objective can be restated as follows:
 \begin{equation}
   \wR(t; \theta) =  \widebar{R}(\theta) +  \frac{1}{t}D(\mathbf{u} \| \mathbf{w}(t; \theta)),
\end{equation}
where $\widebar{R}(\theta)$ is the empirical risk~\eqref{eq: ERM}, $\mathbf{u}$ denotes the uniform $N$-vector, i.e., 
$
    \mathbf{u} := \left( \frac{1}{N}, \ldots, \frac{1}{N}\right),
$
and where for $N$-vectors $\mathbf{p}$ and $\mathbf{q},$
\begin{equation}
    D(\mathbf{p}\|\mathbf{q}) : = \sum_{i \in [N]} p_i \log \left( \frac{p_i}{q_i}\right).
\end{equation}
\end{lemma}
\begin{proof}
The proof is a consequence of the following identity:
\begin{equation}
   \frac{1}{t} \frac{1}{N} \sum_{i \in [N]}  \log\left( \frac{\frac{1}{N}}{w_i(t; \theta)} \right)  + \frac{1}{N}\sum_{i \in [N]}  f(x_i; \theta) =  \frac{1}{t}\log \left(\frac{1}{N}\sum_{i \in [N]} e^{tf(x_i; \theta)}\right).
\end{equation}
\end{proof}
Hence, TERM aims to minimize an average loss regularized by the KL divergence between the weight vector (which exponentially tilts the individual losses) and the uniform vector.

\subsection{TERM and Distributionally Robust Risks} \label{sec:background:dro}

Finally, we note that TERM is closely related to distributionally robust optimization (DRO) objectives~\citep[e.g.,][]{namkoong2017variance,duchi2019variance,chen2020distributionally,gurbuzbalaban2022stochastic,duchi2018learning}. In particular, 
TERM with $t > 0$ is equivalent to a form of DRO with a max-entropy regularizer, i.e., the constraint set is determined by a KL ball around uniform distribution~\citep{follmer2011entropic,qi2020attentional,shapiro2014lectures}: 
\begin{align}
    \widetilde{R}(t;\theta) &= \max_{q \in \Delta_N} \left\{\sum q_i f(x_i;\theta)-\frac{1}{t}\sum_{i \in [N]} q_i \log N q_i \right\} 
    = \max_{q \in \Delta_N} \left\{H(\mathbf{q}\|\mathbf{w}(1; \theta)) - \frac{1}{t} D(\mathbf{q}\| \mathbf{u})\right\},
\end{align}
and the corresponding relations in the distributional form is 
\begin{align}
    R_X(t;\theta) = \max_{q} \left\{ \dE_q[f(X;\theta)] - \frac{1}{t} D(q\|p)\right\} =  \max_{q} \left\{ H(q\|p_\theta) - \frac{1}{t} D(q\|p)\right\}.
\end{align}

This relation is also a special case of Donsker-Varadhan Variational Formula~\citep{dupuis1997weak}. We note that similar connections between DRO and TERM have also been explored in concurrent works by~\citet{qi2020practical,qi2020attentional} specifically in the limited context of stochastic optimization methods for solving class imbalance with $t>0$. 

In the next section, we propose a new risk motivated by TERM, which may be of independent interest.

\section{Tilted Value-at-Risk and Value-at-Risk} \label{sec:var}

In this section we provide connections between TERM and risk measures such as Value-at-Risk  (VAR) that specifically target loss quantiles. 
In particular, based on TERM, we propose a new risk---\textit{Tilted Value-at-Risk (TiVaR)} and discuss its relations with existing risks (Section~\ref{sec:var:approximate}). We find that TiVaR is a computationally efficient alternative to VaR that provides tighter approximations to VaR than prior risks, which again helps to motivate the use of TERM. 

\subsection{{Tail Probabilities of Losses} and Value-at-Risk (VaR)} \label{sec:var:def}

{The tail probabilities of losses focus on quantiles of losses that exceed a certain threshold, as formally defined below.} 
\begin{definition}[{Tail probability of losses}]
\label{def: superquantile-losses}
For all $\gamma \in \mathbb{R},$ let $Q_X(\gamma;\theta)$ denote the probability of the losses $f(X;\theta)$ no smaller than $\gamma$, i.e., 
\begin{align}
    Q_X(\gamma;\theta) := P\left[f(X;\theta) \geq \gamma\right].
\end{align}
Equivalently, define the empirical variant $\wQ(\gamma; \theta)$ over samples ${x_i}$ for ${i \in [N]}$:
\begin{equation}
    \wQ(\gamma; \theta):= \frac{1}{N} \sum_{i \in [N]}  \mathbb{I}\left\{ f(x_i;\theta) \geq \gamma \right\}
\end{equation}
where $\mathbb{I}\{\cdot\}$ is the indicator function.
\end{definition}

Notice that $\wQ(\gamma; \theta) \in \left\{0, \frac{1}{N}, \ldots, 1\right\}$ quantifies the fraction of the data for which loss is at least $\gamma$. 
For example, optimizing for 90\% of the individual losses (ignoring the worst-performing 10\%) 
could be a more reasonable practical objective than the pessimistic min-max objective. Another common application of this is to use the median in contrast to the mean in the presence of noisy outliers. 

Using tail distribution of losses, Value-at-Risk (VaR)~\citep{jorion1996value} with confidence $\alpha$ ($0<\alpha<1$) is defined as 
\begin{align}
    \text{VaR}_X(1-\alpha;\theta) 
    := \min_{\gamma} \left\{\gamma: Q_X(\gamma;\theta) \leq \alpha \right\},
\end{align}
and the empirical variant for $\alpha \in \{\frac{k}{N}\}_{k \in [N]}$ is
\begin{align}
    \widetilde{\text{VaR}}(1-\alpha;\theta) 
    &:= \min_{\gamma} \left\{\gamma : \wQ(\gamma;\theta) \leq \alpha  \right\}.
\end{align}

{Notice that when we view the loss as log-likelihood of a parametric probability distribution function, $Q_X(\gamma; \theta)$ (Definition~\ref{def: superquantile-losses}) can be viewed as the complementary cumulative distribution function (CDF) of the information random variable $f(X; \theta)$. Given the definition of VaR, $Q_X(\gamma; \theta)$ can also be viewed as `inverted' VaR, as we formalize and prove in Lemma~\ref{lemma:var_Q_dis} and~\ref{lemma:var_Q_emp} below.}
Let
\begin{align}
\label{eq: def-theta-Q}
 Q^0_X (\gamma) &: = \min_\theta Q_X(\gamma; \theta),  \quad \theta_X^0(\gamma)~ \new{\in} \arg\min_{\theta}Q_X(\gamma;\theta), \\
\widetilde{Q}^0(\gamma) &:= \min_{\theta} \wQ(\gamma;\theta), \quad \theta^0(\gamma) ~\new{\in} \arg \min_{\theta} \wQ(\gamma;\theta) .
\end{align}
where $Q_X$ and $\widetilde{Q}$ is defined in Definition~\ref{def: superquantile-losses}.
Optimizing $\wQ(\gamma;\theta)$ is equivalent to optimizing VaR. Formally, we have the following lemmas. 

\begin{lemma} \label{lemma:var_Q_dis}
Assume $\min_{\theta} Q_X(\gamma;\theta)$ is strictly decreasing with $\gamma$.
We note
\begin{align}
    \min_{\theta} \left\{ \text{VaR}_{X}(1-Q_X^0(\gamma);\theta)\right\}  = \gamma, \quad \arg\min_{\theta} \left\{ \text{VaR}_{X}(1-Q_X^0(\gamma);\theta)\right\} \new{\ni}~ \theta_X^0(\gamma).
\end{align}
\end{lemma}
{Note that $\min_{\theta} Q_X(\gamma;\theta)$ is non-increasing as $\gamma$ increases by definition. The additional strict monotonic assumption on $\gamma \mapsto \min_{\theta} Q_X(\gamma;\theta)$  can be easily satisfied if $f(X; \theta)$ is a continuous random variable and $\gamma$ is in the range of $f$. Lemma~\ref{lemma:var_Q_dis} is proved as follows.}
\begin{proof}
    First, we note for any $\theta$ and $\gamma_0$ such that $Q_X(\gamma_0;\theta) \leq Q_X^0(\gamma)$,  we have $\gamma_0 \geq \gamma$.
    Otherwise, there exist $\theta', \gamma' < \gamma$ and $Q_X(\gamma'; \theta') \leq Q_X^0(\gamma),$
    which in turn implies that
    \begin{align}
       \min_{\theta} P[f(X;\theta) \geq \gamma'] \leq P[f(X;\theta') \geq \gamma'] \leq Q_X^0(\gamma),
    \end{align}
    contradicting $Q_X^0(\gamma') > Q_X^0(\gamma)$.
The proof completes combining with the fact that the function value of $\text{VaR}_{X}(1-Q_X^0(\gamma);\theta)$ can achieve $\gamma$ at any $\theta_X^0(\gamma)$.
\end{proof} 
Lemma~\ref{lemma:var_Q_emp} below describes the empirical variant, {which does not require the strict monotonic assumption.}
\begin{lemma} \label{lemma:var_Q_emp}
For any $\gamma \in (\wF(-\infty), \wF(+\infty))$ where $\wF(t)$ is defined as the optimal tilted objective as in Eq.~\eqref{def:wF}, let $\gamma^0 = \min \left\{\gamma' | \wQ^0(\gamma') = \wQ^0(\gamma) \right\}$. Then
   \begin{align}
       \min_{\theta} \left\{\widetilde{\text{VaR}}(1-
       \widetilde{Q}^0(\gamma^0); \theta)\right\} = \gamma^0 , \quad \arg\min_{\theta} \left\{\widetilde{{\text{VaR}}}(1-\widetilde{Q}^0(\gamma^0);\theta)\right\} \new{\ni}~ \theta^0(\gamma^0).
   \end{align}
\end{lemma}

Both tail distribution of losses and VaR are usually non-smooth and non-convex, and solving them to global optimality is very challenging. In the next section, we show that TiVaR (an objective based on TERM) provides a good upper bound on VaR, and is computationally more efficient, as VaR is not even continuous. 
In parallel, in Appendix~\ref{app:superquantile}, we prove that TERM also provides a reasonable approximate solution to {the minimizer of tail probability of losses (i.e., inverted VaR).} 

{The proof of one of the main theorems of this section (Theorem~\ref{thm:Q}) relies on a new variant of Chernoff bound for non-negative random variables, which may be of independent interest.} 
{
\begin{theorem}[Chernoff bound for non-negative random variables]
Let $X$ be a non-negative random variable. Further assume that $E\left[e^{tX}\right] < \infty$ for all $t \in \mathbb{R}$. Then for $\gamma > 0$,
\begin{equation}
    P[X \geq \gamma] \leq \inf_{t \in \mathbb{R}} \left\{ \frac{E\left[e^{tX}\right] - 1}{e^{t \gamma} - 1}\right\}  \leq \inf_{t \in \mathbb{R}^{+}} \left\{ \frac{E\left[e^{tX}\right]}{e^{ t \gamma}}\right\},
\end{equation}
where the latter term is the generic Chernoff bound with $\gamma >0$.
\label{thm: new-chernoff}
\end{theorem} }
{
\begin{proof}
The theorem holds by applying Markov's inequality twice on $e^{tX}-1~(t \geq 0)$ and $1- e^{tX}~(t < 0)$, and noting that
\begin{align}
 P[X \geq \gamma] \leq \min\left\{ \inf_{t \in \mathbb{R}^{\geq 0}} \left\{ \frac{E\left[e^{tX}\right] - 1}{e^{t \gamma} - 1}\right\},   \inf_{t \in \mathbb{R}^-} \left\{\frac{1-E\left[e^{tX}\right] }{1-e^{t \gamma}}\right\} \right\} = \inf_{t \in \mathbb{R}} \left\{ \frac{E\left[e^{tX}\right] - 1}{e^{t \gamma} - 1}\right\}.
\end{align}
\end{proof}}
{Theorem~\ref{thm: new-chernoff} presents a tighter Chernoff bound for non-negative random variables. To the best of our knowledge, despite the fact that this bound is a simple extension of the generic Chernoff bound, and the existing variants of Chernoff bounds in prior works~\citep{boucheron2013concentration,yang2017complexity}, we have not seen the result we have here appear elsewhere in this form. In particular, notice that the search for an optimal value of $t$ has been extended from non-negative values to all real numbers. This can result in significantly tighter bounds, especially in small deviations regime, as  visualized empirically on two simple distributions in Figure~\ref{fig:new_chernoff_bound}, Appendix~\ref{app:superquantile}. We will see how this leads to significantly better bounds in robustness applications.} 

\subsection{TiVaR: Tilted Value-at-Risk} \label{sec:var:approximate}

{In this section, we introduce a new risk measure, called Tilted Value-at-Risk (TiVaR). To put TiVaR in perspective, we briefly state other existing risks first.} 
Conditional Value-at-Risk (CVaR) minimizes the average risk of tail events where the risk is above some threshold~\citep{rockafellar2000optimization,rockafellar2002conditional}. One form of CVaR is
\begin{align} \label{eq: cvar}
    \text{CVaR}_X(1-\alpha;\theta) := \min_{\gamma} \left\{\gamma+\frac{1}{\alpha} \dE [f(X;\theta)-\gamma]_{+}\right\}.
\end{align} 
It is worth noting that CVaR$_X(1-\alpha;\theta)$ is a dual formulation of DRO with an uncertainty set that perturbs arbitrary parts of the data by an amount up to $\frac{1}{\alpha}$~\citep{rockafellar2000optimization, curi2019adaptive}. 
Formally, the dual of DRO $\max_{Q: \left\{\frac{dQ}{dP} \leq \frac{1}{\alpha}\right\}}\mathbb{E}_{Q}[f(X; \theta)]$ is $\text{CVaR}_{X}(1-\alpha; \theta) = \min_{\gamma} \left\{\gamma+\frac{1}{\alpha} \mathbb{E}[f(X;\theta)-\gamma]_{+}\right\}$.
Some previous works implicitly minimize CVaR by only training on samples with top-$k$ losses~\citep[e.g.,][]{NIPS2017_6c524f9d}. Entropic Value-at-Risk (EVaR) is proposed as an upper bound of CVaR and VaR that could be more computationally efficient~\citep{ahmadi2012entropic}. EVaR with a confidence level $\alpha$ $(0< \alpha <1)$ is defined as:
\begin{align}
     \text{EVaR}_{X}(1-\alpha;\theta) := \min_{t\in \mathbb{R}^{>0}} \left\{ \frac{1}{t} \log \left( \frac{\mathbb{E}[e^{tf(X; \theta)}]}{\alpha}\right) \right\}= \min_{t \in \mathbb{R}^{>0}} \left\{R_X(t;\theta)-\frac{1}{t}\log \alpha \right\}.
\end{align}
Similarly, for $\alpha \in \{\frac{k}{N}\}_{k \in [N]}$, the empirical variants of  CVaR and EVaR are
\begin{align*}
    \widetilde{\text{CVaR}}(1-\alpha;\theta) &:= \min_{\gamma} \left\{\gamma+\frac{1}{\alpha}\frac{1}{N} \sum_{i \in [N]}[f(x_i;\theta)-\gamma]_{+}\right\}, \\
    \widetilde{\text{EVaR}}(1-\alpha;\theta) &:= \min_{t\in \mathbb{R}^{>0}} \left\{ \frac{1}{t} \log \left( \frac{\frac{1}{N} \sum_{i \in [N]} e^{tf(x_i; \theta)}}{\alpha}\right)\right\} =\min_{t\in \mathbb{R}^{>0}} \left\{ \wR(t; \theta) - \frac{1}{t}  \log \alpha \right\} .
\end{align*}
Notice that TERM objective appears as part of the objective in $\widetilde{\text{EVaR}}$, and particularly optimizing $\widetilde{\text{EVaR}}$ with respect to $\theta$ would be equivalent to solving TERM for some value of $t$ implicitly defined through $\alpha$ (see Lemma~\ref{lem:ER-EVaR relation} and Lemma~\ref{lem:ER-EVaR relation empirical} in the appendix). 

It is known that $\text{VaR}_{X}(1-\alpha;\theta) \leq \text{CVaR}_{X}(1-\alpha;\theta) \leq \text{EVaR}_{X}(1-\alpha;\theta)$~\citep{ahmadi2012entropic}
which directly yields $
    \widetilde{\text{VaR}}(1-\alpha;\theta) \leq \widetilde{\text{CVaR}}(1-\alpha;\theta) \leq \widetilde{\text{EVaR}}(1-\alpha;\theta)$.
Meanwhile, to the best of our knowledge, it is not clear from existing works how entropic risk (or TERM) is related to VaR or EVaR. Next, based on TERM, we propose a new risk-averse objective Tilted Value-at-Risk, showing that it upper bounds VaR and lower bounds EVaR.

\begin{definition}[Tilted Value-at-Risk (TiVaR)]
Let TiVaR for $\alpha \in (0, 1]$ be defined as 
\begin{equation} \label{def:TiVaR-distribution}
  \text{TiVaR}_X(1-\alpha; \theta) := \min_{t \in \mathbb{R}} \left\{ F_X(-\infty) +  \frac{1}{t} \log  \left[\frac{e^{(R_X(t; \theta) - F_X(-\infty))t } - (1-\alpha)}{\alpha}  \right]_+\right\}.
\end{equation}
Similarly, empirical TiVaR is defined for $\alpha \in (0, 1)$, 
\begin{equation} \label{def:TiVaR-empirical}
  \widetilde{\text{TiVaR}}\left(1-\alpha; \theta\right) := \min_{t \in \mathbb{R}} \left\{ \wF(-\infty) +  \frac{1}{t} \log  \left[\frac{e^{(\wR(t; \theta) - \wF(-\infty))t } - (1-\alpha)}{\alpha}  \right]_+\right\}.
\end{equation}
\end{definition}
We note that TiVaR is not a coherent risk measure (see the work of~\citet{artzner1997thinking,artzner1999coherent} for definition of coherent risks), despite that it can be tighter than CVaR in some cases, as discussed in detail later. We next present our main result on relations between TiVaR, VaR, and EVaR.
\begin{theorem} \label{thm:var-TiVaR-cvar}
For $\alpha \in (0,1]$ and any $\theta$,
\begin{equation}
    \text{VaR}_X(1-\alpha; \theta) \leq \text{TiVaR}_X(1-\alpha; \theta) \leq \text{EVaR}_X(1-\alpha; \theta).
\end{equation}
Similarly, for $\alpha \in \{\frac{k}{N}\}_{k \in [N]}$ and any $\theta$, 
\begin{equation}
    \widetilde{\text{VaR}}\left(1-\alpha; \theta\right) \leq \widetilde{\text{TiVaR}}\left(1-\alpha; \theta\right) \leq \widetilde{\text{EVaR}}\left(1-\alpha; \theta\right).
\end{equation}
\end{theorem}
We defer the proof to Appendix~\ref{app:superquantile}, {where the main steps include applying the new Chernoff bound variant (Theorem~\ref{thm: new-chernoff}).} Theorem~\ref{thm:var-TiVaR-cvar} indicates that $\widetilde{\text{TiVaR}}\left(1-\alpha; \theta\right)$ is a tighter approximation to $\widetilde{\text{VaR}}\left(1-\alpha; \theta\right)$ than $\widetilde{\text{EVaR}}\left(1-\alpha; \theta\right)$. 

\paragraph{Comparing TiVaR and CVaR.}
In general, TiVaR and CVaR are not directly comparable,
as both of them can be viewed as approximations to VaR  and neither dominates the other,
i.e., one risk can be tighter than the other depending on the quantile value, 1-$\alpha$. In
the regimes where $\alpha$ is a large value between some intermediate constant and 1, TiVaR provides a tighter approximation to VaR than CVaR. For instance, in the extreme case when $\alpha \to 1$, $\widetilde{\text{VaR}}$ will be close to the min-loss ($\min_{i \in [N]} f(x_i;\theta)$), while the value of $\widetilde{\text{CVaR}}$ is the mean of the losses ($\frac{1}{N}\sum_{i \in [N]} f(x_i;\theta)$). $\widetilde{\text{TiVaR}}$ reduces to the min-loss in this case. In other words, both $\widetilde{\text{VaR}}$ and $\widetilde{\text{TiVaR}}$  sweep the values between the min-loss and max-loss; whereas $\widetilde{\text{CVaR}}$ sweeps the values between the avg-loss and max-loss.
We compare TiVaR with CVaR and other risks in Figure~\ref{fig:super-quantile-2} on mean estimation and linear regression problems, and demonstrate that TiVaR is tighter than CVaR especially when $\alpha$ is close 1 (corresponding to robustness applications).\footnote{\new{While CVaR focuses on upper quantiles, one may explore `inverse' CVaR to better approximate the lower quantiles. However, inverse CVaR, ranging from avg-loss to min-loss, is not a valid upper bound of VaR. Despite this, we empirically explore this approximation to solving VaR, among others, in Appendix~\ref{app:superquantile}.}}


\new{We also note that there exist other risk-averse or risk-seeking formulations that focus on the upper or lower tail of losses, such as the mean-semideviation framework~\citep{kalogerias2018recursive}.  Mean-semideviation recovers a set of risk measures including mean-upper-semideviations and entropic mean-semideviation. Nevertheless, these risks usually cannot handle both fairness and robustness in a single formulation, and can incur more per-iteration gradient evaluations or worse convergence rates compared to vanilla ERM~\citep{kalogerias2018recursive, gurbuzbalaban2022stochastic,zhu2023distributionally}.}
\begin{figure}[t!]
    \centering
    \begin{subfigure}{0.47\textwidth}
        \centering
        \includegraphics[width=\textwidth]{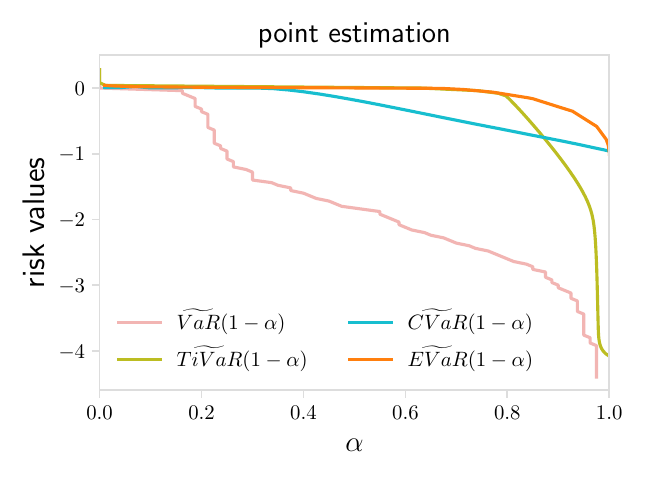}
    \end{subfigure}
    \hfill
    \begin{subfigure}{0.47\textwidth}
        \centering
        \includegraphics[width=\textwidth]{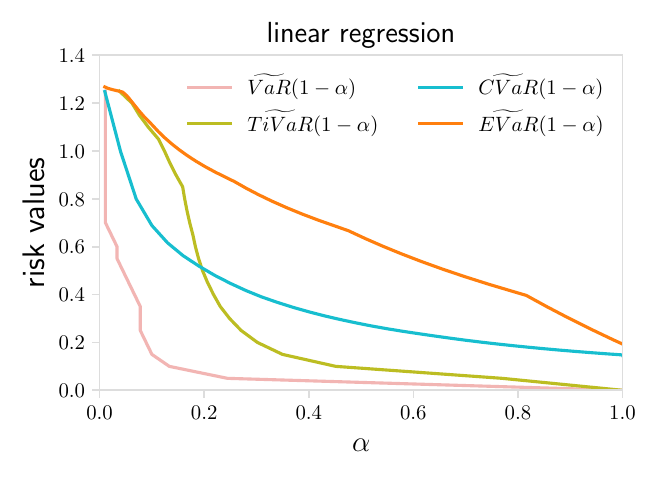}
    \end{subfigure}
    \caption{Comparing values of VaR, TiVaR, CVaR, and EVaR.   $\widetilde{\text{VaR}} (1-\alpha) := \min_{\theta} \widetilde{\text{VaR}}(1-\alpha;\theta)$, and $\widetilde{\text{TiVaR}} (1-\alpha)$, $\widetilde{\text{CVaR}} (1-\alpha)$, and $\widetilde{\text{EVaR}} (1-\alpha)$ are defined in a similar way. From Theorem~\ref{thm:var-TiVaR-cvar}, we know $\widetilde{\text{VaR}}(1-\alpha;\theta) \leq \widetilde{\text{TiVaR}}(1-\alpha;\theta) \leq \widetilde{\text{EVaR}}(1-\alpha;\theta)$, which is also visualized here. Both CVaR and TiVaR values are between VaR and EVaR. TiVaR provides a tighter approximation to VaR than CVaR when $\alpha$ is closer to 1. 
    }
\label{fig:super-quantile-2}
\end{figure}

Finally, we draw connections between the above results and the $k$-loss, defined as the $k$-th smallest loss of $N$ (i.e., $1$-loss is the min-loss, $N$-loss is the max-loss, $\small (N{-}1)/2$-loss is the median-loss). 
Formally, let $R_{(k)}(\theta)$ be the $k$-th order statistic of the loss vector. Hence, $R_{(k)}$ is the $k$-th smallest loss, and particularly 
\begin{align}
    R_{(1)}(\theta) &= \widecheck{R}(\theta), \quad
    R_{(N)} (\theta) = \widehat{R}(\theta). 
\end{align}
Thus, for any $k \in [N],$ we define
\begin{equation}
    R^*_{(k)} := \min_{\theta} R_{(k)}(\theta), \quad
    \theta^*(k) := \arg\min_{\theta} R_{(k)}(\theta).
\end{equation}
Note that 
\begin{align}
    R^*_{(1)}  &= \wF(-\infty), \quad
    R^*_{(N)}  = \wF(+\infty).
\end{align}

While minimizing the $k$-loss is more desirable than ERM in many applications, 
the $k$-loss is non-smooth (and generally non-convex), and is challenging to solve for large-scale problems~\citep{jin2019minmax, nouiehed2019solving}. TERM offers a good approximation to $k$-loss as well. Note that if we fix $\alpha=1-\frac{k}{N}$, minimizing $k$-loss is equivalent to minimizing $\gamma$ where $\wQ(\gamma;\theta) = \alpha$. Based on the bound of $\widetilde{\text{VaR}}$, we obtain a bound on $k$-loss:

\begin{corollary}\label{coro:k-loss-2}
For all $k \in \{2, \ldots, N-1\},$ and all $t \in \mathbb{R}:$
\begin{equation}
    R_{(k)} (\theta) \leq \min_t \left\{\wF(-\infty) + \frac{1}{t} \log  \left[\frac{e^{(\wR(t; \theta) - \wF(-\infty))t } - \frac{k}{N}}{1- \frac{k}{N}}  \right]_+ \right\} \leq \min_{t\in \mathbb{R}^{>0}} \left\{ \wR(t;\theta)-\frac{1}{t} \log \left(1-\frac{k}{N}\right)\right\}.
\end{equation}
\end{corollary}
\begin{proof}
Note that
\begin{equation}
   R_{(k)} (\theta) = \widetilde{\text{VaR}}\left(\frac{k}{N}; \theta \right). 
\end{equation}
The proof completes by setting $\alpha=1-\frac{k}{N}$ in Eq.~\eqref{def:TiVaR-empirical} and noting $\widetilde{VaR}(1-\alpha;\theta) \leq \widetilde{TiVaR}(1-\alpha;\theta) \leq \widetilde{EVaR}(1-\alpha;\theta)$ .
\end{proof}
Corollary~\ref{coro:k-loss-2} optimizes over all $t \in \mathbb{R}$ over the upper bound of $R_{(k)}(\theta)$, which can be relaxed to searching over positive $t$'s, as stated in Corollary~\ref{coro:k-loss} below.

\begin{corollary}\label{coro:k-loss}
For all $k \in \{2, \ldots, N-1\},$ and all $t \in \mathbb{R}^{>0}:$
\begin{equation}
    R_{(k)} (\theta) \leq \wF(-\infty) + \frac{1}{t} \log  \left(\frac{e^{(\wR(t; \theta) - \wF(-\infty))t } - \frac{k}{N}}{1- \frac{k}{N}}  \right).
\end{equation}
\end{corollary}

\section{Solving TERM}\label{sec:solver}

In this section,  we 
develop first-order batch (Section~\ref{sec:solver:batch}) and stochastic (Section~\ref{sec:solver:stochastic})  optimization methods for solving TERM, and rigorously analyze the effects that $t$ has on the convergence of these methods.

Recall that in Section~\ref{sec:property:general}, we discuss the Lipschitzness, convexity, and smoothness properties of TERM. 
$t$-tilted loss remains strongly convex for $t>0,$ so long as the original loss function is strongly convex. On the other hand, for sufficiently large negative $t$, the $t$-tilted loss becomes non-convex.
Hence, while the $t$-tilted solutions for positive $t$ are unique, the objective may have multiple (spurious) local minima for negative $t$ even if the original loss function is strongly convex. 
For negative $t$, we seek the solution for which the parametric set of $t$-tilted solutions obtained by sweeping $t \in \mathbb{R}$ {(i.e., $\breve{\theta}(t)$ defined in Eq.~\eqref{eq: opt_obj})} remains continuous (as in Figure~\ref{fig:example}a-c {and Figure~\ref{fig:toy1}}). To this end, for negative $t$, we solve TERM by smoothly decreasing $t$ from $0$ \new{observing} that the solutions form a continuum in $\mathbb{R}^d$ \new{empirically}. Despite the non-convexity of TERM with $t < 0$, we find that this approach produces effective solutions to multiple real-world problems in Section~\ref{sec:experiments}. Additionally, as the objective remains smooth, it is still relatively efficient to solve. {On the toy problem studied in Figure~\ref{fig:toy1}}, we plot the convergence with $t$ in Figure~\ref{fig:solving} below.

\begin{figure}[h!]
  \centering
  \includegraphics[trim=0 0 0 8,clip, width=0.45\textwidth]{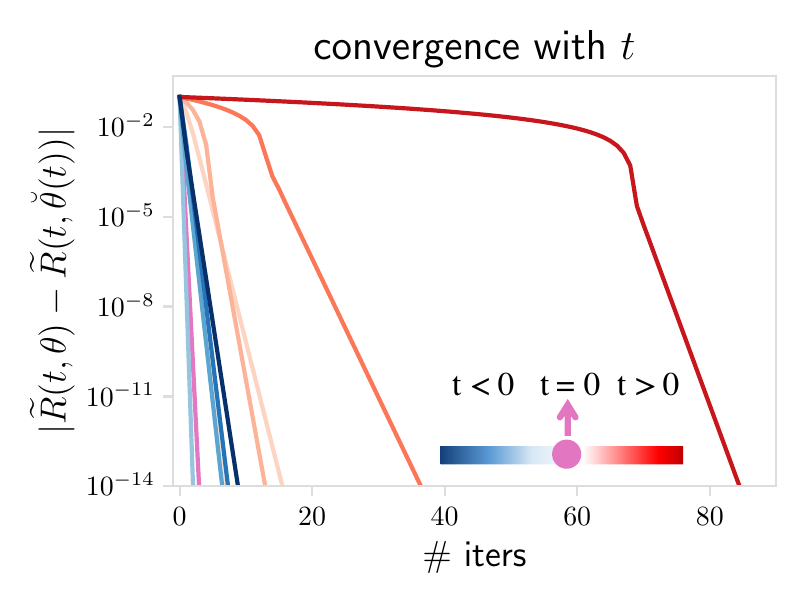}
  \vspace{-1em}
  \caption{As $t\to +\infty$, the objective becomes less smooth in the vicinity of the final solution {where smoothness can be measured by the upper bound of Hessian (see Lemma~\ref{lemma:TERM-smoothness})}, hence suffering from slower convergence. For negative values of $t$, TERM converges fast due to the smoothness in the vicinity of solutions despite its non-convexity. }
  \label{fig:solving}
\end{figure}

\subsection{First-Order Batch Methods} \label{sec:solver:batch}

TERM solver in the batch setting is summarized in Algorithm~\ref{alg:batch-non-h-TERM}. The main steps include {running gradient descent on $\wR(t;\theta)$}, which involve computing the tilted gradients (i.e., a weighted aggregation of individual gradients (Lemma~\ref{lemma:TERM-gradient})) of the objective.  
We also provide convergence results in Theorem~\ref{thm: convergence}--\ref{thm: convergence-batch-NCS} below for Algorithm~\ref{alg:batch-non-h-TERM}.

\begin{algorithm}[h]
\SetKwInOut{Init}{Initialize}
\SetAlgoLined
\DontPrintSemicolon
\SetNoFillComment
\KwIn{$t, \alpha, \theta$}
\While{stopping criteria not reached}{

compute the loss $f(x_i; \theta)$ and gradient $\nabla_\theta f(x_i; \theta)$ for all $i \in [N]$\;
$\wR(t; \theta) \gets \text{$t$-tilted loss~\eqref{eq: TERM}}$ on all $i \in [N]$\;
$w_i(t; \theta) \gets e^{t(f(x_i; \theta) - \wR(t; \theta))}$\;
$\theta \gets \theta - \frac{\alpha}{N} \sum_{i \in [N]} w_i(t; \theta) \nabla_{\theta}f(x_i; \theta)$\;
}
\caption{Batch {(}Non-Hierarchical{)} TERM}\label{alg:batch-non-h-TERM}
\end{algorithm}

{
\begin{theorem}[Convergence of Algorithm~\ref{alg:batch-non-h-TERM} for strongly-convex problems] \label{thm: convergence}
Under Assumption~\ref{assump: regularity}, there exist $\beta_{\max} \leq C_1 < \infty$ and $C_2<\infty$ 
that do not depend on $t$ such that for any $t\in \mathbb{R}^{>0},$ setting the step size $\alpha = \frac{1}{C_1 + C_2t},$ after $k$ iterations: 
\begin{equation}
    \wR(t, \theta_k) - \wR(t, \breve{\theta}(t)) \leq \left(1 - \frac{\beta_{\min}}{C_1 + C_2t}\right)^k \left(  \wR(t, \theta_0) - \wR(t, \breve{\theta}(t)) \right).
\end{equation}
\end{theorem}
\begin{proof}
First note that by Lemma~\ref{lem: Hessian}, $\wR(t, \theta)$ is $\beta_{\min}$-strongly convex for all $t\in \mathbb{R}^{>0}.$ Next, by Lemma~\ref{lemma:TERM-smoothness}, there exist $C_1, C_2 <\infty$ such that $\wR(t; \theta)$ has $(C_1 + C_2t)$-Lipschitz gradients for all $t \in \mathbb{R}^{>0}.$ The result follows directly from~\citet[Theorem 1]{karimi2016linear}.
\end{proof}
}
Note that under additional assumptions on $L$-Lipschitzness of $f(x;\theta),$ we can plug in the explicit smoothness constants established by~\citet[Lemma 5.3]{lowy2021output} to obtain explicit constants in the convergence rate, i.e., $C_1=\beta_{\max}$ and $C_2 =L^2$.
Theorem~\ref{thm: convergence} indicates that solving TERM to a local optimum using gradient-based methods tends to be as efficient as traditional ERM for small-to-moderate values of $t$~\citep{jin2017escape}, which we corroborate via experiments on multiple real-world datasets in Section~\ref{sec:experiments}. This is in contrast to solving for the min-max solution, which would be similar to solving TERM as $t \to +\infty$~\citep{kort1972new,pee2011solving, ostrovskii2020efficient}.


\begin{theorem}[Convergence of Algorithm~\ref{alg:batch-non-h-TERM} for  smooth problems satisfying PL conditions] \label{thm: convergence-batch-NCSPL}
Assume $f(x;\theta)$ is $\beta_{\max}$-smooth and (possibly) non-convex. Further assume $\sum_{i\in [N]} p_i f(x_i; \theta)$ is $\frac{\mu}{2}$-PL for any $\mathbf{p} \in \Delta_{N}$ where $\mathbf{p} := (p_1, \dots, p_N)$. There exist $\beta_{\max} \leq C_1 < \infty$ and $C_2<\infty$ that do not depend on $t$ such that for any $t\in \mathbb{R}^{>0},$ setting the step size $\alpha = \frac{1}{C_1 + C_2t},$ after $k$ iterations: 
\begin{equation}
    \wR(t, \theta_k) - \wR(t, \breve{\theta}(t)) \leq \left(1 - \frac{\mu}{C_1 + C_2t}\right)^k \left(  \wR(t, \theta_0) - \wR(t, \breve{\theta}(t)) \right),
\end{equation}
\end{theorem}
\begin{proof}
If $\sum_{i\in [N]} p_i f(x_i; \theta)$ is $\mu$-PL for any $\mathbf{p} \in \Delta_{N}$, then $\wR(t;\theta)$ is $\mu$-PL~\citep{qi2020practical}. $\wR(t;\theta)$ is  $\beta_{\max}$ smooth for $t<0$ and its smoothness parameter scales linearly with $t$ for $t>0$, following the same proof as Lemma~\ref{lemma:TERM-smoothness}. 
\end{proof}
{Theorem~\ref{thm: convergence-batch-NCSPL} applies to both convex and non-convex smooth functions satisfying PL conditions.} Again, here we can plug in explicit smoothness parameter~\citep[Lemma 5.3]{lowy2021output} if $f(x;\theta)$ is Lipschitz. We next state results without the PL condition assumption for completeness. 

\begin{theorem}[Convergence of Algorithm~\ref{alg:batch-non-h-TERM} for non-convex smooth problems] \label{thm: convergence-batch-NCS}
Assume $f(x;\theta)$ is $\beta_{\max}$-smooth and (possibly) non-convex. Setting the step size $\alpha = \frac{1}{\beta(t)},$ after $K$ iterations, we have:
\begin{equation}
    \frac{1}{K}\sum_{k=0}^{K-1} \| \nabla \wR(t, \theta_k) \|^2 \leq \frac{2\beta(t) (\wR(t, \theta_0) - \wR(t, \breve{\theta}(t)))}{K},
\end{equation}
where for $t \in \mathbb{R}^{>0}$, $\beta(t)=C_1+Ct$ where $C_1, C_2$ are  independent of $t$ and $\beta_{\max} \leq C_1 < \infty, C_2 < \infty$, and for $t \in \mathbb{R}^-$, $\beta(t)=\beta_{\max}$. 
\end{theorem}

Theorem~\ref{thm: convergence-batch-NCS} also covers the case of convex $f(x;\theta)$ with $t<0$. We note that for non-convex problems, when $t<0$, the convergence rate is independent of $t$ under our assumptions. We also observe this on a toy problem in Figure~\ref{fig:solving}. In all applications we studied in Section~\ref{sec:experiments} with negative $t$'s, TERM runs the same number iterations as those of ERM. 

\subsection{First-Order Stochastic Methods} \label{sec:solver:stochastic}
To obtain unbiased stochastic gradients, we need to have access to the normalization weights for each sample (i.e., $\frac{1}{N}\sum_{i \in [N]} e^{tf(x_i;\theta)}$), which is often intractable to compute for large-scale problems. Hence, we use $\doublewidetilde{R}_{t}$, a term that incorporates  stochastic dynamics, to estimate the tilted objective $\wR_{t} := \wR(t;\theta)$, which is used for normalizing the weights as in~\eqref{eq: tilted_grad}. In particular, we do not use a trivial linear averaging of the current estimate and the history to update $\doublewidetilde{R}_{t}$. Instead, we use a tilted averaging to ensure an unbiased estimator (if $\theta$ is not being updated).

On the other hand, the TERM objective can be viewed as a composition of functions $\frac{1}{N} \sum_{i \in [N]} e^{tf(x_i; \theta)}$ and $\frac{1}{t} \log(\cdot)$, and could be optimized based on previous stochastic compositional optimization techniques~\citep[e.g.,][]{wang2017stochastic,qi2020attentional,qi2020practical,wang2016accelerating,ghadimi2020single}. Similar to~\citet{wang2017stochastic}, we maintain two sequences (in our context, the model $\theta$ and the objective estimate $\doublewidetilde{R}_{t}$) throughout the optimization process. This (non-hierarchical) stochastic algorithm is summarized in Algorithm~\ref{alg:stochastic-non-h-TERM} below.

For the purpose of analysis, we sample two independent mini-batches to obtain the gradient of the original loss functions $\nabla_{\theta} f(x;\theta)$ and update $\doublewidetilde{R}_t$, respectively (described in Algorithm~\ref{alg:stochastic-non-h-TERM-two-batch} for completeness). As we will see in Theorem~\ref{thm: convergence_stochastic_convex}, the additional randomness allows us to achieve better convergence rates compared with the algorithm proposed in~\citet{wang2017stochastic} instantiated to our objective. Our rate of this simple algorithm matches the rate of more complicated ones~\citep{qi2020practical}, and developing optimal optimization procedures is out of the scope of this work.
Empirically, we observe that sampling two mini-batches yield similar performance as using the same mini-batch to query the individual losses and the weights (Figure~\ref{fig:convergence_compare} in Appendix~\ref{app: solving-TERM}). Therefore, we employ the cheaper variant of just involving one mini-batch (Algorithm~\ref{alg:stochastic-non-h-TERM}) in the corresponding experiments.

\begin{algorithm}[h]
\SetKwInOut{Init}{Initialize}
\SetAlgoLined
\DontPrintSemicolon
\SetNoFillComment
\Init{$\theta, \doublewidetilde{R}_{t} = \frac{1}{t} \log\left(\frac{1}{N}\sum_{i \in [N]} e^{tf(x_i;\theta)}\right)$}
\KwIn{$t, \alpha, \lambda$}
\While{stopping criteria not reached}{
sample a minibatch $B$ uniformly at random from $[N]$\; compute the loss $f(x; \theta)$ and gradient $\nabla_\theta f(x; \theta)$ for all $x \in B$\;
$\wR_{B, t} \gets \text{$t$-tilted loss~\eqref{eq: TERM} on minibatch $B$}$\;
$\doublewidetilde{R}_{t} \gets \frac{1}{t} \log \left((1-\lambda) e^{t\doublewidetilde{R}_{t}} + \lambda e^{t \wR_{B, t}}\right)$\; $w_{t, x} \gets e^{t f(x; \theta) - t\doublewidetilde{R}_{t}}$\;
$\theta \gets \theta - \frac{\alpha}{|B|} \sum_{x\in B} w_{t, x} \nabla_{\theta} f(x; \theta)$\;
}
\caption{Stochastic {(}Non-Hierarchical{)}  TERM}\label{alg:stochastic-non-h-TERM}
\end{algorithm}

The stochastic algorithm developed here requires roughly the same
time/space complexity as mini-batch SGD, and thus scales similarly for large-scale problems. {It can also help mitigate the potential numerical issues in implementation caused by the exponential tilting operator.} We find that these methods perform well empirically on a variety of tasks (Section~\ref{sec:experiments}).
\begin{theorem}[Convergence of Algorithm~\ref{alg:stochastic-non-h-TERM-two-batch} for strongly-convex problems] \label{thm: convergence_stochastic_convex}
Assume $f: \mathcal{X} \times \Theta \to [\widetilde{F}_{\min}, \widetilde{F}_{\max}]$ is $L$-Lipschitz in $\theta$, i.e., $\widetilde{F}_{\min} \leq f(x;\theta) \leq \widetilde{F}_{\max}$,\footnote{For notation consistency between the max-loss and min-loss for any sample and any iteration, we use $\widetilde{F}_{\min}$ to denote the lower bound of $f(x_i;\theta_k)$. We note that $\widetilde{F}_{\min}=\widetilde{F}(-\infty)$ defined in Definition~\ref{def:wF}.} and  $|f(x; \theta_i) - f(x; \theta_j)| \leq L \|\theta_i-\theta_j\|$ for $x \in \mathcal{X}$ and $\theta_i, \theta_j \in \Theta \subseteq \mathbb{R}^d$. {Assume $\wR(t;\theta)$ has compact domain $\theta$.} Assume $\wR(t;\theta)$ is $\mu$-strongly convex (Assumption~\ref{assump: regularity}) with uniformly bounded stochastic gradient, i.e., $\|\nabla \wR(x_i; \theta)\| := \left\| \frac{e^{tf(x_i;\theta)}}{e^{t\wR(t;\theta)}} \nabla f(x_i;\theta) \right\| \leq B$ for $\theta \in \mathbb{R}^d$ and $i \in [N]$.  Denote $k_t :=\arg\max_k \left(k < \frac{2e}{\mu}+\frac{etLB  e^{t(\widetilde{F}_{\max}-\widetilde{F}_{\min})}}{\mu k}\right)$. Assume the batch size is 1. For $k \geq k_t$,
\begin{align}
    \mathbb{E}[\|\theta_{k+1}-\theta^*\|^2] \leq \frac{V_t}{k+1},
\end{align}
where
\begin{align}
     \theta^* := \breve{\theta}(t), \quad V_t = \max\left\{k_t\mathbb{E}[\|\theta_{k_t}-\theta^*\|^2], \frac{4B^2e^{2+2t(\widetilde{F}_{\max}-\widetilde{F}_{\min})}}{\mu^2}\right\},
\end{align}
and
\begin{align}
    \mathbb{E}[\|\theta_{k_t}-\theta^*\|^2] \leq \max\left\{\mathbb{E}[\|\theta_1-\theta^*\|^2], \frac{B^2e^{2t(\widetilde{F}_{\max}-\widetilde{F}_{\min})+1}}{\mu (1+tLBe^{t(\widetilde{F}_{\max}-\widetilde{F}_{\min})})}\right\}.
\end{align}
\end{theorem}
Our assumptions are standard compared with those in related literature~\citep{wang2017stochastic,qi2020attentional}. The uniformly bounded stochastic gradient of $\wR(t;\theta)$ assumption can be satisfied by the bounded gradient of $f(x_i;\theta)$, \new{which can be a limiting condition but has appeared in previous works on stochastic compositional optimization~\citep{qi2020attentional,wang2016accelerating}.}
{If the objectives are coercive, which typically holds in practice~\citep{bertsekas1997nonlinear}, Algorithm~\ref{alg:stochastic-non-h-TERM} will have bounded iterates and thus the compact domain assumption would hold.}  We defer full proofs to Appendix~\ref{app: solving-TERM}. The main steps involve bounding the expected estimation error  $\dE[e^{t(\wR_k-\doublewidetilde{R}_k)}]$ conditioning on the previous iterates $\{\theta_1, \ldots, \theta_k\}$.

\paragraph{Discussions.} The theorem indicates that Algorithm~\ref{alg:stochastic-non-h-TERM} starts to make progress after $k_t$ iterations, with convergence rate $O(e^{2t}/k)$. {\color{black}Both $k_t$ and $V_k$ could scale exponentially with $t$ in the worst-case analysis, but it does not completely reflect the dependence of Algorithm~\ref{alg:stochastic-non-h-TERM} on $t$ for modest values of $t$. 
Empirically, we observe that the stochastic TERM solver with moderate values of $t$ can converge faster compared with stochastic min-max solvers, which has a rate of $1/\sqrt{k}$ for strongly convex problems~\citep{levy2020large}. This leaves open for future work understanding the exact scaling of the convergence rate of stochastic TERM as $t\to \infty.$}  

Next, we present convergence results on non-convex smooth problems, without and with the assumptions of PL-conditions. We defer all proofs to Appendix~\ref{app: solving-TERM}.

\begin{theorem}[Convergence of Algorithm~\ref{alg:stochastic-non-h-TERM-two-batch} for non-convex smooth problems] \label{thm: convergence_stochastic_NCSM}
Assume $f: \mathcal{X} \times \Theta \to [\widetilde{F}_{\min}, \widetilde{F}_{\max}]$ is $L$-Lipschitz in $\theta$, i.e., $\widetilde{F}_{\min} \leq f(x;\theta) \leq \widetilde{F}_{\max}$, and  $|f(x; \theta_i) - f(x; \theta_j)| \leq L \|\theta_i-\theta_j\|$ for $x \in \mathcal{X}$ and $\theta_i, \theta_j \in \Theta \subseteq \mathbb{R}^d$. Assume $\wR(t;\theta)$ is $\beta$-smooth with uniformly bounded stochastic gradient, i.e., $\|\nabla \wR(x_i; \theta)\|\leq B$ for $\theta \in \mathbb{R}^d$ and $i \in [N]$. Assume the batch size is 1. Denote $k_t := \left\lceil \frac{2(\wF_{\max}-\wF_{\min})t^2L^2}{\beta e^2} \right\rceil$, then for $k \geq k_t$,
\setlength{\thinmuskip}{0mu}
\setlength{\medmuskip}{0.5mu}
\setlength{\thickmuskip}{0.5mu}
\begin{align}
    \frac{1}{K}\sum_{k=k_t}^K  \mathbb{E}[\|\nabla \wR (t;\theta_k)\|^2] \leq \sqrt{8}B e^{t(\wF_{\max}-\wF_{\min})+1} \sqrt{\frac{\beta (\wF_{\max}-\wF_{\min})}{K}}.
\end{align}
\end{theorem}

\setlength{\thinmuskip}{0mu}
\setlength{\medmuskip}{0.5mu}
\setlength{\thickmuskip}{0.5mu}
\begin{theorem}[Convergence of Algorithm~\ref{alg:stochastic-non-h-TERM-two-batch} for non-convex smooth problems with PL conditions] \label{thm: convergence_stochastic_NCSM_PL}
Let the assumptions in Theorem~\ref{thm: convergence_stochastic_NCSM} hold. Further assume that $\sum_{i \in [N]} p_i f(x_i;\theta)$ satisfies $\frac{\mu}{2}$-PL conditions for any $\mathbf{p} \in \Delta_N$ where $\mathbf{p} := (p_1, \dots, p_N)$. Assume the batch size is 1. Denote $ k_t := \arg\max_k \left(k < \frac{4e}{\mu}+\frac{4etLBE^{t(\wF_{\max}-\wF_{\min})}}{\mu k}\right)$, then for $t \in \mathbb{R}^{>0}$ and $k \geq k_t$,
\begin{align}
   \mathbb{E}[\wR(t;\theta_{k+1})-\wR(t;\breve{\theta}(t))] \leq  \frac{V_t}{k+1},
\end{align}
where 
\begin{align}
     V_t =\max\left\{k_t \mathbb{E}[\wR(t;\theta_{k_t})-\wR(t;\breve{\theta}(t))], \frac{8\beta B^2 e^{2t(\wF_{\max}-\wF_{\min})+2}}{\mu^2}\right\}.
\end{align}
\end{theorem}

\section{TERM Extended: Hierarchical Multi-Objective Tilting}\label{sec:term-extended}

We consider an extension of TERM that can be used to address practical applications requiring multiple objectives, e.g., simultaneously achieving robustness to noisy data \textit{and} ensuring fair performance across subgroups. Existing approaches typically aim to address such problems in isolation. 
To handle multiple objectives with TERM, let each sample $x$ be associated with a group $g \in [G],$ i.e., $x \in g.$ These groups could be related to the labels (e.g., classes in a classification task), or may depend only on features. For any $t, \tau \in \mathbb{R},$ we define multi-objective TERM as: 
\vspace{-.05in}
\begin{align}
\wJ(t, \tau; \theta) :=  \frac{1}{t} \log \left(\frac{1}{N} \sum_{g \in [G]} |g| e^{t \widetilde{R}_g (\tau; \theta)}\right)\,, \,\,\, \text{where} \,\,\, \wR_g (\tau; \theta) := \frac{1}{\tau} \log \left(\frac{1}{|g|} \sum_{x \in g}  e^{\tau f(x; \theta)} \right) \, ,
\label{eq: class-TERM}
\vspace{-.05in}
\end{align}
and $|g|$ is the size of group $g$. 
We evaluate the gradient of the hierarchical multi-objective tilt objective in Lemma~\ref{lemma: generalized-tilt-gradient} below.
\begin{lemma}[Hierarchical multi-objective tilted gradient]\label{lemma: generalized-tilt-gradient}
Under Assumption~\ref{assump:smoothness},
\begin{align}
    \nabla_\theta \wJ(t, \tau ;\theta) &= 
 \sum_{g\in [G]} \sum_{x \in g} w_{g, x}(t, \tau; \theta)\nabla_\theta f(x;\theta)
\end{align}
\vspace{-1em}
where
\begin{align}
\label{eq:def-w}
    w_{g, x}(t, \tau; \theta) 
    & := \frac{\left( \frac{1}{|g|} \sum_{y \in g} e^{\tau f(y; \theta)} \right)^{(\frac{t}{\tau} - 1)}}{\sum_{g' \in [G]} |g'| \left(\frac{1}{|g'|} \sum_{y \in g'} e^{\tau f(y; \theta)}\right)^\frac{t}{\tau}} \quad e^{\tau f(x; \theta)}.
\end{align}
\end{lemma}
Similar to the tilted gradient~\eqref{eq: tilted_grad}, Lemma~\ref{lemma: generalized-tilt-gradient} indicates that the multi-objective tilted gradient is a weighted sum of the gradients, making TERM similarly efficient to solve. 
Multi-objective TERM recovers sample-level TERM as a special case for $\tau = t$ (Lemma~\ref{lemma: hierarchical-tilt}), and reduces to group-level TERM with $\tau \to 0$.

\begin{lemma}
 [Sample-level TERM is a special case of hierarchical multi-objective TERM] Under Assumption~\ref{assump:smoothness}, hierarchical multi-objective TERM recovers TERM as a special case for $t = \tau$. That is 
 \label{lemma: hierarchical-tilt}
\begin{equation}
    \wJ(t, t; \theta) = \wR(t; \theta).    
\end{equation}
\end{lemma}
\begin{proof}
The proof is completed by noticing that setting $t = \tau$ in~\eqref{eq:def-w} recovers the original sample-level tilted gradient.
\end{proof}
 Note that all properties discussed in Section~\ref{sec:term_properties} carry over to group-level TERM.
We validate the effectiveness of hierarchical tilting empirically in Section~\ref{sec:exp:multi_obj},
where we show that TERM can significantly outperform baselines to handle class imbalance {\it and} noisy outliers simultaneously, while underperforming a much more complicated method in their setup. Note that hierarchical tilting could be extended to hierarchies of greater depths (than two) to simultaneously handle more than
two objectives at the cost of one extra tilting hyperparameter per each additional optimization objective. For instance, we state the multi-objective tilting for a hierarchy of depth three in Appendix~\ref{app:hierarchical-TERM}.


\subsection{Solving Hierarchical TERM}

To solve hierarchical TERM in the batch setting, we can directly use gradient-based methods with tilted gradients defined for the hierarchical objective in Lemma~\ref{lemma: generalized-tilt-gradient}. Note that Batch hierarchical TERM with $t=\tau$ reduces to solving the sample-level tilted objective~\eqref{eq: TERM}. We summarize this method in Algorithm~\ref{alg:batch-TERM}.

\begin{algorithm}[H]
\SetKwInOut{Init}{Initialize}
\SetAlgoLined
\DontPrintSemicolon
\SetNoFillComment
\KwIn{$t, \tau, \alpha$}
\While{stopping criteria not reached}{
\For{$g \in [G]$}{
compute the loss $f(x; \theta)$ and gradient $\nabla_\theta f(x; \theta)$ for all $x \in g$\;
$\wR_{g, \tau} \gets \text{$\tau$-tilted loss~\eqref{eq: class-TERM} on group $g$}$\; 
$\nabla_\theta \wR_{g, \tau} \gets \frac{1}{|g|}\sum_{x \in g} e^{\tau f(x; \theta) - \tau \wR_{g, \tau}} \nabla_\theta f(x; \theta)$
}
$\wJ_{t, \tau}  \gets  \frac{1}{t} \log \left(\frac{1}{N} \sum_{g \in [G]} |g|  e^{t \widetilde{R}_g (\tau; \theta)}\right)$\;
$w_{t,\tau, g} \gets |g|e^{t \wR_{\tau, g} - t \wJ_{t, \tau}}$\;
$\theta \gets \theta - \frac{\alpha}{N} \sum_{g\in [G]} w_{t, \tau, g} \nabla_{\theta} \wR_{g, \tau}$\;
}
\caption{Batch Hierarchical TERM}\label{alg:batch-TERM}
\end{algorithm}

We next discuss stochastic solvers for hierarchical multi-objective tilting.  
We extend Algorithm~\ref{alg:stochastic-non-h-TERM} to the multi-objective setting, presented in  Algorithm~\ref{alg:stochastic-TERM}. 
At a high level, at each iteration, group-level tilting is addressed by choosing a group based on the tilted weight vector. Sample-level tilting is then incorporated by re-weighting the samples in a uniformly drawn mini-batch. Similarly, we estimate the tilted objective $\wR_{g,\tau}$ for each group $g$ via a tilted average of the current estimate and the history. 
While we sample the group from which we draw the minibatch, for small number of groups, one might want to draw one minibatch per each group and weight the resulting gradients accordingly.

\begin{algorithm}[h!]
\SetKwInOut{Init}{Initialize}
\SetAlgoLined
\DontPrintSemicolon
\SetNoFillComment
\Init{$\doublewidetilde{R}_{g, \tau} = 0~  \forall g\in [G]$}
\KwIn{$t, \tau, \alpha, \lambda$}
\While{stopping criteria not reached}{
sample $g$ on $[G]$ from a Gumbel-Softmax distribution with logits $\doublewidetilde{R}_{g, \tau} + \frac{1}{t}\log |g|$ and temperature {$\frac{1}{t}$}\;
sample minibatch $B$ uniformly at random within group $g$\; compute the loss $f(x; \theta)$ and gradient $\nabla_\theta f(x; \theta)$ for all $x \in B$\;
$\wR_{B, \tau} \gets \text{$\tau$-tilted loss~\eqref{eq: TERM} on minibatch $B$}$\;
$\doublewidetilde{R}_{g, \tau} \gets \frac{1}{\tau} \log \left((1-\lambda) e^{\tau\doublewidetilde{R}_{g, \tau}} + \lambda e^{\tau\wR_{B, \tau}}\right)$\; $w_{\tau, x} \gets e^{\tau f(x; \theta) - \tau \doublewidetilde{R}_{g, \tau}}$\;
$\theta \gets \theta - \frac{\alpha}{|B|} \sum_{x\in B} w_{\tau, x} \nabla_{\theta} f(x; \theta)$\;
}
\caption{Stochastic Hierarchical TERM}\label{alg:stochastic-TERM}
\end{algorithm}

Group-level tilting can be recovered from Algorithm~\ref{alg:batch-TERM} and \ref{alg:stochastic-TERM} by setting the inner-level tilt parameter $\tau=0$.  We apply TERM to a variety of machine learning problems; for clarity, we summarize the applications and their corresponding algorithms in Table~\ref{table: alg_choice} in the appendix.

\section{TERM in Practice: Use Cases}\label{sec:experiments}

We now showcase the  {flexibility}, {wide applicability}, and competitive performance of the TERM framework through empirical results on a variety of real-world problems such as handling outliers (Section~\ref{sec:exp:robustness}), ensuring fairness and improving generalization (Section~\ref{sec:exp:fairness}), and addressing compound issues (Section~\ref{sec:exp:multi_obj}). Despite the relatively straightforward modification TERM makes to traditional ERM, we show that $t$-tilted losses not only outperform ERM, but either outperform or are competitive with state-of-the-art, problem-specific tailored baselines on a wide range of applications. 
We provide implementation details in Appendix~\ref{appen:exp_detail}. 
All code, datasets, and experiments are publicly available at \href{https://github.com/litian96/TERM}{\texttt{github.com/litian96/TERM}}. {The applications explored are summarized in Table~\ref{table: applications} below.}

\begin{table}[h]
	\caption{{Summary of TERM applications.}}
	\centering
	\label{table: applications}
	\scalebox{0.91}{
	\begin{tabular}{llc} 
	       \toprule[\heavyrulewidth]
        \multicolumn{2}{c}{\textbf{Applications}}  & \textbf{Sections} \\
        \midrule
        \multirow{3}{*}{Mitigating noisy outliers ($t<0$)} &	Robust regression &	 Sec.~\ref{sec:exp:robustness:regression} \\
         &	Robust classification &	 Sec.~\ref{sec:exp:robustness:classification} \\
         &	Low-quality annotators & Sec.~\ref{sec:exp:noisy_annotator} \\
        \hline
        \multirow{5}{*}{Fairness and generalization ($t>0$)} & Fair PCA & Sec.~\ref{sec:exp:fairness:pca} \\
        & Fair federated learning  & Sec.~\ref{sec:exp:fairness:fl} \\
         & Fair meta-learning  & Sec.~\ref{sec:exp:fairness:meta} \\
         & Handling class imbalance  & Sec.~\ref{sec:exp:class_imbalance} \\
         & Improving generalization via variance reduction  & Sec.~\ref{sec:variance_reduction} \\
        \hline
        \multirow{2}{*}{Hierarchical multi-objective tilting} & Class imbalance and random noise &  Sec.~\ref{sec:exp:hierarchical1} \\
         & Class imbalance and adversarial noise &  Sec.~\ref{sec:exp:hierarchical2} \\
    \bottomrule[\heavyrulewidth]
	\end{tabular}}
\end{table}

\paragraph{{Choosing $t$.}}{{In applications when we consider tradeoffs between different objectives (e.g., fair
meta-learning and federated learning), we perform a grid search over $t$ from \{0.1, 1, 2, 5, 10, 50, 100, 200\} on the validation set and pick the one with the best fairness performance while not degrading mean performance. When there is not a single t dominating other values (e.g., fair PCA), we report results under different values of $t$.}  
In our initial robust regression experiments, we find that the performance is robust to various $t$'s, and we thus use a fixed $t=-2$ for all experiments involving negative $t$ (Section~\ref{sec:exp:robustness} and Section~\ref{sec:exp:multi_obj}).
For all values of $t$ tested, the number of iterations required to solve TERM is within 2$\times$  that of standard ERM, with the same per-iteration complexity.}

\subsection{Mitigating Noisy Outliers ($t<0$)} \label{sec:exp:robustness}

We begin by investigating TERM's ability to find robust solutions that  reduce the effect of noisy outliers. 
We note that we specifically focus on the setting of `robustness' involving random additive noise; the applicability of TERM to more adversarial forms of robustness would be an interesting direction of future work.
We do not compare with approaches that require additional clean validation data~\citep[e.g.,][]{roh2020fr,veit2017learning,hendrycks2018using, ren2018learning}, as such data can be costly to obtain in practice.

\subsubsection{Robust Regression} \label{sec:exp:robustness:regression}

\paragraph{Label noise.} We first consider a regression task with noise corrupted targets, where we aim to minimize the root mean square error (RMSE) on samples from the Drug Discovery dataset~\citep{olier2018meta,diakonikolas2019sever}. The task is to predict the bioactivities given a set of chemical   compounds. 
We compare against linear regression with an $L_2$ loss, which we view as the `standard' ERM solution for regression, as well as with losses commonly used to mitigate outliers---the $L_1$ loss and Huber loss~\citep{Huber-loss}. We also compare with consistent robust regression (CRR)~\citep{bhatia2017consistent} and STIR~\citep{pmlr-v89-mukhoty19a}, recent state-of-the-art methods specifically designed for label noise in robust regression. 
In this particular problem, TERM is equivalent to exponential squared loss, studied in~\citep{wang2013robust}.
We apply TERM at the sample level with an $L_2$ loss, and generate noisy outliers by assigning random targets drawn from $\mathcal{N}(5,5)$ on a fraction of the samples.

In Table~\ref{table: outlier}, we report RMSE on clean test data for each objective and  under different noise levels. We also present the performance of an oracle method (Genie ERM) which has access to all of the clean data samples with the noisy samples removed. {\it Note that Genie ERM is not a practical algorithm and is solely presented to set the expected performance limit in the noisy setting}. The results indicate that TERM is competitive with baselines on the 20\% noise level, and achieves better robustness with moderate-to-extreme noise. We observe similar trends in  scenarios involving both noisy features and targets  (Appendix~\ref{appen:complete_exp}).  CRR tends to run slowly as it scales cubicly with the number of dimensions~\citep{bhatia2017consistent}, while solving TERM is roughly as efficient as ERM.

\begin{table}[h!]
	\caption{TERM is competitive with robust \textit{regression} baselines, particularly in high noise regimes.}
	\centering
	\label{table: outlier}
	\scalebox{0.88}{
	\begin{tabular}{lcccc} 
	       \toprule[\heavyrulewidth]
        \multicolumn{1}{l}{\multirow{2}{*}{\textbf{objectives}}} & \multicolumn{3}{c}{\textbf{test RMSE} ({\fontfamily{qcr}\selectfont{Drug Discovery}})} \\
        \cmidrule(r){2-4}
          &   20\% noise & 40\% noise & 80\% noise\\
        \midrule
        \rule{0pt}{2ex}ERM  &	1.87 {(.05)}  &	2.83  {(.06)}	& 4.74 {(.06)} \\
        $L_1$ & 	\textbf{1.15} { (.07)}  &	1.70  {(.12)}	& 4.78 {(.08)} \\
        Huber~\citep{Huber-loss} & 	\textbf{1.16} {(.07)}  &	1.78  {(.11)}	& 4.74 {(.07)} \\
         
         STIR~\citep{pmlr-v89-mukhoty19a} & {\bf 1.16} {(.07)} & 1.75 {(.12)} & 4.74 {(.06)} \\
         
         CRR~\citep{bhatia2017consistent}  &	\textbf{1.10} {(.07)}  &	1.51  {(.08)}	& 4.07 {(.06)} \\
          
          \rowcolor{myblue}
        TERM & 	\textbf{1.08} {(.05)}  &	\textbf{1.10}  {(.04)}	& \textbf{1.68} {(.03)} \\
        \hline
        \rule{0pt}{2ex}Genie ERM  &	1.02 {(.04)}  &	1.07  {(.04)} &	1.04 {(.03)} \\
    \bottomrule[\heavyrulewidth]
	\end{tabular}}
\end{table}

\paragraph{Label and feature noise.} Here, we present results involving both feature noise and target noise.
We investigate the performance of TERM on two datasets (cal-housing~\citep{pace1997sparse} and abalone~\citep{dua2019uci}) used in~\citet{yu2012polynomial}. Both datasets have features with 8 dimensions. We generate noisy samples following the setup in~\citet{yu2012polynomial}---sampling 100 training samples, and randomly corrupting 5\% of them by multiplying their features by 100 and multiply their targets by 10,000.  From Table~\ref{table: outlier_setting2} below, we see that TERM significantly outperforms the baseline objectives in the noisy regime on both datasets.

\begin{table}[h]
	\caption{An alternative noise setup involving both feature and label noise. Similarly, TERM with $t=-2$ significantly outperforms several baseline objectives for noisy outlier mitigation.
	}
	\centering
	\label{table: outlier_setting2}
	\scalebox{0.85}{
	\begin{tabular}{ l  ll | ll} 
			\toprule[\heavyrulewidth]
			 \multicolumn{1}{l}{\multirow{2}{*}{ \textbf{objectives}}} & \multicolumn{2}{c}{\textbf{test RMSE} ({\fontfamily{qcr}\selectfont{cal-housing}})}  & \multicolumn{2}{c}{\textbf{test RMSE}  ({\fontfamily{qcr}\selectfont{abalone}})}  \\
			 \cmidrule(r){2-5}
            &      clean         & noisy     & clean             & noisy  \\
            \hline
            ERM	   &     0.766 {\tiny(0.023)}	& 239   {\tiny(9)}  &  2.444 {\tiny (0.105)}	&  1013  {\tiny (72)  }  \\
            $L_1$	   &     0.759 {\tiny(0.019)}	& 139   {\tiny(11)} &   2.435 {\tiny (0.021)}	&  1008  {\tiny (117) }  \\
            Huber~\citep{Huber-loss}	   &     0.762 {\tiny(0.009)}	& 163   {\tiny(7)}  &  2.449 {\tiny (0.018)}	&  922   {\tiny (45) } \\
            CRR~\citep{bhatia2017consistent} & 0.766 {\tiny (0.024)} & 245 {\tiny (8)} & 2.444 {\tiny (0.021)} & 986 {\tiny (146)} \\
            \rowcolor{myblue}
            TERM 	   &      0.745 {\tiny(0.007)}	& {\textbf{0.753} {\tiny(0.016)}}	 &  2.477 {\tiny (0.041)}	&  \textbf{2.449} {\tiny (0.028)} \\
            \hline
            Genie ERM & 0.766 {\tiny(0.023)} & 0.766 {\tiny (0.028)} & 2.444 {\tiny (0.105)} & 2.450 {\tiny (0.109)} \\
			\bottomrule
	\end{tabular}}
\end{table}

\paragraph{Unstructured random v.s. adversarial noise.} 

As a word of caution, we note that the experiments thus far have focused on random noise. This makes it possible for the methods to find the underlying structure of clean data even if the majority of the samples are noisy outliers. To gain more intuition on these cases, we generate synthetic two-dimensional data points and test the performance of TERM under 0\%, 20\%, 40\%, and 80\% noise for linear regression. TERM with $t=-2$ performs well in all noise levels (Figure~\ref{fig:robust_regression_synthetic} and \ref{fig:robust_regression_synthetic2}). However, as one might expect, TERM with negative $t$'s could potentially overfit to outliers if they are constructed in an adversarial way. In the examples shown in Figure~\ref{fig:adversarial_noise}, under 40\% noise and 80\% noise, TERM has a high error measured on the clean data (green dots).

\begin{figure}[h!]
    \centering
    \includegraphics[width=0.97\textwidth]{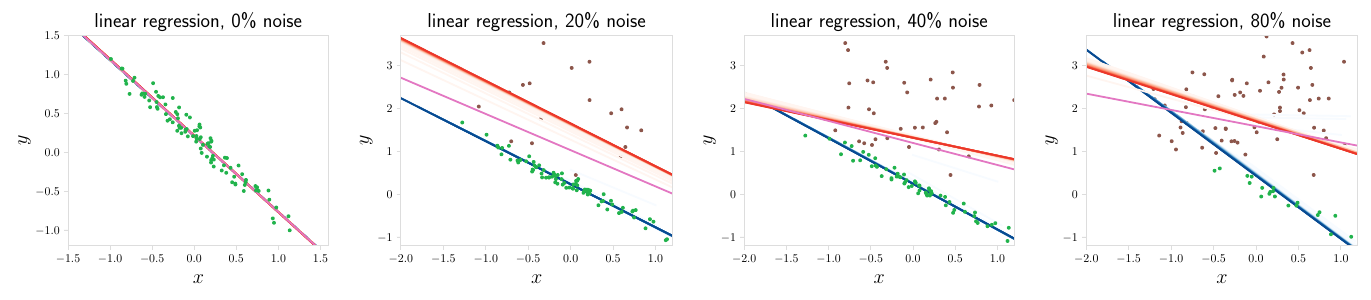}
    \vspace{-0.2in}
    \caption{{Robust regression on synthetic data with random noise where the mean of the noisy samples is different from that of clean ones. TERM with negative $t$'s (blue, $t=-2$) can fit structured clean data at all noise levels, while ERM (purple) and TERM with positive $t$'s (red) overfit to corrupted data. We color inliers in green and outliers in brown for visualization.}}
    \label{fig:robust_regression_synthetic}
    \vspace{-0.1in}
\end{figure}

\begin{figure}[h!]
    \centering
    \includegraphics[width=0.97\textwidth]{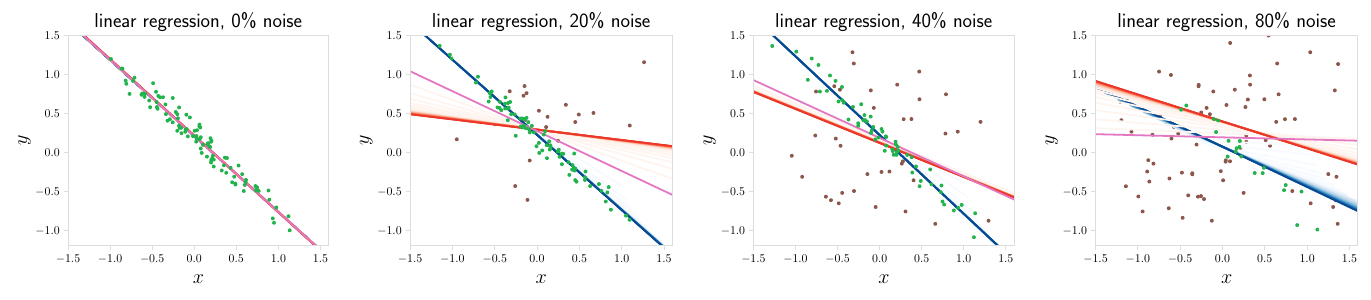}
    \vspace{-0.2in}
    \caption{{In the presence of random noise with the same mean as that of clean data, TERM with negative $t$'s (blue) can still surpass outliers in all cases, while ERM (purple) and TERM with positive $t$'s (red) overfit to corrupted data.  While the performance drops for 80\% noise, TERM can still learn useful information, and achieves much lower error than ERM.}}
    \label{fig:robust_regression_synthetic2}
\end{figure}

\begin{figure}[h!]
    \centering
    \includegraphics[width=1.0\textwidth]{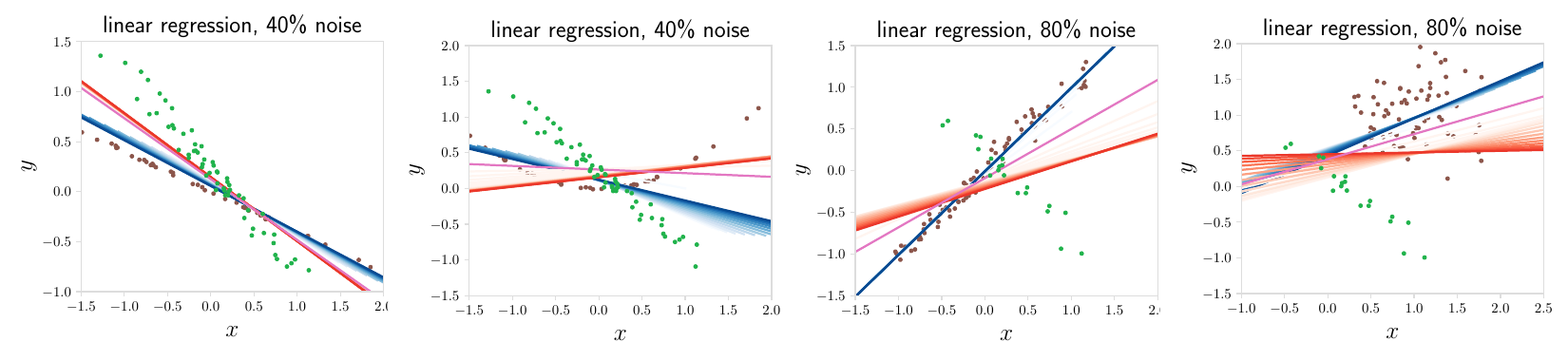}
    \vspace{-0.2in}
    \caption{{TERM with negative $t$'s (blue) cannot fit clean data if the noisy samples (brown) are adversarial or structured in a manner that differs substantially from the underlying true distribution.}}
    \label{fig:adversarial_noise}
\end{figure}

\subsubsection{Robust Classification} \label{sec:exp:robustness:classification}

Deep neural networks can easily overfit to corrupted labels~\cite[e.g.,][]{zhang2016understanding}. While the theoretical properties we study for TERM (Section~\ref{sec:term_properties}) do not directly cover objectives with neural network function approximations, we show that TERM can be applied empirically to DNNs to achieve robustness to noisy training labels. 
MentorNet~\citep{jiang2018mentornet} is a popular method in this setting, which learns to assign weights to samples based on feedback from a student net. Following the setup in~\citet{jiang2018mentornet}, we explore classification on CIFAR10~\citep{krizhevsky2009learning} when a fraction of the training labels are corrupted with uniform noise---comparing TERM with ERM and several state-of-the-art approaches~\citep{kumar2010self,ren2018learning,zhang2018generalized,krizhevsky2009learning}.  As shown in Table~\ref{table:label_noise_cnn}, TERM performs competitively with 20\% noise, and outperforms all baselines in the high noise regimes.  We use MentorNet-PD as a baseline since it does not require clean validation data. In Appendix~\ref{appen:complete_exp}, we  show that TERM also matches the performance of MentorNet-DD, which requires clean validation data. 
{To help reason about the performance of TERM, we also explore a simpler, two-dimensional logistic regression problem in Figure~\ref{fig:robust_classification_synthetic}, Appendix~\ref{appen:complete_exp}, finding that TERM with $t$=$-2$ is similarly robust across the considered noise regimes.}

\begin{table}[h!]
\caption{TERM is competitive with robust \textit{classification} baselines, and is superior in high noise regimes.}
\centering
\label{table:label_noise_cnn}
\scalebox{0.92}{
\begin{tabular}{lccc}
	   \toprule[\heavyrulewidth]
        \multicolumn{1}{l}{\multirow{2}{*}{\textbf{objectives}}} & \multicolumn{3}{c}{\textbf{test accuracy} ({\fontfamily{qcr}\selectfont{CIFAR10}}, Inception)} \\
        \cmidrule(r){2-4}
          &  20\% noise & 40\% noise & 80\% noise\\
        \midrule
    ERM & 0.775 {(.004)} &  0.719 {(.004)} & 0.284 {(.004)}  \\
 RandomRect~\citep{ren2018learning} & 0.744 {(.004)} & 0.699 {(.005)} & 0.384 {(.005)}  \\
SelfPaced~\citep{kumar2010self} & 0.784 {(.004)} & 0.733 {(.004)} & 0.272 {(.004)}  \\
 MentorNet-PD~\citep{jiang2018mentornet} & 0.798 {(.004)} & 0.731 {(.004)} & 0.312 {(.005)} \\
GCE~\citep{zhang2018generalized} & \textbf{0.805} {(.004)}  & 0.750 {(.004)} & 0.433 {(.005)} \\
 \rowcolor{myblue}
 TERM & 0.795 {(.004)}  & \textbf{0.768} {(.004)} & \textbf{0.455} {(.005)} \\
 \hline 
Genie ERM & 0.828 {(.004)} & 0.820 {(.004)} & 0.792 {(.004)} \\
\bottomrule[\heavyrulewidth]
\end{tabular}}
\end{table}

\subsubsection{Low-Quality Annotators}
\label{sec:exp:noisy_annotator}

It is not uncommon for practitioners to obtain human-labeled  data for their learning tasks from crowd-sourcing platforms.  
However, these labels are usually noisy in part due to the varying quality of the human annotators. Given a collection of labeled samples from crowd-workers, we aim to learn statistical models that are robust to the potentially low-quality annotators. 
As a case study, following the setup of~\citep{khetan2017learning}, we take the CIFAR-10 dataset and simulate 100 annotators where 20 of them are {\it hammers} (i.e., always correct)  and 80 of them are {\it spammers} (i.e., assigning labels uniformly at random). We apply TERM at the annotator group level in~\eqref{eq: class-TERM}, which is equivalent to assigning annotator-level weights based on the aggregate value of their loss. As shown in Figure~\ref{fig:noisy_annotator}, TERM is able to achieve the test accuracy limit set by \textit{Genie ERM}, i.e., \textit{the ideal performance obtained by completely removing the known outliers}. 
We note in particular that the accuracy reported by~\citep{khetan2017learning} (0.777) is lower than TERM (0.825) in the same setup, even though their approach is a two-pass algorithm requiring at least to double the training time. We provide full empirical details and investigate additional noisy annotator scenarios in Appendix~\ref{appen:complete_exp}.

\begin{figure}[h!]
    \centering
    \includegraphics[width=0.5\linewidth]{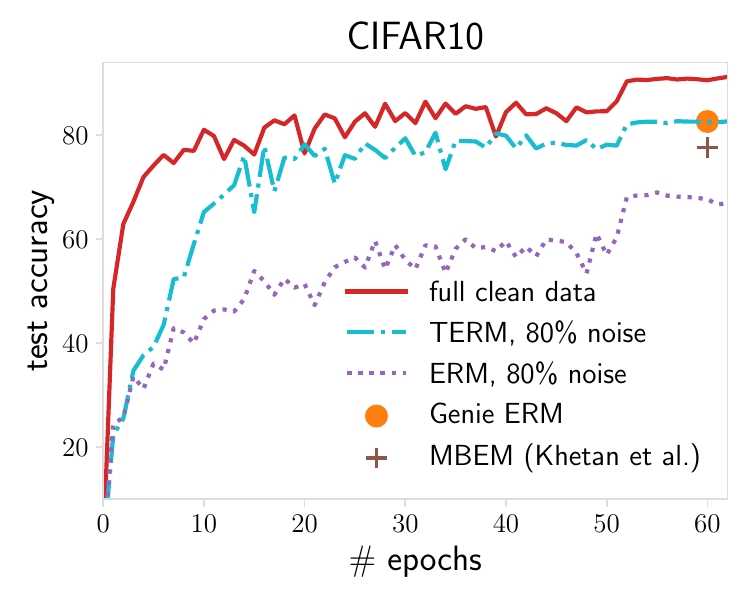}
    \captionof{figure}{TERM ($t{=}{-}2$) completely removes the impact of noisy annotators, reaching the performance limit set by Genie ERM.}
    \label{fig:noisy_annotator}
\end{figure}

\subsection{Fairness and  Generalization ($t>0$)}\label{sec:exp:fairness}

In this section, we show that positive values of $t$ in TERM can help promote fairness via learning fair representations and enforcing fairness during optimization, and offer variance reduction for better generalization.

\subsubsection{Fair Principal Component Analysis (PCA)} \label{sec:exp:fairness:pca}

We explore the flexibility of TERM in learning fair representations using PCA. 
In fair PCA, the goal is to learn low-dimensional representations which are fair to all considered subgroups (e.g., yielding similar reconstruction errors)~\citep{samadi2018price, tantipongpipat2019multi, kamani2019efficient}. 
Despite the non-convexity of the fair PCA problem, we apply TERM to this task, referring to the resulting objective as TERM-PCA. We tilt the same loss function as in~\citet{samadi2018price}: 
$f(X; U) = \frac{1}{|X|}\left(\|X-XUU^\top \|_{F}^2- \|X-\hat{X}\|_{F}^2\right) \, ,$
where $X \in \mathbb{R}^{n \times d}$ is a subset (group) of data, $U \in \mathbb{R}^{d \times r}$ is the current projection, and $\hat{X} \in \mathbb{R}^{n \times d}$ is the optimal rank-$r$ approximation of $X$.
Instead of solving a more complex min-max problem using semi-definite programming as in~\citet{samadi2018price}, which scales poorly with problem dimension,
we apply gradient-based methods,  re-weighting the gradients at each iteration based on the loss on each group. In Figure~\ref{fig:pca}, we plot the aggregate loss for two groups (high vs. low education) in the Default Credit dataset~\citep{yeh2009comparisons} for  different target dimensions $r$.
By varying $t$, we achieve varying degrees of performance improvement on different groups---TERM ($t=200$) recovers the min-max results of~\citep{samadi2018price} by forcing the losses on both groups to be (almost) identical, while TERM ($t=10$) offers the flexibility of reducing the performance gap less aggressively. {We also provide convergence plots for different values of $t$ in this application (Figure~\ref{fig:pca_convergence}), and observe slower convergence for larger values of $t$, which is consistent with our analyses in Section~\ref{sec:term_properties} and~\ref{sec:solver}. However, we do not observe exponential dependence on $t$ from the convergence curves, which suggest that the theoretical dependence on $t$ in convergence proofs for the solvers may be an artifact of our proof techniques, and might possibly be further improved by other analysis techniques for typical practical use cases.} 

\begin{figure}[h!]
\centering
\begin{minipage}{.5\textwidth}
  \centering
  \includegraphics[width=0.96\linewidth]{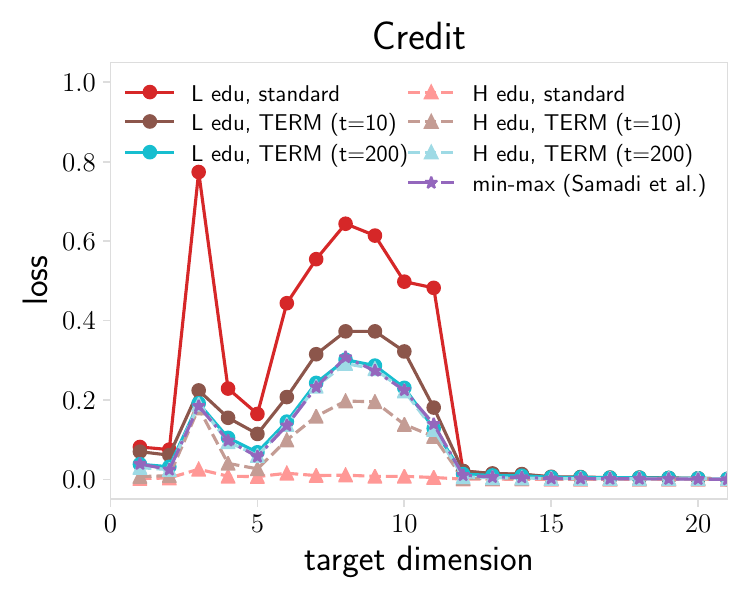}
  \captionof{figure}{TERM-PCA  flexibly trades the performance on the high (H) edu group for the performance on the low (L) edu group.}
  \label{fig:pca}
\end{minipage}%
\hfill
\begin{minipage}{.46\textwidth}
  \centering
  \includegraphics[width=0.92\linewidth]{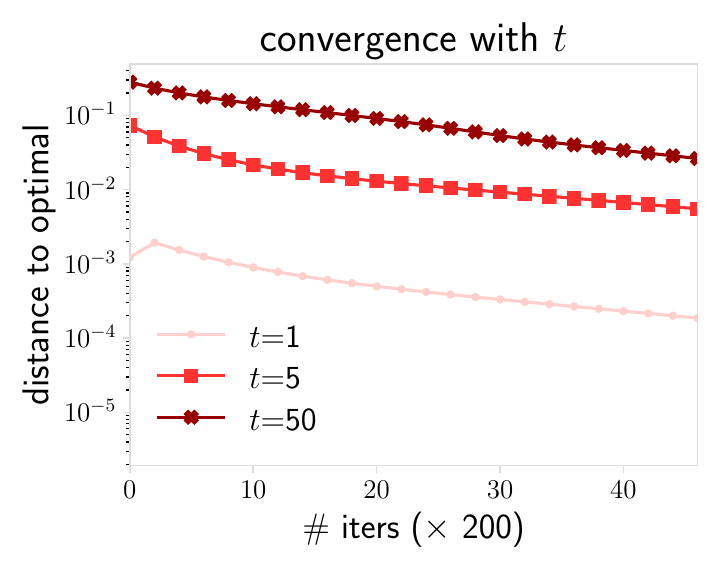}
  \captionof{figure}{{Convergence of TERM with respect to $t$ in fair PCA (target dimension=7). We tune optimal learning rates separately for each $t$. As $t$ increases, the convergence becomes slower, which validates our analyses in Section~\ref{sec:term_properties} and~\ref{sec:solver}.}}
  \label{fig:pca_convergence}
\end{minipage}
\end{figure}

\subsubsection{Fair Federated Learning} \label{sec:exp:fairness:fl}
Federated learning involves learning statistical models across massively distributed networks of remote devices or isolated organizations~\citep{mcmahan2016FedAvg,li2020survey}. Ensuring fair (i.e., uniform) performance distribution across the devices is a major concern in federated settings~\citep{mohri2019agnostic, li2019fair}, as using current approaches for federated learning (FedAvg~\citep{mcmahan2016FedAvg}) may result in highly variable performance across the network.~\citet{li2019fair} consider solving an alternate objective for federated learning, called $q$-FFL, to dynamically emphasize the worst-performing devices, which is conceptually similar to the goal of TERM, though it is applied specifically to the problem of federated learning and limited to the case of positive $t$. 
Here, we compare TERM with $q$-FFL in their setup on the vehicle dataset~\citep{duarte2004vehicle} consisting of data collected from 23 distributed sensors (hence 23 devices). We tilt the $L_2$ regularized linear SVM objective at the device level. At each communication round, we re-weight the accumulated local model updates from each selected device based on the weights estimated via Algorithm~\ref{alg:stochastic-TERM}. From Figure~\ref{fig:flearn}, we see that similar to $q$-FFL, TERM ($t=0.1$) can also significantly promote the accuracy on the worst device while maintaining the overall performance. 
The statistics of the accuracy distribution are reported in Table~\ref{table: fair_flearn} below.

\begin{figure}[h]
\centering
\begin{minipage}{.46\textwidth}
  \centering
  \includegraphics[width=0.92\linewidth]{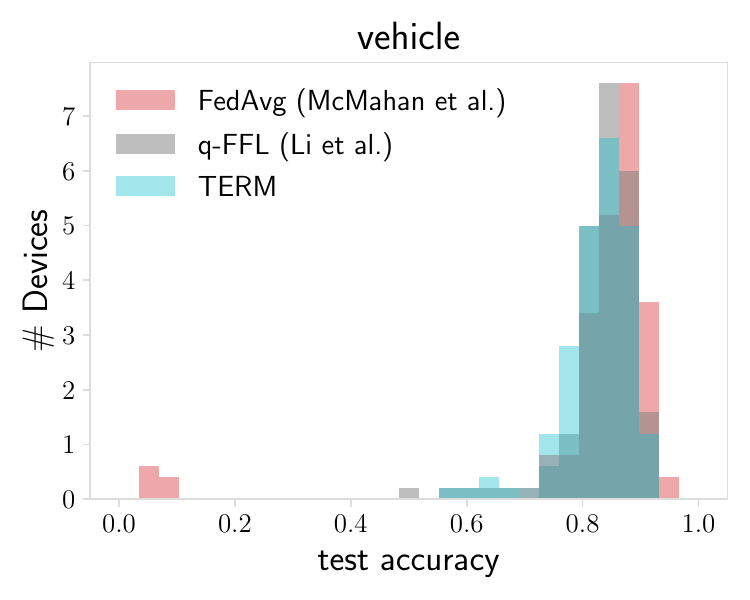}
  \captionof{figure}{TERM FL ($t = 0.1$) significantly increases the accuracy on the worst-performing device (similar to $q$-FFL) while obtaining a similar average accuracy.}
  \label{fig:flearn}
\end{minipage}%
\hfill
\begin{minipage}{.51\textwidth}
  \centering
  \captionof{table}{Both $q$-FFL and TERM can encourage more uniform accuracy distributions across the devices in
federated networks while maintaining similar average performance. {Numbers in the parentheses correspond to the standard error of each metric across 5 runs.}} 
\setlength{\tabcolsep}{3pt}
  \scalebox{0.85}{
 \begin{tabular}{l ccc} 
     \toprule[\heavyrulewidth]
       \multirow{2}{*}{\textbf{objectives}} & \multicolumn{3}{c}{\textbf{test accuracy}} \\
        \cmidrule(r){2-4}
		  & average &  worst 10\%  & stdev \\
		\hline
		 FedAvg & 0.853 {\tiny (.078)} & 0.421 {\tiny (.007)}  & 0.173 {\tiny  (.001)} \\
		$q$-FFL ($q=5$) & 0.862 {\tiny (.029)} & {\bf 0.704} {\tiny (.033)} & {\bf 0.064} {\tiny (.005)}\\
		\rowcolor{myblue}
		TERM ($t=0.1$) & 0.853 {\tiny (.027)} & {\bf 0.707} {\tiny (.009)} & {\bf 0.061} {\tiny (.003)} \\
       \bottomrule[\heavyrulewidth]
      \end{tabular}}
      \vspace{1em}
  \label{table: fair_flearn}
\end{minipage}
\vspace{-1em}
\end{figure}

\subsubsection{Fair Meta-Learning} \label{sec:exp:fairness:meta}

Meta-learning aims to learn a shared initialization across all tasks such that the initialization can quickly adapt to unseen tasks (i.e., meta-testing tasks) using a few samples. In practice, the resulting performance across meta-testing tasks can vary due to different data distributions associated with these tasks. One of the popular meta-learning methods is MAML~\citep{finn2017model}, whose objective is to minimize the sum of empirical losses across tasks $\{\mathcal{T}_i\}$ generated from $p(\mathcal{T})$ after one step of adaptation, i.e., $\min_{\theta} \sum_{\mathcal{T}_i \sim p(\mathcal{T})} f\left(\mathcal{T}_i; \theta-\alpha \nabla_{\theta} f(\mathcal{T}_i; \theta) \right)$. Previous works have proposed a min-max variant of MAML to encourage a more fair (uniform) performance distribution by optimizing the worst meta-training task called TR-MAML~\citep{collins2020task}. We apply TERM to MAML by replacing the ERM formulation with tilted losses. Following the setup in~\citet{collins2020task}, we evaluate TERM on the popular sin wave regression problem. For a fair comparison, we perform task-level tilting for TERM, and operates on task-level reweighting for TR-MAML.  From Table~\ref{table:fair_meta_learning}, we see that TERM with $t=2$ not only decreases the standard deviation of test errors, but also achieves lower mean errors than MAML. As the number of tasks is large (5,000), solving the min-max variant (TR-MAML) is challenging, and results in slightly worse performance than TERM.

\begin{figure}[h]
\centering
\begin{minipage}{.44\textwidth}
  \centering
  \includegraphics[width=0.96\linewidth]{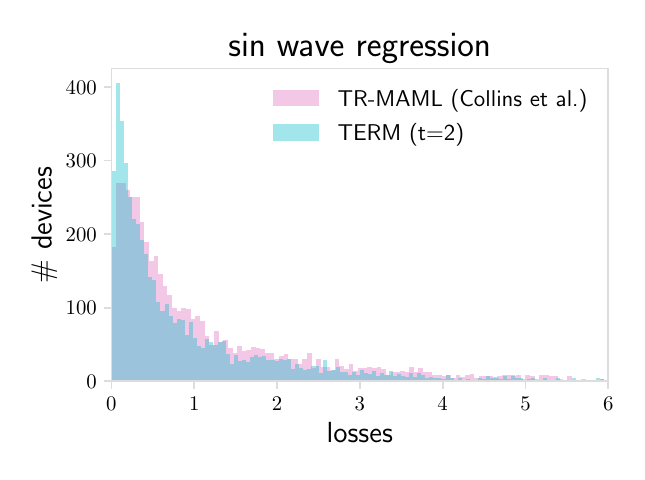}
  \captionof{figure}{{Loss distribution of TERM compared with the TR-MAML baseline.}}
  \label{fig:sin_loss_distribution}
\end{minipage}%
\hfill
\begin{minipage}{.52\textwidth}
  \centering
  \captionof{table}{TERM ($t=2$) results in fairer and lower test errors across meta-test tasks after adaptation compared with MAML~\citep{finn2017model}. TERM also outperforms a recently proposed min-max task-robust MAML method (TR-MAML)~\citep{collins2020task}.}
\setlength{\tabcolsep}{3pt}
  \scalebox{0.96}{
 \begin{tabular}{lcccc} 
     \toprule[\heavyrulewidth]
       methods & mean & std & max & worst 10\%  \\ \hline
       MAML & 1.23 &	1.63 &	19.1 & 5.16 \\
       TR-MAML 
       & 1.25	& 1.51	& 14.31  & 4.85 \\
       \rowcolor{myblue}
       TERM ($t=2$) & \textbf{1.14}	& \textbf{1.33}	& \textbf{13.59} &  \textbf{4.29} \\ 
       \bottomrule[\heavyrulewidth]
      \end{tabular}}
      \vspace{1em}
  \label{table:fair_meta_learning}
\end{minipage}
\end{figure}

\subsubsection{Handling Class Imbalance}
\label{sec:exp:class_imbalance}
Next, we show that TERM can reduce the performance variance across classes 
with extremely imbalanced data when training deep neural networks. 
We compare TERM with several baselines which re-weight samples during training, including assigning weights inversely proportional to the class size (InverseRatio), focal loss~\citep{lin2017focal}, HardMine~\citep{malisiewicz2011ensemble}, and LearnReweight~\citep{ren2018learning}.  Following the setting of~\citet{ren2018learning}, the datasets are composed of imbalanced $4$ and $9$ digits from MNIST~\citep{lecun1998gradient}. In Figure~\ref{fig:class_imabalance}, we see that TERM obtains similar (or higher) final accuracy on the clean test data as the state-of-the-art methods. We note that compared with LearnReweight, 
which optimizes the model over an additional balanced validation set and requires three gradient calculations for each update, TERM neither requires  such balanced  validation data nor does it increase the per-iteration complexity.

\begin{figure}[h!]
    \centering
    \vspace{-1em}
    \includegraphics[width=0.5\linewidth]{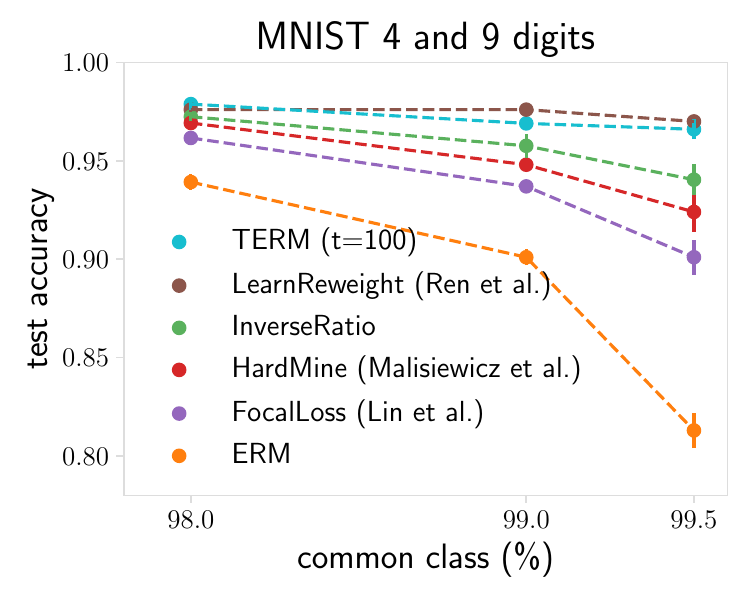}
    \caption{TERM ($t{=}100$) is competitive with state-of-the-art methods for classification with imbalanced classes.}
    \label{fig:class_imabalance}
\end{figure}

\subsubsection{Improving Generalization via Variance Reduction} \label{sec:variance_reduction}

\setlength{\tabcolsep}{3pt}
\begin{table}[b!]
	\caption{TERM ($t=0.1$) is competitive with strong baselines in generalization. TERM ($t=50$) outperforms ERM$_{+}$ (with decision threshold changed for providing fairness) and is competitive with RobustRegRisk$_{+}$ with no need for extra hyperparameter tuning. 
	}
	\centering
	\label{table: dro}
\scalebox{0.84}{
	\begin{tabular}{ l  ll  ll  ll }
			\toprule[\heavyrulewidth]
        \multicolumn{1}{l}{\multirow{2}{*}{ \textbf{objectives}}} 
			& \multicolumn{2}{c}{\textbf{accuracy} ($Y=0$)}  & \multicolumn{2}{c}{\textbf{accuracy} ($Y=1$)} & \multicolumn{2}{c}{\textbf{overall accuracy} (\%)}    \\
            \cmidrule(r){2-7}
             & train & test & train & test & train & test \\
            \hline 
           \rule{0pt}{2ex}ERM   & 0.841 {\tiny (.005)} &	0.822 {\tiny (.009)} &	0.971 {\tiny (.000)} &	0.966 {\tiny (.002)}  &	0.944 {\tiny (.000)} &	0.934 {\tiny (.003)} \\
            Linear SVM & 0.873 {\tiny (.003)}	& {\bf 0.838} {\tiny (.013)}	& 0.965 {\tiny (.000)}	& 0.964 {\tiny (.002)}	& 0.951 {\tiny (.001)}	& \textbf{0.937} {\tiny (.004)} \\
            {CVaR~\citep{rockafellar2000optimization}} & {0.877 {\tiny (.004)}} & {{\bf 0.844} {\tiny (.013)}} & {0.972 {\tiny (.000)}} & {0.964 {\tiny (.003)}} & {0.952 {\tiny (.001)}} & {{\bf 0.937} {\tiny (.003)}}  \\
            LearnReweight~\citep{ren2018learning}  & 0.860 {\tiny (.004)} & \textbf{0.841} {\tiny (.014)} & 0.960 {\tiny (.002)} & 0.961 {\tiny (.004)} & 0.940 {\tiny (.001)} & 0.934 {\tiny (.004)}\\
            FocalLoss~\citep{lin2017focal} & 0.871 {\tiny  (.003) } &	{\bf 0.834} {\tiny (.013)} &	0.970 {\tiny (.000)} &	0.966 {\tiny (.003)} &	0.949 {\tiny (.001)} &	\textbf{0.937} {\tiny (.004)} \\
            HRM~\citep{leqi2019human}  & 0.875 {\tiny (.003)}  & {\bf 0.839} {\tiny (.012)} & 0.972 {\tiny (.000)} & 0.965 {\tiny (.003)} & 0.952 {\tiny (.001)} & {\bf 0.937} {\tiny (.003)} \\
            RobustRegRisk
            (\hyperlink{cite.duchi2019variance}{Duchi et al., 2019})
            & 0.875 {\tiny (.003)} &   \textbf{0.844} {\tiny (.010)} &	0.971 {\tiny (.000)} &	0.966 {\tiny (.003)}  &	0.951 {\tiny (.001)} &	\textbf{0.939} {\tiny (.004)} \\
            \rowcolor{myblue}
            TERM ($t=0.1$) & 0.864 {\tiny (.003)} &	\textbf{0.840}  {\tiny (.011)} &	0.970 {\tiny (.000)} &	0.964 {\tiny (.003)} &	0.949 {\tiny (.001)} &	\textbf{0.937} {\tiny (.004)} \\
            \hline
           \rule{0pt}{2ex}ERM$_{+}$ (thresh = 0.26)  & 0.943 {\tiny (.001)} &	0.916 {\tiny (.008)} &	0.919 {\tiny (.001)}  & 0.917 {\tiny (.003)}  & 0.924 {\tiny (.001)}  &	0.917 {\tiny (.002)} \\
           RobustRegRisk$_{+}$ (thresh=0.49) & 0.943 {\tiny (.000)} &	0.917 {\tiny (.005)} &	0.928 {\tiny (.001)} &	{\bf 0.928} {\tiny (.002)} &	0.931 {\tiny (.001)} &	\textbf{0.924} {\tiny (.001)} \\
           \rowcolor{myblue}
      TERM ($t = 50$) &  0.942 {\tiny(.001)} &   0.917 {\tiny (.005)} &	0.926 {\tiny (.001)} &  {\bf 0.925} {\tiny(.002)}   &	0.929 {\tiny(.001)}	 &  \textbf{0.924} {\tiny (.001)} \\
			\bottomrule
	\end{tabular}}
\end{table}

A common alternative to ERM is to consider a distributionally robust objective, which optimizes for the worst-case training loss over a set of distributions, and has been shown to offer variance-reduction properties that benefit generalization~\citep[e.g.,][]{duchi2019variance,sinha2017certifying,chen2020distributionally,duchi2018learning}.
While not directly developed for distributional robustness, TERM also enables variance reduction for positive values of $t$ (Theorem~\ref{thm: variance-reduction}), which can be used to strike a better bias-variance trade-off for generalization. 
We compare TERM with several baselines including robustly regularized risk (RobustRegRisk)~\citep{duchi2019variance}, linear SVM~\citep{ren2018learning}, {Conditional Value-at-Risk (CVaR)~\citep{rockafellar2000optimization,soma2020statistical}},  LearnRewight~\citep{ren2018learning}, FocalLoss~\citep{lin2017focal}, and HRM~\citep{leqi2019human} on the HIV-1 dataset~\citep{hiv1,dua2019uci} originally investigated by~\citet{duchi2019variance}. We  examine the accuracy on the rare class ($Y=0$), the common class ($Y=1$), and overall accuracy. 

The mean and standard error of accuracies are reported in Table~\ref{table: dro}. RobustRegRisk and TERM offer similar performance improvements compared with other baselines, such as linear SVM, {CVaR}, LearnRewight, FocalLoss, and HRM. {Note that here RobustRegRisk~\citep{duchi2019variance} and CVaR~\citep{rockafellar2000optimization} can both be viewed as specific instances of the distributionally robust optimization framework, with different uncertainty sets.} For larger $t$, TERM achieves similar accuracy in both classes, while RobustRegRisk does not show similar trends by sweeping its hyperparameters. 
It is common to adjust the decision threshold to boost the accuracy on the rare class. 
We do this for ERM and RobustRegRisk and optimize the threshold so that ERM$_{+}$ and RobustRegRisk$_{+}$ result in the same validation accuracy on the rare class as TERM ($t=50$). TERM achieves similar performance to RobustRegRisk$_{+},$ without 
the need for an extra tuned hyperparameter.

\subsection{Solving Compound Issues: Hierarchical Multi-Objective Tilting}\label{sec:exp:multi_obj}

Finally, in this section, we focus on settings where multiple issues, e.g., class imbalance and label noise, exist in the data simultaneously. We discuss two possible instances of hierarchical multi-objective TERM to tackle such problems. One can think of other variants in this hierarchical tilting space which could be useful depending on applications at hand. 


\subsubsection{Class Imbalance and Random Noise} \label{sec:exp:hierarchical1}
We explore the  HIV-1 dataset~\citep{hiv1}, as in Section~\ref{sec:exp:fairness}. 
We report both overall accuracy and accuracy on the rare class in four  scenarios: {\bf (a) clean and 1:4}, the original dataset that is naturally slightly imbalanced with rare samples represented 1:4 with respect to the common class; {\bf (b) clean and 1:20}, where we subsample to introduce a 1:20 imbalance ratio; {\bf (c) noisy and 1:4}, which is the original dataset with labels associated with 30\% of the samples randomly reshuffled; and {\bf (d) noisy and 1:20}, where 30\% of the labels of the 1:20  imbalanced dataset are reshuffled.

\setlength{\tabcolsep}{3pt}
\begin{table}[h!]
\centering
\caption{Hierarchical TERM can address both class imbalance and noisy samples.}
\label{table:hierarchical}
\scalebox{0.68}{
\begin{tabular}{lcc|cc|cc|cc}
\toprule[\heavyrulewidth]
\multicolumn{1}{l}{\multirow{3}{*}{\textbf{objectives}}} & 
\multicolumn{8}{c}{\textbf{test accuracy} ({\fontfamily{qcr}\selectfont{HIV-1}})} \\
\cmidrule(r){4-7}
\multicolumn{1}{c}{~}& \multicolumn{4}{c}{clean data}  & \multicolumn{4}{c}{30\% noise} \\
\cmidrule(r){3-4}\cmidrule(r){7-8}
\multicolumn{1}{c}{~}& \multicolumn{2}{c}{1:4}  & \multicolumn{2}{c}{1:20} & \multicolumn{2}{c}{1:4} & \multicolumn{2}{c}{1:20} \\
\cmidrule(r){2-9}
  & $Y=0$ & overall & $Y=0$ & overall  & $Y=0$ & overall  & $Y=0$ & overall  \\
\hline
\rule{0pt}{2ex}ERM    & 0.822 {\tiny (.009)} &   0.934 {\tiny (.003)} & 0.503 {\tiny (.013)} & 0.888 {\tiny (.006)} & 0.656 {\tiny (.014)} & 0.911 {\tiny (.006)} & 0.240 {\tiny (.018)} & 0.831 {\tiny (.011)}\\ 
{CVaR~\citep{rockafellar2000optimization}} & {{\bf 0.844} {\tiny (.013)}} & {{\bf 0.937} {\tiny (.003)}} & {0.621 {\tiny (.011)}} & {0.906 {\tiny (.005)}} & {0.651 {\tiny (.015)}} & { 0.909 {\tiny (.006)}} & {0.252 {\tiny (.014)}} & {0.834 {\tiny (.010)}} \\
GCE~\citep{zhang2018generalized} & 0.822 {\tiny (.009)} &   0.934 {\tiny (.003)} & 0.503 {\tiny (.013)} & 0.888 {\tiny (.006)} & 0.732 {\tiny (.021)} &  {\bf 0.925} {\tiny (.005)} & 0.324 {\tiny (.017)}  & 0.849 {\tiny (.008)}\\
LearnReweight~\citep{ren2018learning} & \textbf{0.841} {\tiny (.014)} & 0.934 {\tiny (.004)} & 0.800 {\tiny (.022)} & 0.904 {\tiny (.003)} & 0.721 {\tiny (.034)} & 0.856 {\tiny (.008)} & 0.532 {\tiny (.054)} & 0.856 {\tiny (.013)} \\
RobustRegRisk (\hyperlink{cite.duchi2019variance}{Duchi et al., 2019}) &  {\bf 0.844}  \tiny (.010) & {\bf 0.939} {\tiny (.004)} & 0.622 {\tiny (.011)} & 0.906 {\tiny (.005)} & 0.634 {\tiny (.014)} & 0.907 {\tiny (.006)} & 0.051 {\tiny (.014)} & 0.792 {\tiny (.012)} \\
FocalLoss~\citep{lin2017focal}  &  {\bf 0.834} {\tiny (.013)} &  {\bf 0.937} {\tiny (.004)} & {\bf 0.806} {\tiny (.020)} & {\bf 0.918} {\tiny (.003)} & 0.638 {\tiny (.008)} & 0.908 {\tiny (.005)} & 0.565 {\tiny (.027)} & {\bf 0.890} {\tiny (.009)}\\
HAR~\citep{cao2021heteroskedastic}  & {\bf 0.842} {\tiny (.011)} & 0.936 {\tiny (.004)} & 0.817 {\tiny (.013)} & {\bf 0.926} {\tiny (.004)} & {\bf 0.870} {\tiny (.010)} & 0.915 {\tiny (.004)} & {\bf 0.800} {\tiny (.016)} & 0.867 {\tiny (.012)} \\
\rowcolor{myred}
TERM$_{sc}$ & {\bf 0.840}  {\tiny (.010)} &	{\bf 0.937} {\tiny (.004)} & {\bf 0.836} {\tiny (.018)} & {\bf 0.921} {\tiny (.002)} & {\bf 0.852} {\tiny (.010)} & {\bf 0.924} {\tiny (.004)} & {\bf 0.778} {\tiny (.008)} & {\bf 0.900} {\tiny (.005)} \\
\hline
\rowcolor{myblue}
 TERM$_{ca}$ & {\bf 0.844} {\tiny (.014)} & {\bf 0.938} {\tiny (.004)} & {\bf 0.834} {\tiny (.021)} & {\bf 0.918} {\tiny (.003)} & {\bf 0.846} {\tiny (.015)} & {\bf 0.933} {\tiny (.003)} & {\bf 0.806} {\tiny (.020)} & {\bf 0.901} {\tiny (.010)}\\
\bottomrule
\end{tabular}
}
\end{table}

In Table~\ref{table:hierarchical}, hierarchical TERM is applied at the sample level and class level (TERM$_{sc}$), where we use the sample-level tilt of $\tau{=}{-}2$ for noisy data. We use class-level tilt of $t{=}0.1$ for the 1:4 case and $t{=}50$ for the 1:20 case. We compare against baselines for robust classification and class imbalance (discussed previously in Sections~\ref{sec:exp:robustness} and~\ref{sec:exp:fairness}), where we tune them for best performance (Appendix~\ref{appen:exp_detail}). Similar to the experiments in Section~\ref{sec:exp:robustness}, we avoid using baselines that require clean validation data~\cite[e.g.,][]{roh2020fr}. We compare TERM with an additional baseline of HAR~\citep{cao2021heteroskedastic}, a recent work addressing the issues of noisy and rare samples simultaneously with adaptive Lipschitz regularization. While different baselines (except HAR) perform well in their respective problem settings, TERM and HAR are far superior to all baselines when considering noisy samples and class imbalance simultaneously (rightmost column in Table~\ref{table:hierarchical}). Finally, in the last row of Table~\ref{table:hierarchical}, we simulate the noisy annotator setting of Section~\ref{sec:exp:noisy_annotator} assuming that the data is coming from 10 annotators, i.e., in the 30\% noise case we have 7 hammers and 3 spammers. In this case, we apply hierarchical TERM at both class and annotator levels (TERM$_{ca}$), where we perform the higher level tilt at the annotator (group) level and the lower level tilt at the class level (with no sample-level tilting). We show that this approach can benefit noisy/imbalanced data even further (far right, Table~\ref{table:hierarchical}), while suffering only a small performance drop on the clean and noiseless data (far left, Table~\ref{table:hierarchical}).

\subsubsection{Class Imbalance and Adversarial Noise} \label{sec:exp:hierarchical2}

We evaluate hierarchical tilting on a more difficult task involving more adversarial noise with deep neural network models. We take the setup studied in~\citet{cao2021heteroskedastic}. The noise is created by exchanging labels of 40\% samples which come from similar classes (`cat' and `dog', `vehicle' and `automobile') in the CIFAR10 dataset. To simulate class imbalance, only 10\% of the training data from these four noisy classes are subsampled. For TERM, we apply group-level positive tilting by linearly scaling $t$ from 0 to 3, and perform sample-level negative tilting within each class with $\tau$ scaling from 0 to -2. Table~\ref{table:hierarchical_setup2} reports the results of hierarchical TERM (TERM$_{sc}$) compared with HAR~\citep{cao2021heteroskedastic} and other baselines. We see that TERM underperforms HAR, 
and outperforms all other approaches. Note that HAR is a more complicated method which requires to perform end-to-end training for two times with higher per-iteration complexity (involving second-order information), while TERM is a simple method and enjoys the same training time as that of ERM on this problem.

\begin{table}[h]
\caption{TERM outperforms most baselines addressing the co-existence of noisy samples and class imbalance by a large margin, and is worse than a more complicated method HAR.}
\centering
\label{table:hierarchical_setup2}
\scalebox{0.92}{
\begin{tabular}{lccc}
	   \toprule[\heavyrulewidth]
        \multicolumn{1}{l}{\multirow{2}{*}{\textbf{objectives}}} & \multicolumn{2}{c}{\textbf{test accuracy} ({\fontfamily{qcr}\selectfont{CIFAR10}}, ResNet32)} \\
        \cmidrule(r){2-3}
          & noisy, rare class & clean, common class \\
        \midrule
    ERM & 0.529 {(.012)} & \textbf{0.944} (.001) \\
     GCE~\citep{zhang2018generalized} & 0.482 {(.006)} & 0.916 {(.003)} \\
     MentorNet~\citep{jiang2018mentornet} & 0.541 (.010)  & 0.903 (.005)  \\
 MW-Net~\citep{sun2019mw-net} & 0.554 {(.011)} & 0.917 {(.005)} \\
HAR~\citep{cao2021heteroskedastic} & \textbf{0.635} {(.008)}  & \textbf{0.943} {(.002)}  \\
 \rowcolor{myblue}
 TERM$_{sc}$ & 0.585 {(.014)}  & {0.913} {(.003)} \\
\bottomrule[\heavyrulewidth]
\end{tabular}}
\end{table}

\section{Related Approaches in Machine Learning} \label{sec:background:related_work}


Here we discuss related problem-specific works in machine learning addressing deficiencies of ERM. We roughly group them into alternate aggregation schemes, alternate loss functions, and sample re-weighting schemes.

\paragraph{Alternate aggregation schemes.} 
A common alternative to the standard average loss in empirical risk minimization is to consider a {min-max} objective, which aims to minimize the max-loss. {Min-max} objectives are commonplace in machine learning, and have been used for a wide range of applications, such as ensuring fairness across subgroups~\citep{hashimoto2018fairness,mohri2019agnostic, stelmakh2019peerreview4all, samadi2018price,tantipongpipat2019multi,lahoti2020fairness}, enabling robustness under small perturbations~\citep{sinha2017certifying}, or generalizing to unseen domains~\citep{volpi2018generalizing}. As discussed in Section~\ref{sec:term_properties}, the TERM objective can be viewed as a minimax smoothing~\citep{kort1972new,pee2011solving} with the added flexibility of a tunable $t$ to allow the user to optimize utility for different quantiles of loss similar to superquantile approaches~\citep{rockafellar2000optimization, laguel2020device}, directly trading off between robustness/fairness and utility for positive and negative values of $t$ (see Section~\ref{sec:term_properties} for these connections). 
However, the TERM objective remains smooth (and efficiently solvable) for moderate values of $t$, resulting in faster convergence even when the resulting solutions are effectively the same as the min-max solution or other desired quantiles of the loss (as we demonstrate in the experiments of Section~\ref{sec:experiments}).  Interestingly, Cohen et al. introduce Simnets~\citep{cohen2014simnets,cohen2016deep}, with a similar exponential smoothing operator, though for a differing purpose of  achieving layer-wise operations {\it between} sum and max in deep neural networks.  

\paragraph{Alternate loss functions.} Rather than modifying the way the losses are aggregated, as in (smoothed) min-max or superquantile methods, it is also quite common to modify the losses themselves.
For example, in robust regression, it is common to consider losses such as the $L_1$ loss, Huber loss, or general $M$-estimators~{\citep{holland2019better}} as a way to mitigate the effect of outliers~\citep{bhatia2015robust}. \citet{wang2013robust} study a similar exponentially tilted loss for robust regression and characterize the break down point, though it is limited to the squared loss and only corresponds to $t{<}0$. Losses can also be modified to address outliers by favoring small losses~\citep{yu2012polynomial,zhang2018generalized} or gradient clipping~\citep{menon2020can}. Some works mitigate label noise by explicitly modeling noise distributions into end-to-end training combined with an additional noise model regularizer~\citep{jindal2016learning,jindal2019effective}. 
On the other extreme, the largest losses can  be magnified to encourage focus on hard samples~\citep{lin2017focal,wang2016adaptive, li2019fair}, which is a popular approach for curriculum learning. 
Constraints could also be imposed to promote fairness during the optimization procedure~\citep{hardt2016equality,donini2018empirical,rezaei2019fair,zafar2017fairness,baharlouei2019r,cotter2019optimization, lowy2021fermi, alghamdi2020model,zafar2019fairness,prost2019toward}. A line of work proposes $\alpha$-loss, which is able to promote fairness or robustness for classification tasks~\citep{sypherd2019tunable}.
Ignoring the log portion of the objective in~\eqref{eq: TERM}, TERM can  be viewed as an alternate loss function exponentially shaping the loss to achieve both of these goals with a single objective, i.e., magnifying hard examples with $t>0$ and suppressing outliers with $t<0$. In addition, we show that TERM can even achieve both goals simultaneously with hierarchical multi-objective optimization (Section~\ref{sec:exp:multi_obj}). 

\paragraph{Sample re-weighting schemes.} There exist approaches that implicitly modify the underlying ERM objective by re-weighting the influence of the samples themselves. These re-weighting schemes can be enforced in many ways. 
A simple and widely used example is to subsample training points in different classes.
Alternatively, one can re-weight examples according to their loss function when using a stochastic optimizer, which can be used to put more emphasis on ``hard'' or ``unfair'' examples~\citep{shrivastava2016training, jiang2019accelerating, katharopoulos2017biased,leqi2019human,abernethy2020active}. Re-weighting can also be implicitly enforced via the inclusion of a regularization parameter~\citep{abdelkarim2020long}, loss clipping~\citep{yang2010relaxed}, or modelling crowd-worker qualities~\citep{khetan2017learning}. 
Such an explicit re-weighting has been explored for other applications~\citep[e.g.,][]{lin2017focal,jiang2018mentornet,shu2019meta,chang2017active,gao2015active,ren2018learning}, though in contrast to these methods, TERM is applicable to a general class of loss functions, with theoretical  guarantees.  
TERM is equivalent to a dynamic re-weighting of the samples based on the values of the objectives (Lemma~\ref{lemma:TERM-gradient}), which could be viewed as a convexified version of loss clipping. {We note that such view holds more generally for all distributionally robust objectives~\citep{slowik2022distributionally}.} We compare to several sample re-weighting schemes empirically in Section~\ref{sec:experiments}.

\section{Discussion and Conclusion} \label{sec:conclusion}

In this manuscript, we have explored the use of exponential tilting in risk minimization, examining tilted empirical risk minimization (TERM) as a flexible extension to the ERM framework. We rigorously established connections between TERM and related objectives including VaR, CVaR, and DRO. 
We explored, both theoretically and empirically, TERM's ability to handle various known issues with ERM, such as robustness to noise, class imbalance, fairness, and  generalization, as well as more complex issues like the simultaneous existence of class imbalance and noisy outliers. 
Despite the straightforward modification TERM makes to traditional ERM objectives,
the framework consistently outperforms ERM and delivers competitive performance with state-of-the-art,
problem-specific methods on a wide range of applications.

Our work highlights the effectiveness and versatility of tilted objectives in machine learning. 
As such, our framework (TERM) could be widely used for applications both positive and negative. However, our hope is that the TERM framework will allow machine learning practitioners to easily modify the ERM objective to handle practical concerns such as enforcing fairness amongst subgroups, mitigating the effect of outliers, and ensuring robust performance on new, unseen data. One potential downside of the TERM objective is that if the underlying dataset is \textit{not} well-understood, incorrectly tuning $t$ could have the unintended consequence of \textit{magnifying} the impact of biased/corrupted data in comparison to traditional ERM. Indeed, critical to the success of such a framework is understanding the implications of the modified objective, both theoretically and empirically. The goal of this work is therefore to explore these implications so that it is clear when such a modified objective would be appropriate.

In terms of the use-cases explored with the TERM framework, we relied on benchmark datasets that have been commonly explored in prior work~\citep[e.g.,][]{yang2010relaxed,samadi2018price,tantipongpipat2019multi,yu2012polynomial}. However, we note that some of these common benchmarks, such as cal-housing~\citep{pace1997sparse} and Credit~\citep{yeh2009comparisons}, contain potentially sensitive information. 
While the goal of our experiments was to showcase that the TERM framework could be useful in learning fair representations that suppress membership bias and hence promote fairer performance, developing an understanding for---and removing---such membership biases requires a more comprehensive treatment of the problem that is outside the scope of this work.

In the future, in addition to generalization bounds of TERM, it would be interesting to further explore applications of tilted losses in machine learning. 
We note that since the early TERM work~\citep{TERM} was made public, there are several subsequent works applying (variants of) TERM to handle other real-world ML applications~\citep{szabo2021tilted,zhou2021robust}, or exploring risk bounds on differential private TERM~\citep{lowy2021output}, which suggest rich implications and wide applicability of TERM, beyond what is studied in this work.



\newpage

\appendix

\section*{Appendix}

In this appendix, we provide full statements and proofs of the analyses presented in Section~\ref{sec:term_properties}-Section~\ref{sec:var} (Appendix~\ref{app:properties} and~\ref{app:superquantile}); details and convergence proof on the methods we propose for solving TERM (Appendix~\ref{app:solving}), and complete empirical results and details of our empirical setup (Appendix~\ref{appen:exp_full}). We provide a table of contents below for easier navigation.

\etocdepthtag.toc{mtappendix}
\etocsettagdepth{mtchapter}{none}
\etocsettagdepth{mtappendix}{subsection}
{
\parskip=0em
\tableofcontents
}

\newpage

\section{Properties and Interpretations (Proofs and Additional Results)} \label{app:properties}

{
In this section, we provide the proofs of the main results in the paper, along with additional results on the properties of TERM objective, its solution, as well as the corresponding solvers.} 

\subsection{Proofs of Basic Properties of the TERM Objective}
\label{app:general-TERM-properties}
We first provide proofs for the basic properties of the TERM objective. 


\paragraph{Proof of Lemma~\ref{lem: lipschitz}.} The conclusion follows by noting that for any $\theta_1, \theta_2 \in \Theta$,
\begin{align}
  \left | \wR(t;\theta_1) - \wR(t;\theta_2) \right | &=   \left | \frac{1}{t} \log \left(\frac{1}{N} \sum_{i \in [N]} e^{tf(x_i;\theta_1)}\right) - \frac{1}{t} \log \left(\frac{1}{N} \sum_{i \in [N]} e^{tf(x_i;\theta_2)}\right) \right | \\
  &=  \left | \frac{1}{t} \log \left(\frac{\sum_{i \in [N]} e^{tf(x_i;\theta_1)}}{\sum_{i \in [N]} e^{tf(x_i;\theta_2)}}\right)\right | \\
  & \leq  \left | \frac{1}{t} \log \left(e^{t L \|\theta_1-\theta_2\|_2} \frac{\sum_{i \in [N]} e^{tf(x_i;\theta_2)}}{\sum_{i \in [N]} e^{tf(x_i;\theta_2)}} \right)\right | \\
  &= L \|\theta_1-\theta_2\|_2.
\end{align}
\hfill \qedsymbol

\paragraph{Proof of Lemma~\ref{lem: Hessian}.}
Recall that 
\begin{align}
        \nabla_\theta \wR(t ;\theta) &= \frac{\sum_{i \in [N]}\nabla_{\theta}f(x_i;\theta)e^{t f(x_i; \theta)}}{\sum_{i \in [N]} e^{t f(x_i; \theta)}}\\
        & = {\frac{1}{N}} \sum_{i \in [N]}\nabla_{\theta}f(x_i;\theta)e^{t ( f(x_i; \theta) - \wR(t; \theta))}.
\end{align}
The proof of the first part is completed by differentiating again with respect to $\theta,$ followed by algebraic manipulation. To prove the second part, notice that the term in~\eqref{eq:smoothness-23} is positive semi-definite, whereas the term in~\eqref{eq:smoothness-24} is positive definite and lower bounded by $\beta_{\min} \mathbf{I}$ (see Assumption~\ref{assump: regularity}, Eq.~\eqref{eq: assump1-strong-convex}).\hfill \qedsymbol

\paragraph{Proof of Lemma~\ref{lemma:TERM-smoothness}.}
Let us first provide a proof for $t\in \mathbb{R}^-$.
Invoking Lemma~\ref{lem: Hessian} and Weyl's inequality~\citep{weyl1912asymptotische}, we have
\begin{align}
          \lambda_{\max} &\left( \nabla^2_{\theta \theta^\top}\wR(t ;\theta) \right)   \nonumber\\
          &\leq \lambda_{\max}\left(
        {\frac{\color{black}t}{N}} \sum_{i \in [N]} ( \nabla_{\theta}f(x_i;\theta) - \nabla_\theta \wR(t ;\theta))(\nabla_{\theta}f(x_i;\theta) - \nabla_\theta \wR(t ;\theta))^\top  e^{t (f(x_i; \theta) - \wR(t; \theta))}\right)\\
        &+ \lambda_{\max}\left( {\frac{1}{N}} \sum_{i \in [N]} \nabla^2_{\theta \theta^\top}f(x_i;\theta) e^{t ( f(x_i; \theta) - \wR(t; \theta))}\right)\\
        & \leq \beta_{\max},
\end{align}
where we have used the fact that the term in~\eqref{eq:smoothness-23} is  negative semi-definite for $t<0$, and that the term in~\eqref{eq:smoothness-24} is positive definite for all $t$ with smoothness bounded by $\beta_{\max}$ ({which would hold from smoothness of $f(x_i;\theta)$}; see Assumption~\ref{assump: regularity},~Eq.~\eqref{eq: assump1-strong-convex}).

For $t \in \mathbb{R}^{>0},$ following Lemma~\ref{lem: Hessian} and Weyl's inequality~\citep{weyl1912asymptotische}, we have
\begin{align}
        \left(\frac{1}{t}\right) & \lambda_{\max}  \left(  \nabla^2_{\theta \theta^\top}\wR(t ;\theta) \right)\nonumber\\
        & \leq \lambda_{\max}\left( 
        {\frac{1}{N}} \sum_{i \in [N]} ( \nabla_{\theta}f(x_i;\theta) - \nabla_\theta \wR(t ;\theta))(\nabla_{\theta}f(x_i;\theta) - \nabla_\theta \wR(t ;\theta))^\top  e^{t (f(x_i; \theta) - \wR(t; \theta))}\right)\\
        & + \left(\frac{1}{t}\right)\lambda_{\max} \left({\frac{1}{N}}\sum_{i \in [N]} \nabla^2_{\theta \theta^\top}f(x_i;\theta) e^{t ( f(x_i; \theta) - \wR(t; \theta))}\right).
\end{align}
{Due to Weyl’s inequality, the smoothness of $f(x_i; \theta)$, and the fact that $\frac{1}{N}\sum_{i \in [N]} e^{t(f(x_i;\theta)-\wR(t;\theta))} = 1$, $\sum_{i \in [N]} \nabla^2_{\theta \theta^\top}f(x_i;\theta) e^{t ( f(x_i; \theta) - \wR(t; \theta))}$ is bounded. Consequently,}
\begin{equation}
  \lim_{t \to +\infty} \left(\frac{1}{t}\right)  \lambda_{\max} \left(  \nabla^2_{\theta \theta^\top}\wR(t ;\theta) \right)  <+\infty.
\end{equation}
On the other hand, following Weyl's inequality~\citep{weyl1912asymptotische},
\begin{align}
            \lambda_{\max} &\left(  \nabla^2_{\theta \theta^\top}\wR(t ;\theta) \right) \nonumber \\
            &\geq t \lambda_{\max} \left( {\frac{1}{N}} 
        \sum_{i \in [N]} ( \nabla_{\theta}f(x_i;\theta) - \nabla_\theta \wR(t ;\theta))(\nabla_{\theta}f(x_i;\theta) - \nabla_\theta \wR(t ;\theta))^\top  e^{t (f(x_i; \theta) - \wR(t; \theta))}\right),
\end{align}
and hence,
\begin{equation}
   \lim_{t \to +\infty}  \left(\frac{1}{t}\right) \lambda_{\max} \left(  \nabla^2_{\theta \theta^\top}\wR(t ;\theta) \right) >0,
\end{equation}
where we have used the fact that no solution $\theta$ exists that would make all $f_i$'s vanish (Assumption~\ref{assump: regularity}).\hfill \qedsymbol

Under the strict saddle property (Assumption~\ref{assump:strict-saddle}), it is known that gradient-based methods would converge to a local minimum~\citep{ge2015escaping}, i.e., $\breve{\theta}(t)$ would be obtained using gradient descent (GD). 
The rate of convergence of GD scales linearly with the smoothness parameter of the optimization landscape, which is characterized by Lemma~\ref{lemma:TERM-smoothness}.

\paragraph{Proof of Lemma~\ref{lemma: special_cases}.}
For $t \to 0,$
\begin{align}
  \lim_{t \to 0}  \wR(t ;\theta) &=   \lim_{t \to 0}  \frac{1}{t} \log \left(\frac{1}{N} \sum_{i \in [N]} e^{t f(x_i; \theta)} \right)\nonumber\\
    & = \lim_{t \to 0} \frac{\sum_{i \in [N]} f(x_i; \theta)e^{t f(x_i; \theta)}}{\sum_{i \in [N]} e^{t f(x_i; \theta)}}\label{eq: follow-lopital-2}\\
    & = \frac{1}{N}\sum_{i \in [N]} f(x_i; \theta),
\end{align}
where~\eqref{eq: follow-lopital-2} is due to  L'H\^opital's rule {applied to $t$ as the denominator and $\log \left(\frac{1}{N} \sum_{i \in [N]} e^{t f(x_i; \theta)} \right)$ as the numerator}.

For $t \to - \infty$, we proceed as follows:
\begin{align}
  \lim_{t \to -\infty}  \wR(t ;\theta) &=   \lim_{t \to -\infty} \frac{1}{t} \log \left(\frac{1}{N} \sum_{i \in [N]} e^{t f(x_i; \theta)} \right)\nonumber\\
    & \leq \lim_{t \to -\infty} \frac{1}{t} \log \left(\frac{1}{N} \sum_{i \in [N]} e^{t \min_{j \in [N]}f(x_j; \theta)} \right)\\
    & = \min_{i \in [N]}f(x_i; \theta).\label{eq:upper}
\end{align}
On the other hand,
\begin{align}
  \lim_{t \to -\infty}  \wR(t ;\theta) &=   \lim_{t \to -\infty} \frac{1}{t} \log \left(\frac{1}{N} \sum_{i \in [N]} e^{t f(x_i; \theta)} \right)\nonumber\\
    & \geq \lim_{t \to -\infty} \frac{1}{t} \log \left(\frac{1}{N}  e^{t \min_{j \in [N]}f(x_j; \theta)} \right)\\
    & = \min_{i \in [N]}f(x_i; \theta) - \lim_{t \to -\infty} \left\{ \frac{1}{t} \log N \right\}\\
    & = \min_{i \in [N]}f(x_i; \theta).\label{eq:lower}
\end{align}
Hence, the proof follows by putting together~\eqref{eq:upper} and~\eqref{eq:lower}.

The proof proceeds similarly to $t \to - \infty$ for $t \to +\infty$ and is omitted for brevity. \hfill \qedsymbol

\subsection{General Properties of the Objective for GLMs}
\label{app:general-TERM-GLM}

In this section, even if not explicitly stated, all results are derived under Assumption~\ref{assump: expnential_family} with a generalized linear model and loss function of the form~\eqref{eq: loss-exponential}, effectively assuming that the loss function is the negative log-likelihood of an exponential family~\citep{wainwright2008graphical}.

\begin{definition}[Empirical cumulant generating function]
Let
\begin{equation}
\label{eq: def-cumulant}
\widetilde{\Lambda}(t; \theta) := t \wR (t; \theta) .   
\end{equation}
\end{definition}
\begin{definition}[Empirical log-partition function~\citep{wainwright2005new}]
Let $\Gamma(t; \theta)$ be
\begin{equation}
    \Gamma(t; \theta) :=  \log \left(\frac{1}{N}\sum_{i \in [N]} e^{-t \theta^\top T(x_i) }\right).
\end{equation}
\end{definition}

Thus, we have
\begin{equation}
\label{eq: R-A-gamma}
    \wR(t; \theta) = A(\theta) + \frac{1}{t} \log \left(\frac{1}{N}\sum_{i \in [N]} e^{-t \theta^\top T(x_i)}\right) = A(\theta) + \frac{1}{t} \Gamma(t; \theta).
\end{equation}

\begin{definition}[{Tilted} empirical mean and empirical variance of the sufficient statistic]
Let $\cM$ and $\cV$ denote the mean and the variance of the sufficient statistic, and be given by 
\begin{align}
    \cM(t; \theta) &:= \frac{1}{N}\sum_{i \in [N]} T(x_i) e^{-t \theta^\top T(x_i)  - \Gamma(t;\theta)},\\
    \cV (t; \theta) &:= \frac{1}{N}\sum_{i\in [N]} (T(x_i) - \mathcal{M}(t;\theta)) (T(x_i)- \mathcal{M}(t;\theta))^\top e^{- t \theta^\top T(x_i)  - \Gamma(t; \theta)}. 
\end{align}
\end{definition}
{We notice that $\cM(t; \theta)$ and $\cV (t; \theta)$ defined here are equivalent to tilted empirical mean/variance in the main text (Eq.~\eqref{eq: tilted_mean} and Eq.~\eqref{eq: tilted_variance}) over sufficient statistic, i.e.,
\begin{align}
      \cM(t; \theta) &= \sum_{i \in [N]} w_i(t; \theta) T(x_i),\\
    \cV (t; \theta) &= \sum_{i\in [N]} w_i(t; \theta) (T(x_i) - \mathcal{M}(t;\theta)) (T(x_i)- \mathcal{M}(t;\theta))^\top.
\end{align}
Similarly, as a special case of $t$-tilted empirical mean/variance (Eq.~\eqref{eq: t_tilted_mean} and Eq.~\eqref{eq: t_tilted_variance}), $t$-tilted empirical mean/variance over sufficient statistic are defined as
\begin{align}
    \cM_t &:=  \cM(t; \breve{\theta}(t)), \label{eq: t_tilted_mean_statistic} \\
    \cV_t & := \cV(t; \breve{\theta}(t)). \label{eq: t_tilted_variance_statistic}
\end{align}
The quantities $\cM(t; \theta),  \cV (t; \theta), \cM_t$, and $ \cV_t$ will be used for proving general properties of TERM solutions in this section.}

\begin{lemma}
\label{lem: V}
For all $t \in \mathbb{R},$ we have
$
    \mathcal{V}(t; \theta) \succeq 0.
$
\end{lemma}

Next we state a few key relationships that we will use in our characterizations. The proofs are straightforward and omitted for brevity.
\begin{lemma}[Partial derivatives of $\Gamma$]
\label{lemma: parrial-Gamma}
For all $t \in \mathbb{R}$ and all $\theta \in \Theta,$
\begin{align}
    \frac{\partial}{\partial t} \Gamma(t; \theta) &= - \theta^\top \mathcal{M}(t;\theta), \\
    \nabla_\theta \Gamma(t; \theta) &= - t \mathcal{M}(t; \theta).
\end{align}
\end{lemma}

\begin{lemma}[Partial derivatives of $\cM$]
\label{lemma: partial-M}
For all $t \in \mathbb{R}$ and all $\theta \in \Theta,$
\begin{equation}
    \frac{\partial}{\partial t} \cM(t; \theta) = - \mathcal{V}(t; \theta) \theta,
\end{equation}
\begin{equation}
    \nabla_{\theta} \mathcal{M} (t; \theta) = - t \mathcal{V}(t; \theta).
\end{equation}
\end{lemma}




The next few lemmas characterize the partial derivatives of the cumulant generating function.
\begin{lemma}(Derivative of $\widetilde{\Lambda}$ with $t$)
For all $t \in \mathbb{R}$ and all $\theta \in \Theta,$
\begin{equation}
    \frac{\partial}{\partial t} \widetilde{\Lambda} (t; \theta) =  A(\theta) - \theta^\top \mathcal{M}(t; \theta).
\end{equation}
\end{lemma}
\begin{proof}
The proof is carried out by
\begin{equation}
    \frac{\partial}{\partial t} \widetilde{\Lambda} (t; \theta) = 
    A(\theta) - \theta^\top \sum_{i \in [N]} T(x_i) e^{-t \theta^\top T(x_i) - \Gamma(t;\theta)} = A(\theta) - \theta^\top \mathcal{M}(t; \theta).
\end{equation}
\end{proof}


\begin{lemma}[Second derivative of $\widetilde{\Lambda}$ with $t$]
\label{lemma: second-derivative-Lambda}
For all $t \in \mathbb{R}$ and all $\theta \in \Theta,$
\begin{equation}
    \frac{\partial^2}{\partial t^2} \widetilde{\Lambda} (t; \theta) =  \theta^\top \mathcal{V}(t; \theta)\theta.
\end{equation}
\end{lemma}

\begin{lemma}[Gradient of $\widetilde{\Lambda}$ with $\theta$]\label{lem: dLambda-theta}
For all $t \in \mathbb{R}$ and all $\theta \in \Theta,$
\begin{equation}
    \nabla_\theta \widetilde{\Lambda}(t; \theta) =  t \nabla_\theta A(\theta) - 
    t \mathcal{M}(t;\theta).
\end{equation}
\end{lemma}

\begin{lemma}[Hessian of $\widetilde{\Lambda}$ with $\theta$]\label{lem: dLambda-theta-theta}
For all $t \in \mathbb{R}$ and all $\theta \in \Theta,$
\begin{equation}
    \nabla^2_{\theta \theta^\top} \widetilde{\Lambda}(t; \theta) = t \nabla^2_{\theta\theta^\top} A(\theta) + t^2 \mathcal{V}(t; \theta).
\end{equation}
\end{lemma}
\begin{lemma}[Gradient of $\widetilde{\Lambda}$ with respect to $t$ and $\theta$]\label{lem: dLambda-t-theta}
For all $t \in \mathbb{R}$ and all $\theta \in \Theta,$
\begin{equation}
  \frac{\partial}{\partial t } \nabla_\theta \widetilde{\Lambda}(t; \theta)  = \nabla_{\theta} A(\theta) - \mathcal{M}(t; \theta) + t \mathcal{V}(t; \theta) \theta.
\end{equation}
\end{lemma}

\paragraph{Proof of Theorem~\ref{thm: obj-increasing}.}
Following~\eqref{eq: R-A-gamma}, 
\begin{align}
    \frac{\partial}{\partial t}\wR(t ;\theta) &= \frac{\partial}{\partial t} \left\{\frac{1}{t} \Gamma(t; \theta) \right\} \\
    & = -\frac{1}{t^2} \Gamma(t; \theta) - \frac{1}{t}\theta^\top \mathcal{M}(t; \theta) , \label{eq: follow-partial-gamma}\\
    & =: g(t; \theta),\label{eq:def-g}
\end{align}
where~\eqref{eq: follow-partial-gamma} follows from Lemma~\ref{lemma: parrial-Gamma}, and~\eqref{eq:def-g} defines $g(t; \theta).$

Let $g(0; \theta) := \lim_{t \to 0} g(t; \theta)$
Notice that
\begin{align}
    g(0; \theta) &=  \lim_{t \to 0} \left\{-\frac{1}{t^2} \Gamma(t; \theta) -\frac{1}{t} \theta^\top \mathcal{M}(t; \theta) \right\}\\
    &= - \lim_{t \to 0} \left\{ \frac{\frac{1}{t} \Gamma(t; \theta) + \theta^\top \cM(t; \theta)}{t}\right\}\\
    & =  \theta^\top \cV(0;\theta) \theta, \label{eq: follow-lopital}
\end{align}
where~\eqref{eq: follow-lopital} is due to L'H\^0pital's rule and Lemma~\ref{lemma: second-derivative-Lambda}.
 Now consider
\begin{align}
 \frac{\partial}{\partial t} \left\{ t^2 g(t; \theta)\right\} & = 
 \frac{\partial}{\partial t} \left\{-\Gamma(t; \theta) - t\theta^\top \cM(t; \theta)   \right\} \\
 &=  \theta^\top \cM(t; \theta) \label{eq: follow-dgamma} \\
 & \quad - \theta^\top \cM(t; \theta) + t \theta^\top \cV(t; \theta) \theta \label{eq: follow-dM}\\
 & = t \theta^\top \cV(t; \theta) \theta, \label{eq:137}
\end{align}
where $g(t; \theta)= \frac{\partial}{\partial t} \wR(t; \theta)$,~\eqref{eq: follow-dgamma} follows from Lemma~\ref{lemma: parrial-Gamma},~\eqref{eq: follow-dM} follows from the chain rule and Lemma~\ref{lemma: partial-M}.
Hence, $t^2 g(t; \theta)$ is an increasing function of $t$ for $t \in \mathbb{R}^{>0}$, and a decreasing function of $t$ for $t \in \mathbb{R}^-$, taking its minimum at $t = 0.$
Hence, $t^2 g(t; \theta)\geq 0$ for all $t \in \mathbb{R}.$ This implies that $g(t; \theta) \geq 0$ for all $t \in \mathbb{R},$ which in conjunction with~\eqref{eq:def-g} implies the statement of the theorem.


\subsection{General Properties of TERM Solutions for GLMs} \label{app: general-TERM-GLM-solution}
Next, we characterize some of the general properties of the solutions of  TERM objectives. Note that these properties are established under Assumptions~\ref{assump: expnential_family} and~\ref{assump:strict-saddle}.

\begin{lemma} \label{lem: Lambda-solution}
For all $t \in \mathbb{R},$  
\begin{equation}
 \nabla_\theta \widetilde{\Lambda}(t; \breve{\theta}(t)) = 0.
\end{equation}
\end{lemma}
\begin{proof}
The proof follows from definition and the assumption that $\Theta$ is an open set.
\end{proof}
\begin{lemma}
\label{lem: A-M}
For all $t \in \mathbb{R},$
\begin{equation}
 \nabla_\theta A(\breve{\theta}(t)) = \mathcal{M}(t; \breve{\theta}(t)).
\end{equation}
\end{lemma}
\begin{proof}
The proof is completed by noting Lemma~\ref{lem: Lambda-solution} and Lemma~\ref{lem: dLambda-theta}.
\end{proof}

\begin{lemma}[Derivative of the solution with respect to tilt]
\label{lem:implicit-func-theorem}
Under Assumption~\ref{assump:strict-saddle}, for all $t \in \mathbb{R},$
\begin{equation}
    \frac{\partial}{\partial t} \breve{\theta}(t) =  -  \left(  \nabla^2_{\theta\theta^\top} A(\breve{\theta}(t)) + t \mathcal{V}(t; \breve{\theta}(t))  \right)^{-1}  \mathcal{V}(t; \breve{\theta}(t)) \breve{\theta}(t),
\end{equation}
where
\begin{equation}
    \nabla^2_{\theta\theta^\top} A(\breve{\theta}(t)) + t \mathcal{V}(t; \breve{\theta}(t)) \succ 0
\end{equation}
{is a symmetric positive definite matrix.}
\end{lemma}
\begin{proof}
By noting Lemma~\ref{lem: Lambda-solution}, and further differentiating with respect to $t$, we have
\begin{align}
     0 &= \frac{\partial}{\partial t}\nabla_\theta \widetilde{\Lambda}(t; \breve{\theta}(t)) \\
     & = \left.\frac{\partial}{\partial \tau}\nabla_\theta \widetilde{\Lambda}(\tau; \breve{\theta}(t))\right|_{\tau = t} +  \nabla^2_{\theta \theta^\top} \widetilde{\Lambda}(t; \breve{\theta}(t)) \left(\frac{\partial}{\partial t} \breve{\theta}(t) \right) \label{eq: eq:2}\\
     & = t \cV(t; \breve{\theta}(t)) \breve{\theta}(t) + \left( t \nabla^2_{\theta\theta^\top} A(\theta) + t^2 \mathcal{V}(t; \theta)\right)\left(\frac{\partial}{\partial t} \breve{\theta}(t) \right) \label{eq: eq:3},
\end{align}
where~\eqref{eq: eq:2} follows from the chain rule,~\eqref{eq: eq:3} follows from Lemmas~\ref{lem: dLambda-t-theta} and~\ref{lem: A-M} and~\ref{lem: dLambda-theta-theta}. The proof is completed by noting that $\nabla^2_{\theta \theta^\top} \widetilde{\Lambda}(t; \breve{\theta}(t))$ {is symmetric positive definite} for all $t \in \mathbb{R}$ under Assumption~\ref{assump:strict-saddle}.
\end{proof}


Finally, we state an auxiliary lemma that will be used in the proof of the main theorem.
\begin{lemma}\label{lem: auxiliary}
For all $t, \tau \in \mathbb{R}$ and all $\theta \in \Theta,$
\begin{align}
    \mathcal{M}(\tau; \theta) - \mathcal{M}(t; \theta) &= -\left(\int_{t}^{\tau} \mathcal{V}(\nu; \theta)  d\nu\right) \theta.
\end{align}
\end{lemma}
\begin{proof}The proof is completed by noting that
\begin{align}
    \mathcal{M}(\tau; \theta) - \mathcal{M}(t; \theta) &= \int_{t}^{\tau} \frac{\partial}{\partial \nu} \mathcal{M}(\nu; \theta) d\nu = -\left(\int_{t}^{\tau} \mathcal{V}(\nu; \theta)  d\nu\right) \theta.
\end{align}
\end{proof}

\paragraph{Proof of Theorem~\ref{thm: opt-obj-increasing}.}
Notice that for all $\theta$, and all $\epsilon \in \mathbb{R}^{>0},$
\begin{align}
    \wR(t+\epsilon; \theta) &\geq \wR(t; \theta)\label{eq: follow-thm-2}\\
    & \geq \wR(t; \breve{\theta}(t)),\label{eq: optimality-follow}
\end{align}
where~\eqref{eq: follow-thm-2} follows from Theorem~\ref{thm: obj-increasing} and~\eqref{eq: optimality-follow} follows from the definition of $\breve{\theta}(t)$. Hence,
\begin{equation}
    \wR(t+\epsilon; \breve{\theta}(t+\epsilon)) = 
    \min_{\theta \in B(\breve{\theta}(t), r)} \wR(t+\epsilon; \theta) \geq \wR(t; \breve{\theta}(t)),
\end{equation}
which completes the proof. \hfill \qedsymbol

\paragraph{Proof of Theorem~\ref{thm: variance-reduction}.}
{Recall that 
$f(x_i; \theta ) = A(\theta) - \theta^\top T(x_i).$ Thus, 
\begin{equation}
  \widehat{E}_t(\mathbf{f}(\theta)) =  \sum_{i \in [N]} w_i(t; \breve{\theta}(t)) f(x_i;\theta) =  A(\theta) - \theta^\top \sum_{i \in [N]} w_i(t; \breve{\theta}(t)) T(x_i) = A(\theta) - \theta^\top \cM_t,
\end{equation}
where $\cM_t$ is defined in~\eqref{eq: t_tilted_mean_statistic}.
Consequently, 
\begin{align}
 \widehat{\var}_t(\mathbf{f}(\theta)) &=  \widehat{E}_t \left( f(x_i; \theta) - \widehat{E}_t(\mathbf{f}(\theta) \right)^2\\
 & = \widehat{E}_t \left( \theta^\top T(x_i) - \theta^\top \cM_t \right)^2\\
 & = \theta^\top \widehat{E}_t  \left( (T(x_i) - \cM_t)(T(x_i) - \cM_t)^\top \right) \theta\\
 & = \theta^\top \mathcal{V}_t \theta,
\end{align}
where $\mathcal{V}_t$ is defined in~\eqref{eq: t_tilted_variance_statistic}.
Hence, 
\begin{align}
     \frac{\partial}{\partial \tau} \left\{ \widehat{\var}_t (\mathbf{f}(\breve{\theta}(\tau)))\right\} & = \left( \frac{\partial}{\partial \tau} \breve{\theta}(\tau) 
 \right)^\top \nabla_\theta \left\{ \widehat{\var}_t (\mathbf{f}(\breve{\theta}(\tau)))\right\}\\
 & = 2\left( \frac{\partial}{\partial \tau} \breve{\theta}(\tau) 
 \right)^\top \mathcal{V}_t \breve{\theta}(\tau)\\
 & = -2\breve{\theta}^\top(\tau)
 \mathcal{V}(\tau; \breve{\theta}(\tau)) 
 \left(  \nabla^2_{\theta\theta} A(\breve{\theta}(\tau)) + \tau \mathcal{V}(\tau; \breve{\theta}(\tau))  \right)^{-1}  \mathcal{V}_t \breve{\theta}(\tau),
\end{align}
and in turn 
\begin{align}
     \left. \frac{\partial}{\partial \tau} \left\{ \widehat{\var}_t(\mathbf{f}(\breve{\theta}(\tau)))\right\} \right |_{\tau=t}  \leq 0,
\end{align}
where we have used the fact that $\mathcal{V}_\tau 
 \left(  \nabla^2_{\theta\theta} A(\breve{\theta}(\tau)) + \tau \mathcal{V}_\tau  \right)^{-1}  \mathcal{V}_\tau $ is a symmetric positive semidefinite matrix (due to Lemma~\ref{lem: V}), hence completing the proof. \hfill \qedsymbol}

\paragraph{Proof of Theorem~\ref{thm: uniform_gradient_weights}.}
Notice that
\begin{align}
    H\left({\bf w}(t; \theta)\right) &= -  \sum_{i \in [N]} w_i(t; \theta) \log w_i(t; \theta)  \\
    & = -\frac{1}{N} \sum_{i \in [N]}  (tf(x_i; \theta) - \widetilde{\Lambda}(t;\theta)) e^{tf(x_i; \theta) - \widetilde{\Lambda}(t;\theta)} \\
    & = \widetilde{\Lambda}(t;\theta) -t \frac{1}{N} \sum_{i \in [N]}  f(x_i; \theta)  e^{t f(x_i; \theta)  - \widetilde{\Lambda}(t;\theta) } \\
    & = \widetilde{\Lambda}(t;\theta) - t A(\theta) + t\theta^\top \mathcal{M}(t;\theta).
\end{align}
Thus,
\begin{align}
    \nabla_\theta H\left({\bf w}(t; \theta)\right) &=
\nabla_\theta  \left( \widetilde{\Lambda}(t;\theta) - t A(\theta) + t\theta^\top \mathcal{M}(t;\theta) \right)\\
& = t \nabla_\theta A(\theta) - t \mathcal{M}(t; \theta) - t\nabla_\theta A(\theta) + t \mathcal{M}(t;\theta) - t^2\mathcal{V}(t;\theta) \theta\\
&= - t^2 \mathcal{V}(t;\theta) \theta.
\end{align}
Hence,
\begin{align}
  \frac{\partial}{\partial \tau} H\left({\bf w}(t; \breve{\theta}(\tau))\right) &= \left( \frac{\partial}{\partial \tau} \breve{\theta}(\tau) 
 \right)^\top  \nabla_\theta H\left({\bf w}( t; \breve{\theta}(\tau))\right)
 \\
 &= \left( \frac{\partial}{\partial \tau} \breve{\theta}(\tau) 
 \right)^\top
\nabla_\theta  \left( \widetilde{\Lambda}(t;\theta) - t A(\theta) + t\theta^\top \mathcal{M}(t;\theta) \right)\\
& = t^2 \breve{\theta}^\top(\tau)
 \mathcal{V}(\tau; \breve{\theta}(\tau)) 
 \left(  \nabla^2_{\theta\theta} A(\breve{\theta}(\tau)) + \tau \mathcal{V}(\tau; \breve{\theta}(\tau))  \right)^{-1}   \mathcal{V}(t;\breve{\theta}(\tau)) \breve{\theta}(\tau)
\end{align}
and 
\begin{align}
   \left. \frac{\partial}{\partial \tau} H\left({\bf w}(t; \breve{\theta}(\tau))\right) \right|_{t=\tau} \geq 0,
\end{align}
completing the proof. \hfill \qedsymbol

{There are different ways to define performance uniformity. In Theorem~\ref{thm: cosine-similarity}, we further prove that the tilted cosine similarity between the scaled loss vector and the all-ones vector increases as $t$ decreases by a small amount, which shows that larger $t$ promotes a more \textit{uniform} performance across all losses and can have implications for fairness defined as representation disparity~\citep{hashimoto2018fairness} (Section~\ref{sec:exp:fairness}).} 

{
\begin{definition}[$t$-tilted cosine similarity]
For $\mathbf{u}, \mathbf{v} \in \mathbb{R}^N,$
let cosine similarity be defined as
\begin{equation}
    s(\mathbf{u}, \mathbf{v}) := \frac{\mathbf{u}^\top \mathbf{v}}{\|\mathbf{u}\|_2 \|\mathbf{v}\|_2}.
\end{equation}
For a weight vector $\mathbf{w},$ let the weighted cosine similarity be defined as
\begin{equation}
    s_{\mathbf{w}} (\mathbf{u}, \mathbf{v}) := s\left(\sqrt{\mathbf{W}} \mathbf{u}, \sqrt{\mathbf{W}} \mathbf{v}\right),
\end{equation}
where $\mathbf{W} := \text{diag}(\mathbf{w}).$
In particular, we call $ s_{\mathbf{w}(t; \breve{\theta}(t))} (\cdot, \cdot)$ the $t$-tilted cosine similarity.
\end{definition}

\begin{theorem}[$t$-tilted cosine similarity of the scaled loss vector and the all-ones vector increases with $t$]\label{thm: cosine-similarity}
Let
\begin{equation}
    \mathbf{f}^+(\theta):= \left\{f(x_i; \theta) - \wF(-\infty)\right\}_{i \in [N]},
\end{equation} 
where $\wF(-\infty)$ is defined in Eq.~\eqref{eq: opt_obj}, and let $\mathbf{1}_N$ denote the all-one $N$-vector.
 Then, under Assumption~\ref{assump: expnential_family} and Assumption~\ref{assump:strict-saddle}, for any $t \in \mathbb{R},$
\begin{equation}
    \left. \frac{\partial}{\partial t} \left\{ s_{\mathbf{w}(\tau, \breve{\theta}(\tau))}\left(\mathbf{f}^+(\breve{\theta}(t)) , \mathbf{1}_N \right)\right\} \right|_{\tau = t}>0,
\end{equation}
where $\mathbf{w}(t; \breve{\theta}(t))$ is the tilted weight vector defined in Eq.~\eqref{eq: w_i}.
\end{theorem}
\begin{proof}
Notice that 
\begin{equation}
    s_{\mathbf{w}(t; \breve{\theta}(t))}(\mathbf{f}^+(\theta), \mathbf{1}_N) = \frac{\widehat{E}_t f(x_i; \theta)-\wF(-\infty) }{\sqrt{\widehat{E}_t (f(x_i; \theta)-\wF(-\infty))^2}}.
\end{equation}
Hence,
\begin{align}
    \widehat{E}_t f(x_i; \theta) -\wF(-\infty) &=  A(\theta) - \theta^\top \mathcal{M}_t -\wF(-\infty),\\
    \widehat{E}_t (f(x_i; \theta) -\wF(-\infty))^2 & = (A(\theta) - \theta^\top \mathcal{M}_t -\wF(-\infty))^2 +  \theta^\top \mathcal{V}_t\theta,
\end{align}
where $\cM_t$ and $\cV_t$ are defined in~\eqref{eq: t_tilted_mean_statistic} and~\eqref{eq: t_tilted_variance_statistic}, respectively.
Notice that 
\begin{align}
       & \nabla_\theta \left\{ s_{\mathbf{w}(t; \breve{\theta}(t))}^2(\mathbf{f}^+(\theta), \mathbf{1}_N)\right\} \\ 
       & = \nabla_\theta \left\{ \frac{\left(\widehat{E}_t f(x_i; \theta)-\wF(-\infty) \right)^2}{\widehat{E}_t (f(x_i; \theta)-\wF(-\infty))^2}\right\}\\
       & = \nabla_\theta \left\{ \frac{(A(\theta) - \theta^\top \mathcal{M}_t -\wF(-\infty))^2}{(A(\theta) - \theta^\top \mathcal{M}_t-\wF(-\infty))^2 +  \theta^\top \mathcal{V}_t\theta} \right\}\\
       & = \frac{2(A(\theta) - \theta^\top \mathcal{M}_t-\wF(-\infty)) (\nabla_{\theta} A(\theta) - \mathcal{M}_t) \theta^\top \mathcal{V}_t\theta - 2(A(\theta) - \theta^\top \mathcal{M}_t-\wF(-\infty))^2\mathcal{V}_t\theta }
       {\left((A(\theta) - \theta^\top \mathcal{M}_t-\wF(-\infty))^2 +  \theta^\top \mathcal{V}_t\theta \right)^2}\\
       & = \frac{2(A(\theta) - \theta^\top \mathcal{M}_t -\wF(-\infty)) \left(  \theta^\top(\nabla_{\theta} A(\theta) - \mathcal{M}_t) - A(\theta)+ \theta^\top \mathcal{M}_t + \wF(-\infty)\right) \mathcal{V}_t\theta }
       {\left((A(\theta) - \theta^\top \mathcal{M}_t-\wF(-\infty))^2 +  \theta^\top \mathcal{V}_t\theta \right)^2}\\
       & = \frac{2(A(\theta) - \theta^\top \mathcal{M}_t-\wF(-\infty)) \left(  \theta^\top\nabla_{\theta} A(\theta)  - A(\theta) +\wF(-\infty)\right) \mathcal{V}_t\theta }
       {\left((A(\theta) - \theta^\top \mathcal{M}_t-\wF(-\infty))^2 +  \theta^\top \mathcal{V}_t\theta \right)^2}.
\end{align}
Hence, 
\begin{align}
     &\frac{\partial}{\partial \tau} \left\{s_{\mathbf{w}(t; \breve{\theta}(t))}^2(\mathbf{f}^+(\breve{\theta}(\tau)), \mathbf{1}_N)\right\} \\ & = \left( \frac{\partial}{\partial \tau} \breve{\theta}(\tau) 
 \right)^\top \nabla_\theta \left\{s_{\mathbf{w}(t; \breve{\theta}(t))}^2(\mathbf{f}^+(\breve{\theta}(\tau)), \mathbf{1}_N) \right\}\\
 & = -\breve{\theta}^\top(\tau)
 \mathcal{V}(\tau; \breve{\theta}(\tau)) 
 \left(  \nabla^2_{\theta\theta} A(\breve{\theta}(\tau)) + \tau \mathcal{V}(\tau; \breve{\theta}(\tau))  \right)^{-1}  \nonumber\\
& \quad \quad \times -\frac{2(A(\breve{\theta}(\tau)) - \breve{\theta}(\tau)^\top \mathcal{M}_t-\wF(-\infty)) (A(\breve{\theta}(\tau)) - \breve{\theta}(\tau)^\top \mathcal{M}(\tau;\breve{\theta}(\tau))-\wF(-\infty))  }
       {\left((A(\breve{\theta}(\tau)) - \breve{\theta}(\tau)^\top \mathcal{M}_t-\wF(-\infty))^2 +  \breve{\theta}(\tau)^\top \mathcal{V}_t\theta \right)^2} \mathcal{V}_t \breve{\theta}(\tau).\label{eq:184}
\end{align}
Note that $2(A(\breve{\theta}(\tau)) - \breve{\theta}(\tau)^\top \mathcal{M}_t-\wF(-\infty)) (A(\breve{\theta}(\tau)) - \breve{\theta}(\tau)^\top \mathcal{M}(\tau;\breve{\theta}(\tau))-\wF(-\infty)) > 0$ by definition, and $\mathcal{V}_\tau \left(  \nabla^2_{\theta\theta} A(\breve{\theta}(\tau)) + \tau \mathcal{V}_\tau  \right)^{-1} \mathcal{V}_\tau$ is a symmetric positive semi-definite matrix. Therefore,
The proof is completed following that the quantity in Eq.~\eqref{eq:184} is non-negative for $t = \tau.$ 
\end{proof}}

\newpage
\section{Connections to Other Objectives (Proofs and Additional Results)}
\label{app:superquantile}

\begin{figure}[h!]
    \centering
    \begin{subfigure}{0.49\textwidth}
        \centering
        \includegraphics[width=\textwidth]{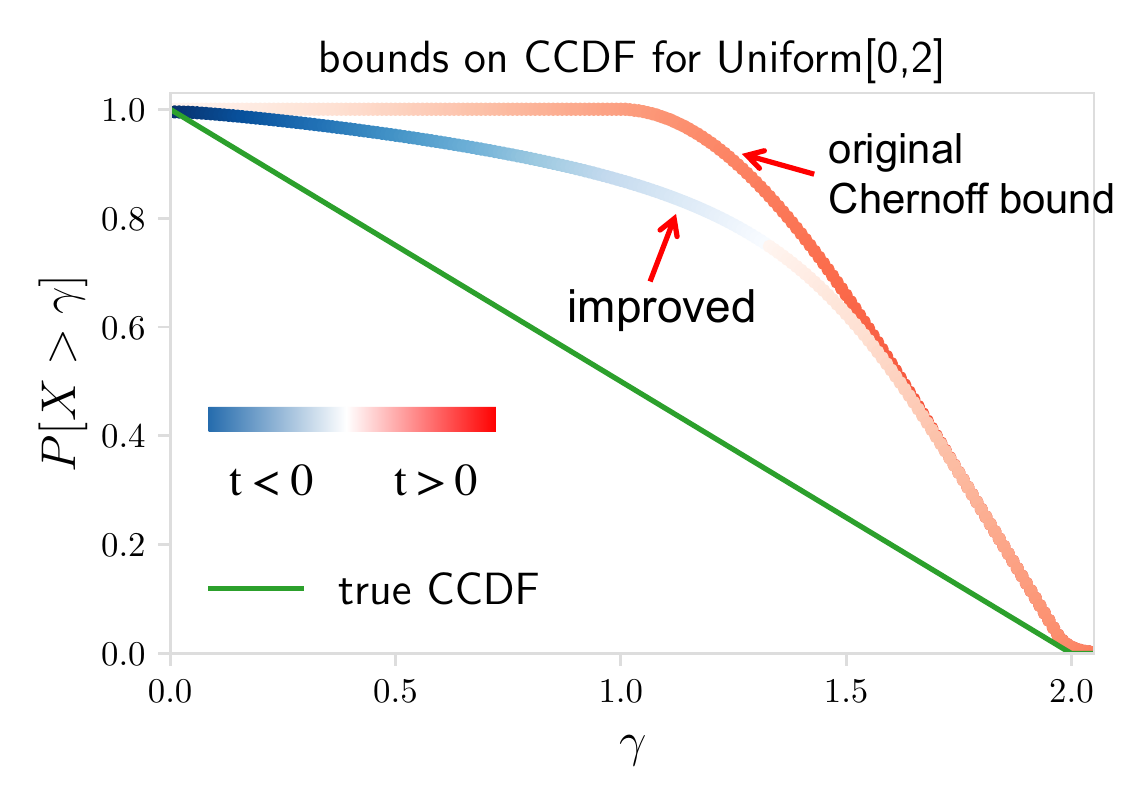}
    \end{subfigure}
    \hfill
    \begin{subfigure}{0.49\textwidth}
        \centering
        \includegraphics[width=\textwidth]{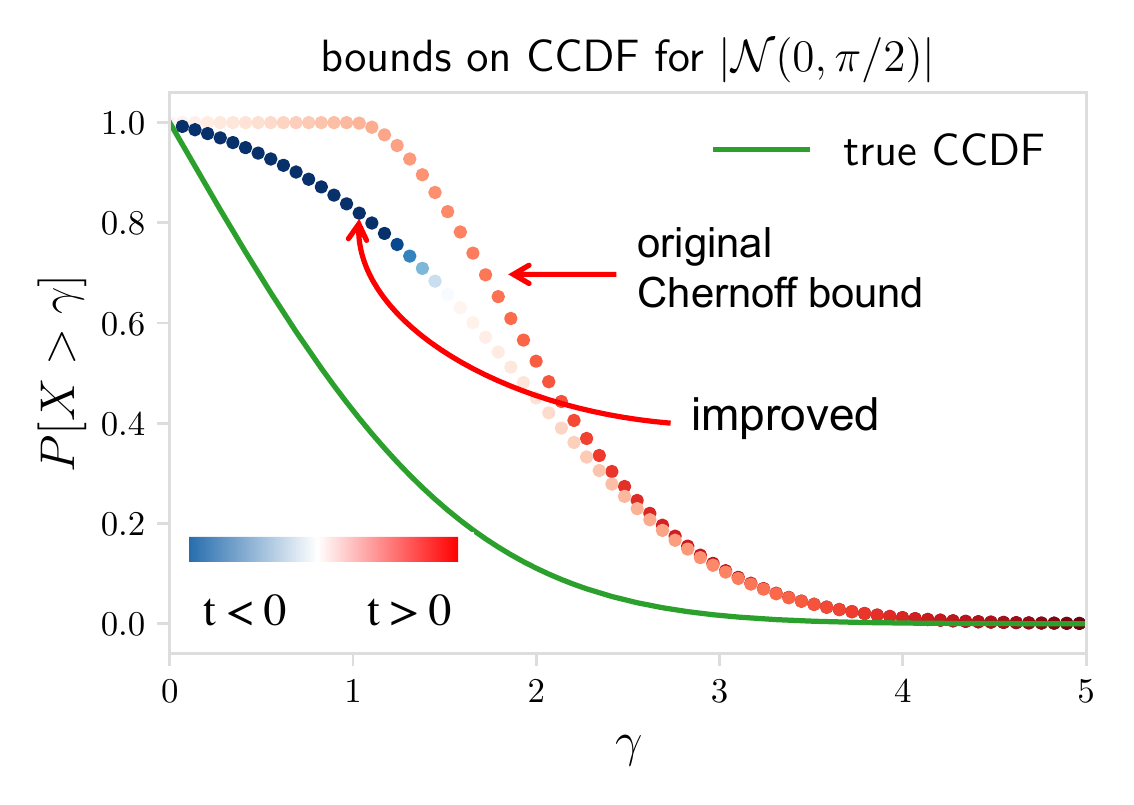}
    \end{subfigure}
    \caption{{Comparing the new Chernoff bound on complementary CDF (CCDF) (i.e., $P[X \geq \gamma]$) proposed in Theorem~\ref{thm: new-chernoff} (denoted as `improved') with the original Chernoff bound in two cases: $X \sim \text{Uniform}[0,2]$ and $X \sim \left|\mathcal{N}(0, \pi/2)\right|$. We see that by sweeping $t$ from all real numbers, our bound is significantly tighter than the generic Chernoff bound which optimizes over $t \in \mathbb{R}^{+}$, especially in the small deviations regime.}}
\label{fig:new_chernoff_bound}
\end{figure}

\begin{lemma}
If $a<\wF(- \infty)$ then $\wQ^0(\gamma) = 1$. Further, if $\gamma> \wF(+\infty)$ then $\wQ^0(\gamma) = 0$,
where $\wF(\cdot)$ is defined in Definition~\ref{def:wF}, and is reproduced here:
\begin{align}
    \wF(-\infty) &= \lim_{t \to -\infty} \wR(t; \breve{\theta}(t)) = \min_{\theta} \min_{i \in [N]} f(x_i; \theta),\\ 
    \wF(+\infty) &= \lim_{t \to +\infty} \wR(t; \breve{\theta}(t)) = \min_{\theta} \max_{i \in [N]} f(x_i; \theta).
\end{align}
\end{lemma}

Next, we present our main result on the connection between {tail distribution of losses} and TERM, using Theorem~\ref{thm: new-chernoff}.

\begin{theorem}
For all $t \in \mathbb{R},$ and all $\theta$, and all $\gamma \in (\wF(-\infty), \wF(+\infty)),$\footnote{We define the RHS at $t=0$ via continuous extension.}
\begin{equation}
    \wQ(\gamma; \theta) \leq \overline{Q}(\gamma; t, \theta) := \frac{e^{ \wR(t; \theta) t} -  e^{\wF(-\infty) t}}{e^{\gamma t} -  e^{\wF(-\infty) t}}.
\label{eq:Q-thm}
\end{equation}
\label{thm:Q}
\end{theorem}
\begin{proof}
{The proof is a direct application of  Theorem~\ref{thm: new-chernoff} to the non-negative random variable $(f(X; \theta) - \wF(-\infty) )$, where $X$ is distributed according to the empirical distribution. 
}
\end{proof}

Recall that optimizing $\widetilde{\text{VaR}}$ is equivalent to optimizing $\wQ$. Next we show how TERM is related to optimizing $\wQ$. Recall that $\wQ^0(\gamma)$ denotes the optimal value of $\wQ(\gamma; \theta)$ optimized over $\theta.$ Let
\begin{equation}
     \wQ^1(\gamma) := \inf_{t\in \mathbb{R}} \left\{ \wQ(\gamma; \breve{\theta}(t))\right\},
\end{equation}
which denotes the value at risk optimized over the $t$-tilted solutions.
\begin{theorem}
For all $\gamma \in (\wF(-\infty), \wF(+\infty)),$ we have 
\begin{equation}
     \wQ^0(\gamma) \leq \wQ^1(\gamma) \leq \wQ^2(\gamma) \leq \wQ^3(\gamma) = \inf_{t \in \mathbb{R}} \left\{ \overline{Q}(\gamma, t)  \right\},
\label{eq:Q-2}
\end{equation}
where
\begin{align}
     \overline{Q}(\gamma, t)& :=  \frac{e^{\wF(t) t} - e^{\wF(-\infty) t} }{e^{\gamma t} - e^{\wF(-\infty) t}},\\
     \wt^3(\gamma) &:= \arg\inf_{t \in \mathbb{R}} \left\{ \overline{Q}(\gamma, t) \right\}, \\
    \wQ^2(\gamma) &:= \wQ(\gamma; \breve{\theta}(\wt^3(\gamma))),\\
    \wQ^3(\gamma) &:= \overline{Q}(\gamma, \wt^3(\gamma)) .
\end{align}
\label{thm:Q-opt}
\end{theorem}
\begin{proof}
The only non-trivial step is to show that $Q^2(\gamma) \leq Q^3(\gamma).$
Following Theorem~\ref{thm:Q},
\begin{align}
   Q^2(\gamma) &= \wQ(\gamma; \breve{\theta}(\wt(\gamma))
   \leq \inf_{t \in \mathbb{R}} \overline{Q}(\gamma; t, \breve{\theta}(t)) = Q^3(\gamma),
\end{align}
which completes the proof.
\end{proof}
 Theorem~\ref{thm:Q-opt} motivates us with the following approximation on the solutions of the {minimizing the tail distribution of losses (Definition~\ref{def: superquantile-losses})}.
\begin{approximation}
For all $\gamma \in (\wF(-\infty), \wF(+\infty)),$
\begin{equation}
   \wQ(\gamma; \theta^0(\gamma)) = \wQ^0(\gamma) \approx \wQ^2(\gamma) =  \wQ(\gamma; \breve{\theta}(\wt(\gamma))),
\label{eq:approx-superquantile-sol}
\end{equation}
and hence, $\breve{\theta}(\wt(\gamma))$ is an approximate solution to the tail probability optimization problem.
\label{approx:superquantile}
\end{approximation}

\begin{figure}[t!]
    \centering
    \begin{subfigure}{0.49\textwidth}
        \centering
        \includegraphics[width=\textwidth]{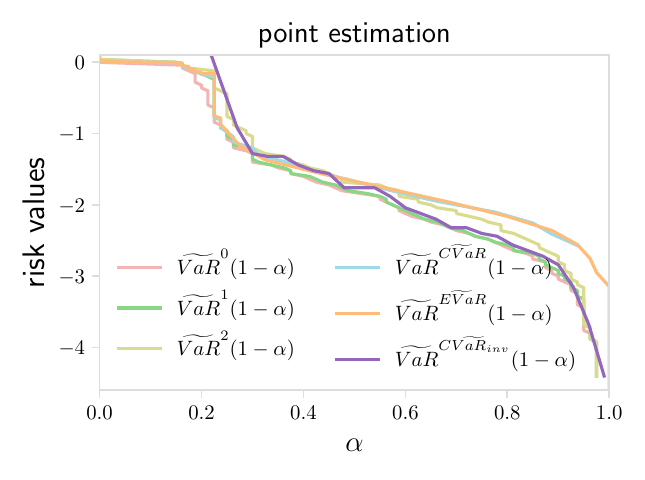}
    \end{subfigure}
    \hfill
    \begin{subfigure}{0.49\textwidth}
        \centering
        \includegraphics[width=\textwidth]{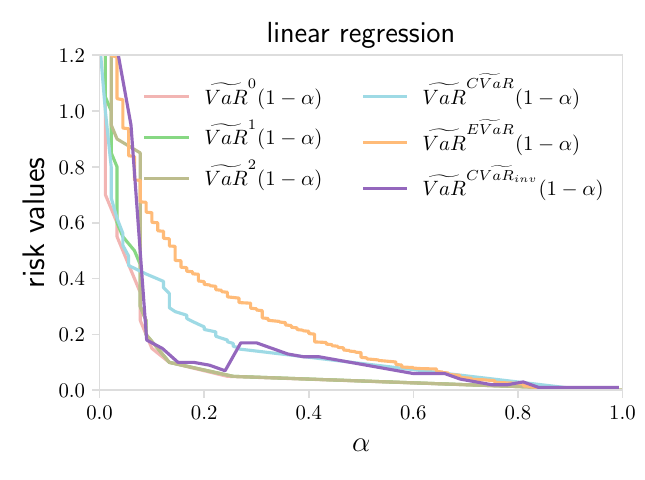}
    \end{subfigure}
    \caption{Comparing the solutions of different risks in terms of how well they solve VaR. For $i \in \{0, 1, 2\}$, $\widetilde{\text{VaR}}^i (1-\alpha) := \min_{\gamma} \{\gamma | \wQ^i (\gamma) \leq \alpha \}$. $\widetilde{\text{VaR}}^0 (1-\alpha) := \min_{\gamma} \{\gamma | \wQ^0 (\gamma) \leq \alpha \}$ is the optimal $\widetilde{\text{VaR}}(1-\alpha;\theta)$. By definition, $\widetilde{\text{VaR}}^2 (1-\alpha)$ is the risk value of  $\widetilde{\text{VaR}}(1-\alpha;\theta)$ with $\theta$ being the solutions of $\widetilde{\text{TiVaR}} (1-\alpha;\theta)$. $\widetilde{\text{VaR}}^{\widetilde{\text{CVaR}}}(1-\alpha)$ denotes the value of $\widetilde{\text{VaR}}(1-\alpha; \theta)$ evaluated at  $\arg\min_{\theta} \widetilde{\text{CVaR}}(1-\alpha;\theta)$, and \new{$\widetilde{\text{VaR}}^{\widetilde{\text{CVaR}}_{\text{inv}}}(1-\alpha)$  and $\widetilde{\text{VaR}}^{\widetilde{\text{EVaR}}}(1-\alpha)$ are defined in the similar way}. We see that $\widetilde{\text{VaR}}^1(1-\alpha)$ and $\widetilde{\text{VaR}}^2(1-\alpha)$ are close to $\widetilde{\text{VaR}}^0(1-\alpha)$, which indicates VaR with the solutions obtained from solving $\widetilde{\text{TiVaR}}(1-\alpha;\theta)$ (which is $\widetilde{\text{VaR}}^2(1-\alpha)$) is a tight upper bound of the globally optimal $\widetilde{\text{VaR}}(1-\alpha;\theta)$. $\widetilde{\text{VaR}}^2(1-\alpha)$ is also tighter than VaR under EVaR solutions when $\alpha$ is not small.}
\label{fig:super-quantile-1}
\end{figure}

While we have not characterized how tight this approximation is for  $\gamma \in (\wF(-\infty), \wF(+\infty))$, we believe that Approximation~\ref{approx:superquantile} provides a reasonable solution to the {tail distribution} optimization problem in general. This is evidenced empirically when the approximation is evaluated on the toy examples of Figure~\ref{fig:example}, and compared with the global solutions of the {tail distribution optimization} method, as shown in Figure~\ref{fig:super-quantile-1}. As can be seen, $\wQ^0(\gamma) \approx \wQ^2(\gamma)$ as suggested by Approximation~\ref{approx:superquantile}. Also, we can see that while the bound in Theorem~\ref{thm:Q-opt} ($\widetilde{Q}^3(\gamma)$) is not tight, the solution that is obtained from solving it ($\widetilde{Q}^2(\gamma)$) results in a good approximation to the {tail distribution minimization} ($\widetilde{Q}^0(\gamma)$).

\paragraph{Inverse CVaR.}  \new{We note that while the most popular form of CVaR focuses on upper quantiles (as discussed in the main text), one may explore `inverse' CVaR that can focus on lower quantiles, with its empirical form  $\widetilde{\text{CVaR}}_{\text{inv}}(1-\alpha; \theta)$ for $\alpha \in [0,1)$  defined as 
\begin{align}
    \widetilde{\text{CVaR}}_{\text{inv}}(1-\alpha; \theta) := -\min_{\gamma} \left\{\gamma + \frac{1}{1-\alpha}\frac{1}{N} \sum_{i \in [N]} [-f(x_i;\theta)-\gamma]_+\right\}.
\end{align}
As $\alpha$ ranges from 0 to 1, optimizing $\widetilde{\text{CVaR}}_{\text{inv}}(1-\alpha; \theta)$ transitions from solving avg-loss to min-loss. However, different from TiVaR or CVaR, $\text{CVaR}_{\text{inv}}$ is not a valid upper bound of VaR. Despite this, we optimize $\min_{\theta} \widetilde{\text{CVaR}}_{\text{inv}}(1-\alpha; \theta)$, plug in the optimal model parameters to evaluate VaR values, and compare with the approximate VaR values under the solutions of other risks including TiVaR.  From Figure~\ref{fig:super-quantile-1}, we see that VaR values under TiVaR solutions can be smaller than those under $\text{CVaR}_{\text{inv}}$ solutions on linear regression. Given any $\alpha$, our proposed TiVaR objective approximates VaR, ranging from min-loss to max-loss smoothly \textit{in a single formulation}, which can be more desirable than optimizing two objectives.}

\paragraph{Proof of Theorem~\ref{thm:var-TiVaR-cvar}.}
We first prove $\widetilde{\text{TiVaR}}(1-\alpha;\theta) \leq \widetilde{\text{EVaR}}(1-\alpha;\theta)$. 
\begin{align}
   \widetilde{\text{EVaR}}(1-\alpha;\theta) - \wF(-\infty) &=  \min_{t \in \mathbb{R}^{>0}} \frac{1}{t}\log\left(\frac{\frac{1}{N}\sum_{i\in [N]} e^{tf(x_i;\theta)}}{\alpha}\right) - \wF(-\infty) \\
   &=  \min_{t \in \mathbb{R}^{>0}} \frac{1}{t}\log\left(\frac{ e^{(\wR(t;\theta)-\wF(-\infty)) t}}{\alpha}\right) \\
   & \geq \min_{t \in \mathbb{R}^{>0}} \frac{1}{t}\log\left[\frac{ e^{(\wR(t;\theta)-\wF(-\infty)) t}-(1-
   \alpha)}{\alpha}\right]_+ \\
   &\geq \min_{t \in \mathbb{R}} \frac{1}{t}\log\left[\frac{ e^{(\wR(t;\theta)-\wF(-\infty)) t}-(1-
   \alpha)}{\alpha}\right]_+.
\end{align}
We next prove $\widetilde{\text{VaR}}(1-\alpha;\theta) \leq \widetilde{\text{TiVaR}}(1-\alpha;\theta)$. From Theorem~\ref{thm:Q}, we know that for any $t, \theta$,
\begin{align}
    \wQ(\gamma;\theta) \leq \min_{t\in \mathbb{R}} \left\{\frac{e^{\wR(t;\theta) t} - e^{-\wF(-\infty) t}}{e^{\gamma t} - e^{\wF(-\infty)t}}\right\}
\end{align}
Let $\wQ(\gamma;\theta) = \alpha$, and $\gamma^*=\widetilde{\text{VaR}}(1-\alpha;\theta)$. We have $\min_{t\in \mathbb{R}} \left\{\frac{e^{\wR(t;\theta) t} - e^{-\wF(-\infty) t}}{e^{\gamma^* t} - e^{\wF(-\infty)t}}\right\} \geq \alpha.$  We also note
\begin{align}
    \min_{t\in \mathbb{R}}\left\{ \frac{e^{\wR(t;\theta) t} - e^{-\wF(-\infty) t}}{e^{ \widetilde{\text{TiVaR}}(1-\alpha;\theta) t} - e^{\wF(-\infty)t}}\right\} = \alpha.
\end{align}
Hence,
\begin{align}
    \widetilde{\text{TiVaR}}(1-\alpha;\theta) \geq \gamma^* = \widetilde{\text{VaR}}(1-\alpha;\theta).
\end{align}

\paragraph{TERM and Entropic Value-at-Risk.}

Let  $\breve{\theta}_X(t)$ be the minimizer of entropic risk $R_X(t;\theta)$:
\begin{equation}
    \breve{\theta}_X(t):=\arg\min_{\theta \in \Theta} R_X(t;\theta). 
\end{equation}
Further, let $F_X(t)$ be the optimum value of entropic risk, i.e.,
\begin{equation}
    F_X(t) := R_X(t; \breve{\theta}_X(t)).
\label{eq: F_X}
\end{equation}
Our next result will relate EVaR to entropic risk.
\begin{lemma}[Relations between entropic risk and EVaR]
\label{lem:ER-EVaR relation}
Assume that for $t \in \mathbb{R}^{>0},$ $F_X(t)$ is a strongly convex function of $\frac{1}{t}$ . Further, let
\begin{align}
    \breve{t}_X(\alpha) \in \arg\min_{t \in \mathbb{R}^{>0}} \left\{F_X(t)-\frac{1}{t}\log \alpha\right\}, 
\end{align}
then
\begin{align}
     \arg\min_{\theta} \left\{\text{EVaR}_{X}(1-\alpha; \theta)\right\} &= \arg\min_{\theta} \left\{R_{X}(\breve{t}_X(\alpha);\theta)\right\} := \breve{\theta}_X(\breve{t}_X(\alpha)), \\
      R_X\left(\breve{t}_X(\alpha); \breve{\theta}_X(\breve{t}_X(\alpha))\right) &= F_{X}(\breve{t}_X(\alpha)) \leq \text{EVaR}_{X}\left(1-\alpha; \breve{\theta}_X(\breve{t}_X(\alpha))\right).
\end{align}
\end{lemma}
\begin{proof} 
Consider the any minimizer of $R_X(\breve{t}_X(\alpha), \theta)$, i.e.,
\begin{align}
    \breve{\theta}_X(\breve{t}_X(\alpha)) &\in \arg\min_{\theta} R_X(\breve{t}_X (\alpha);\theta),
\end{align}
we next prove 
\begin{align}
    \breve{\theta}_X (\breve{t}_X (\alpha)) \in \arg\min_{\theta}\left(\min_{t>0}\left(\frac{1}{t}\log \mathbb{E}[e^{tf(X;
    \theta)}]-\frac{1}{t}\log \alpha \right)\right).
\end{align}
Denote $\min_{t>0}\left(\frac{1}{t}\log \mathbb{E}[e^{tf(X;
    \theta)}]-\frac{1}{t}\log \alpha \right)$ as $h(\theta;\alpha)$. Let
\begin{align}
    \theta_{v}^* &\in \arg\min_{\theta}h(\theta;\alpha),\\
    {t}_v(\theta_v^*) &\in \arg\min_{t>0} \left(\frac{1}{t}\log \mathbb{E}[e^{tf(X;
    \theta_v^*)}]-\frac{1}{t}\log \alpha\right).
\end{align}
By the definition of $\breve{t}_X (\alpha)$ and $\breve{\theta}_X(t)$, we have
\begin{align}
    &\frac{1}{\breve{t}_X (\alpha)} \log \mathbb{E}[e^{\breve{t}_X (\alpha) f(X;\breve{\theta}_X(\breve{t}_X (\alpha)))}] - \frac{1}{\breve{t}_X (\alpha)}\log \alpha \\ &\leq  \frac{1}{{t}_v(\theta_v^*)} \log \mathbb{E}[e^{{t}_v(\theta_v^*) f(X;\breve{\theta}_X ({t}_v(\theta_v^*)))}] - \frac{1}{{t}_v(\theta_v^*)}\log \alpha \\
    &\leq \frac{1}{{t}_v(\theta_v^*)} \log \mathbb{E}[e^{{t}_v(\theta_v^*) f(X; \theta_v^*)}] - \frac{1}{{t}_v(\theta_v^*)}\log \alpha. \label{eq:tmp1}
\end{align}
By the definition of $\theta_v^*$, $h(\theta_v^*;\alpha) \leq h(\breve{\theta}_X (\breve{t}_X (\alpha));\alpha)$, i.e.,
\begin{align}
\min_{t>0} \left(\frac{1}{t}\log \mathbb{E}[e^{tf(X;
    \theta_v^*)}]-\frac{1}{t}\log \alpha\right) &\leq \min_{t>0} \left(\frac{1}{t}\log \mathbb{E}[e^{tf(X;
    \breve{\theta}_X (\breve{t}_X (\alpha)))}]-\frac{1}{t}\log \alpha\right).
\end{align}
We have
\begin{align}
    &\frac{1}{{t}_v(\theta_v^*)} \log \mathbb{E}[e^{{t}_v(\theta_v^*) f(X; \theta_v^*)}] - \frac{1}{{t}_v(\theta_v^*)}\log \alpha \\ &\leq \min_{t>0} \left(\frac{1}{t}\log \mathbb{E}[e^{tf(X;
    \breve{\theta}_X (\breve{t}_X(\alpha)))}]-\frac{1}{t}\log \alpha\right) \\
    &\leq \frac{1}{\breve{t}_X (\alpha)} \log \mathbb{E}[e^{\breve{t}_X(\alpha) f(X;\breve{\theta}_X(\breve{t}_X(\alpha)))}] - \frac{1}{\breve{t}_X (\alpha)}\log \alpha, 
\end{align}
Hence, $\breve{\theta}_X(\breve{t}_X(\alpha)) \in \arg\min_{\theta}\left(\min_{t>0}\left(\frac{1}{t}\log \mathbb{E}[e^{tf(X;
    \theta)}]-\frac{1}{t}\log \alpha \right)\right)$.
    
For the other direction, consider any minimizer of $\text{EVaR}_X(1-\alpha;\theta)$, i.e.,
\begin{align}
    \theta_v^* &\in \arg\min_{\theta} h(\theta;\alpha) \\
    {t}_v(\theta_v^*) &\in \arg\min_{t>0} \left(\frac{1}{t}\log \mathbb{E}[e^{f(X;\theta_v^*)}]-\frac{1}{t}\log \alpha \right).
\end{align}
We next prove $\theta_v^* \in \arg\min_{\theta} \frac{1}{\breve{t}_X (\alpha)}\log \mathbb{E}[e^{\breve{t}_X (\alpha) f(X;\theta)}]$. 
By the definition of $\breve{\theta}_X(t)$ and $\breve{t}_X(\alpha)$,
\begin{align}
    &\frac{1}{{t}_v(\theta_v^*)} \log \mathbb{E}[e^{{t}_v(\theta_v^*) f(X; \theta_v^*)}] - \frac{1}{{t}_v(\theta_v^*)}\log \alpha \\ &\geq \frac{1}{{t}_v(\theta_v^*)} \log \mathbb{E}[e^{{t}_v(\theta_v^*) f(X; {\breve{\theta}}_X({t}_v(\theta_v^*))}] - \frac{1}{{t}_v(\theta_v^*)}\log \alpha \\
    &\geq \frac{1}{\breve{t}_X(\alpha)} \log \mathbb{E}[e^{\breve{t}_X(\alpha) f(X; \breve{\theta}_X(\breve{t}_X(\alpha))}] - \frac{1}{\breve{t}_X(\alpha)}\log \alpha.
\end{align}
On the other hand, by the definition of $\theta_v^*$ and ${t}_v(\theta_v^*)$,
\begin{align}
    \frac{1}{{t}_v(\theta_v^*)} \log \mathbb{E}[e^{{t}_v(\theta_v^*) f(X; \theta_v^*)}] - \frac{1}{{t}_v(\theta_v^*)}\log \alpha  &\leq \min_{t>0} \left(\frac{1}{t}\log \mathbb{E}[e^{tf(X;
    \breve{\theta}_X(\breve{t}_X(\alpha)))}]-\frac{1}{t}\log \alpha\right) \\
    &\leq \frac{1}{\breve{t}_X(\alpha)} \log \mathbb{E}[e^{\breve{t}_X(\alpha) f(X;\breve{\theta}_X(\breve{t}_X(\alpha)))}] - \frac{1}{\breve{t}_X(\alpha)}\log \alpha.
\end{align}
Therefore,
\begin{align}
     &\frac{1}{{t}_v(\theta_v^*)} \log \mathbb{E}[e^{{t}_v(\theta_v^*) f(X; \theta_v^*)}] - \frac{1}{{t}_v(\theta_v^*)}\log \alpha \\&= \frac{1}{{t}_v(\theta_v^*)} \log \mathbb{E}[e^{t_v(\theta_v^*) f(X; \breve{\theta}_X({t}_v(\theta_v^*))}] - \frac{1}{{t}_v(\theta_v^*)}\log \alpha \\&= \frac{1}{\breve{t}_X(\alpha)} \log \mathbb{E}[e^{\breve{t}_X(\alpha) f(X;\breve{\theta}_X(\breve{t}_X(\alpha)))}] - \frac{1}{\breve{t}_X(\alpha)}\log \alpha.
\end{align}
{If 
\begin{align}
    s(t) := \frac{1}{t}\log \mathbb{E}[e^{tf(X;\breve{\theta}_X(t))}] - \frac{1}{t}\log \alpha
\end{align}
has a unique minimizer, 
\begin{align}
    \breve{t}_X(\alpha) = {t}_v(\theta_v^*),
\end{align}}
and
\begin{align}
    \theta_v^* \in \arg\min_{\theta} \frac{1}{\breve{t}_X(\alpha)} \log \mathbb{E}[e^{\breve{t}_X (\alpha) f(X;\theta)}].
\end{align}
Hence, we have proved
\begin{align}
    \arg\min_{\theta} \text{EVaR}_{X}(1-\alpha; \theta) = \arg\min_{\theta} R_{X}(\breve{t}_X(\alpha);\theta),
\end{align}
and
\begin{align}
   R_{X}(\breve{t}_X(\alpha);\breve{\theta}_X(\breve{t}_X(\alpha))) \leq \text{EVaR}_{X}(1-\alpha; \breve{\theta}_X(\breve{t}_X(\alpha))). 
\end{align}
\end{proof}
The lemma relates the solution and the optimal value of EVaR with those of entropic risk. We can extend Lemma~\ref{lem:ER-EVaR relation} to the empirical version below.

\begin{lemma}[Relations between empirical entropic risk and empirical EVaR]
\label{lem:ER-EVaR relation empirical}
Assume that  $\wF(t)$ is a strongly convex function of $\frac{1}{t}.$ For $\alpha \in \{\frac{k}{N}\}_{k \in [N]}$, let 
\begin{align}
    \breve{t}(\alpha) \in \arg\min_{t>0} \left\{\wF(t)-\frac{1}{t}\log \alpha\right\}, 
\end{align}
then
\begin{align}
     \arg\min_{\theta} \widetilde{\text{EVaR}}(1-\alpha; \theta) &= \arg\min_{\theta} \wR(\breve{t}(\alpha);\theta), \\
      \wF(\breve{t}(\alpha)) &\leq \widetilde{\text{EVaR}}(1-\alpha; \breve{\theta}(\breve{t}(\alpha))).
\end{align}
\end{lemma}

\newpage
\section{Solving TERM (Proofs and Details)} \label{app:solving}

\subsection{Hierarchical Multi-Objective Tilting} \label{app:hierarchical-TERM} 
We state the hierarchical multi-objective tilting for a hierarchy of depth $3.$ While we don't directly use this form, it is stated to clarify the experiments in Section~\ref{sec:experiments} where tilting is done at class level and annotator level, and the sample-level tilt value could be understood to be $0.$ 
\begin{align}
\wJ(m, t, \tau; \theta) &:=  \frac{1}{m} \log \left(\frac{1}{N} \sum_{G \in [GG]} \left(\sum_{g \in [G]} |g|\right) e^{m \widetilde{J}_{G} (\tau; \theta)}\right)\\
\wJ_{G}(t, \tau; \theta) &:=  \frac{1}{t} \log \left(\frac{1}{\sum_{g \in [G]} |g| } \sum_{g \in [G]} |g| e^{t \widetilde{R}_g (\tau; \theta)}\right)\\
\wR_g (\tau; \theta) &:= \frac{1}{\tau} \log \left(\frac{1}{|g|} \sum_{x \in g}  e^{\tau f(x; \theta)} \right),
\label{eq: tri-TERM}
\end{align}

\paragraph{Proof of Lemma~\ref{lemma: generalized-tilt-gradient}.}
We proceed as follows. First notice that by invoking Lemma~\ref{lemma:TERM-gradient},
\begin{align}
\label{eq:133}
    \nabla_\theta \wJ(t, \tau ;\theta) &= \sum_{g \in [G]} w_g(t, \tau; \theta) \nabla_\theta \wR_g(\tau; \theta)
\end{align}
where
\begin{equation}
    w_g (t, \tau; \theta) := \frac{|g|e^{t \wR_g(\tau; \theta)}}{\sum_{g' \in [G]} |g'|e^{t \wR_{g'}(\tau; \theta)}}.
\end{equation}
where $\wR_g (\tau; \theta)$ is defined  in \eqref{eq: class-TERM}, and is reproduced here:
\begin{equation}
    \wR_g (\tau; \theta) := \frac{1}{\tau} \log \left(\frac{1}{|g|} \sum_{x \in g}  e^{\tau f(x; \theta)} \right).
\end{equation}
On the other hand, by invoking Lemma~\ref{lemma:TERM-gradient}, 
\begin{equation}
\label{eq:136}
  \nabla_\theta \wR_g(\tau; \theta) =   \sum_{x \in g} w_{g,x}(\tau; \theta) \nabla_\theta f(x;\theta)
\end{equation}
where
\begin{equation}
    w_{g, x} (\tau; \theta) := \frac{e^{\tau f(x; \theta)}}{\sum_{y \in g} e^{\tau f(y; \theta)}}.
\end{equation}
Hence, combining~\eqref{eq:133} and~\eqref{eq:136},
\begin{align}
    \nabla_\theta \wJ(t, \tau ;\theta) &= 
 \sum_{g\in [G]} \sum_{x \in g} w_g(t, \tau; \theta)w_{g, x}(\tau; \theta)\nabla_\theta f(x;\theta).
\end{align}
The proof is completed by algebraic manipulations to show that
\begin{equation}
    w_{g, x}(t, \tau; \theta) =  w_g(t, \tau; \theta)w_{g, x}(\tau; \theta).
\end{equation}
\hfill \qedsymbol


\newpage
\subsection{Proofs of Convergence for TERM Solvers}
\label{app: solving-TERM}

\begin{algorithm}[h]
\SetKwInOut{Init}{Initialize}
\SetAlgoLined
\DontPrintSemicolon
\SetNoFillComment
\Init{$\theta, \doublewidetilde{R}_{t} = \frac{1}{t} \log\left(\frac{1}{N}\sum_{i \in [N]} e^{tf(x_i;\theta)}\right)$}
\KwIn{$t, \alpha, \lambda$}
\While{stopping criteria not reached}{
sample two independent minibatches $B_1, B_2$ uniformly at random from $[N]$\; compute the loss $f(x; \theta)$ and gradient $\nabla_\theta f(x; \theta)$ for all $x \in B_1$\;
$\wR_{B, t} \gets \text{$t$-tilted loss~\eqref{eq: TERM} on minibatch $B_2$}$\;
$\doublewidetilde{R}_{t} \gets \frac{1}{t} \log \left((1-\lambda) e^{t\doublewidetilde{R}_{t}} + \lambda e^{t \wR_{B, t}}\right)$\; $w_{t, x} \gets e^{t f(x; \theta) - t\doublewidetilde{R}_{t}}$\;
$\theta \gets \theta - \frac{\alpha}{|B_1|} \sum_{x\in B_1} w_{t, x} \nabla_{\theta} f(x; \theta)$\;
}
\caption{Stochastic Non-Hierarchical TERM with two mini-batches}\label{alg:stochastic-non-h-TERM-two-batch}
\end{algorithm}



To prove our convergence results in Theorem~\ref{thm: convergence_stochastic_convex}, we first prove a lemma below.
\begin{lemma}
Denote $k_t :=\arg\max_k \left(k < \frac{2e}{\mu}+\frac{etLB  e^{t(\widetilde{F}_{\max}-\widetilde{F}_{\min})}}{\mu k}\right)$. Let $\lambda = 1-\frac{1}{2e}$, and
\begin{align}\label{eq: learning_rate}
\alpha_k = 
    \begin{cases}
    \frac{1}{tLB e^{t(\widetilde{F}_{\max}-\widetilde{F}_{\min})}+1},& \text{if } k \leq k_t\\
    \frac{2e}{\mu k},              & \text{otherwise},
\end{cases}
\end{align} 
then for any $k$,
\begin{align}
    \mathbb{E}[e^{t(\doublewidetilde{R}_k - \wR_k)} | \theta_1, \ldots, \theta_k]  \leq 2e,
\end{align}
where $\wR_k := \wR(t;\theta_k) = \frac{1}{t} \log \left(\frac{1}{N} \sum_{i \in [N]} e^{tf(x_i; \theta_k)}\right)$.
\end{lemma}
\begin{proof}
We have the updating rule
\begin{align}
    e^{t\doublewidetilde{R}_{k+1}} &= \lambda e^{tf(\xi_k, \theta_k)} + (1-\lambda) e^{t\doublewidetilde{R}_k} \label{eq:update_R}. 
\end{align}
Taking conditional expectation $\mathbb{E}[\cdot | \theta_1, \ldots, \theta_{k+1}]$ on both sides of~\eqref{eq:update_R} gives
\begin{align}
    &\mathbb{E}[e^{t(\doublewidetilde{R}_{k+1}-\wR_k)} | \theta_1, \ldots, \theta_{k+1}] \\ &= \lambda \mathbb{E}[e^{t(f(\xi_k; \theta_k)-\wR_k)} |  \theta_1, \ldots, \theta_{k+1}] + (1-\lambda) \mathbb{E}[e^{t(\doublewidetilde{R}_k-\wR_k)} |  \theta_1, \ldots, \theta_{k+1}]  \\
    &= \lambda +  (1-\lambda) \mathbb{E}[e^{t(\doublewidetilde{R}_k-\wR_k)} |  \theta_1, \ldots, \theta_k].
\end{align}
For any $k$, we have
\begin{align}
    \| \theta_{k+1}-\theta_k \| = \alpha_k \left\|\frac{e^{t \wR_k}}{e^{t\doublewidetilde{R}_k}} \nabla \wR_k \right\| \leq \alpha_k e^{t(\widetilde{F}_{\max}-\widetilde{F}_{\min})} B. 
\end{align}
Therefore,
 \begin{align}
    |f(x_i; \theta_{k-1})-f(x_i; \theta_{k})|\leq L\|\theta_{k-1}-\theta_k\| \leq \alpha_k LB e^{t(\widetilde{F}_{\max}-\widetilde{F}_{\min})},
\end{align}
and
\begin{align}
    e^{-t\alpha_k LB e^{t(\widetilde{F}_{\max}-\widetilde{F}_{\min})}} \leq e^{t(\wR_{k}-\wR_{k+1})} = \frac{\sum_{i \in [N]} e^{tf(x_i; \theta_{k})}}{\sum_{i \in [N]} e^{tf(x_i; \theta_{k+1})}} \leq e^{t\alpha_k LB e^{t(\widetilde{F}_{\max}-\widetilde{F}_{\min})}},
\end{align}
\begin{align}
   &e^{-t\alpha_{k}LB e^{t(\widetilde{F}_{\max}-\widetilde{F}_{\min})}}  \mathbb{E}[e^{t(\doublewidetilde{R}_{k+1}-\wR_{k+1})} | \theta_1, \ldots, \theta_{k+1}] \\&\leq  \mathbb{E}[e^{t(\doublewidetilde{R}_{k+1}-\wR_k)} |\theta_1, \ldots, \theta_{k+1}] \nonumber \\
   &\leq e^{t\alpha_{k}LB e^{t(\widetilde{F}_{\max}-\widetilde{F}_{\min})}}  \mathbb{E}[e^{t(\doublewidetilde{R}_{k+1}-\wR_{k+1})}|\theta_1, \ldots, \theta_{k+1}].
\end{align}
Hence,
\begin{align}
    &e^{-t\alpha_{k} LB e^{t(\widetilde{F}_{\max}-\widetilde{F}_{\min})}} \mathbb{E}[e^{t(\doublewidetilde{R}_{k+1}-\wR_{k+1})} | \theta_1, \ldots, \theta_{k+1}] \\&\leq  \lambda + (1-\lambda) \mathbb{E}[e^{t(\doublewidetilde{R}_k- \wR_k)} | \theta_1, \ldots, \theta_{k}] \\
    &\leq e^{t\alpha_{k} LB e^{t(\widetilde{F}_{\max}-\widetilde{F}_{\min})}} \mathbb{E}[e^{t(\doublewidetilde{R}_{k+1}-\wR_{k+1})} | \theta_1, \ldots, \theta_{k+1}].
\end{align}
(i) When $k \leq k_t$, {under the learning rate $\alpha_k$ set as in Eq.~\eqref{eq: learning_rate}}, we have
\begin{align}
    \alpha_k LB  e^{t(\widetilde{F}_{\max}-\widetilde{F}_{\min})} < 1.
\end{align}
Hence,
\begin{align}
    \mathbb{E}[e^{t(\doublewidetilde{R}_k - \wR_k)} | \theta_1, \ldots, \theta_k] &\leq e(\lambda + (1-\lambda)\mathbb{E}[e^{t(\doublewidetilde{R}_{k-1}-\wR_{k-1})}| \theta_1, \ldots, \theta_{k-1}]) \\ &\leq e + \frac{1}{2} \mathbb{E}[e^{t(\doublewidetilde{R}_{k-1}-\wR_{k-1})}| \theta_1, \ldots, \theta_{k-1}] \\
    &\leq \cdots \leq e\left(2-\frac{1}{2^{k-2}}\right) + \frac{1}{2^{k-1}}\mathbb{E}[e^{t(\doublewidetilde{R}_1-\wR_1)}|\theta_1] \leq 2e. 
\end{align}
(ii) When $k > k_t$, 
\begin{equation}
    \alpha_k = \frac{2e}{\mu k} < \frac{k}{k+tLB e^{t(\widetilde{F}_{\max}-\widetilde{F}_{\min})}}.
\end{equation}
Similarly, we have 
\begin{align}
\mathbb{E}[e^{t(\doublewidetilde{R}_k - \wR_k)} | \theta_1, \ldots, \theta_k]  &\leq e^{t \alpha_k LB  e^{t(\widetilde{F}_{\max}-\widetilde{F}_{\min})}} \left(\lambda + (1-\lambda)\mathbb{E}[e^{t(\doublewidetilde{R}_{k-1}-\wR_{k-1})}| \theta_1, \ldots, \theta_{k-1}]\right) \\
&\leq \cdots \leq 2e,
\end{align}
which completes the proof.
\end{proof}

\paragraph{Proof of Theorem~\ref{thm: convergence_stochastic_convex}} 
Denote the empirical optimal solution $\breve{\theta}(t)$ as $\theta^*$.
Denote the tilted stochastic gradient on data $\zeta_k$ as $g_k$, where
\begin{align}
    g_k =  \frac{e^{tf(\zeta_k; \theta_k)}}{e^{t\doublewidetilde{R}_{k}}} \nabla f(\zeta_k; \theta_k) = \frac{e^{t\wR_k}}{e^{t\doublewidetilde{R}_{k}}}\frac{e^{tf(\zeta_k; \theta_k)}}{e^{t\wR_k}} \nabla f(\zeta_k; \theta_k) = \frac{e^{t\wR_k}}{e^{t\doublewidetilde{R}_{k}}} \nabla \wR_k(\zeta_k).
\end{align} 
Therefore, for any $k \geq 1$,
\begin{align}
    \mathbb{E}[\langle\theta_k-\theta^*, g_k \rangle] &= \mathbb{E}[\mathbb{E}[\langle \theta_k-\theta^*, g_k \rangle |\theta_1, \ldots, \theta_k]] \\
    &= \mathbb{E}[\langle \theta_k-\theta^*, \mathbb{E}[g_k  |\theta_1, \ldots, \theta_k] \rangle] \\
    &= \mathbb{E}[\langle \theta_k-\theta^*, \mathbb{E}[e^{t(\wR_k-\doublewidetilde{R}_k)} | \theta_1, \ldots, \theta_k] \mathbb{E}[\nabla \wR_k(\zeta_k) | \theta_1, \ldots, \theta_k] \rangle] \label{eq:decouple}
    \\& \geq \frac{1}{2e}\mathbb{E}[\langle \theta_k-\theta^*, \nabla \wR(\theta_k)\rangle]\quad(\mathbb{E}[e^{t(\wR_k-\doublewidetilde{R}_k)}|\theta_1, \ldots, \theta_k] \geq 1/\mathbb{E}[e^{t(\doublewidetilde{R}_k-\wR_k)}|\theta_1, \ldots, \theta_k] )\\
    & \geq \frac{\mu}{2e} \mathbb{E}[\|\theta_k-\theta^*\|^2] \quad (\mu\text{-strong convexity of }\wR),
\end{align}
where \eqref{eq:decouple} follows from the fact that $e^{t(\wR_k-\doublewidetilde{R}_k)}$ and $\nabla \wR_k(\zeta_k)$ are independent given $\{\theta_1,\ldots,\theta_k\}$. 
\noindent
For $k \geq k_t$ with $\alpha_k = \frac{2e}{\mu k}$,
\begin{align}
   \mathbb{E}[\|\theta_{k+1}-\theta^*\|^2] &= \mathbb{E}[\|\theta_k - \alpha_k g_k - \theta^*\|^2] \\
    &= \mathbb{E}[\|\theta_k-\theta^*\|^2] - 2\alpha_k \mathbb{E}[\langle \theta_k-\theta^*, g_k\rangle] + \alpha_k^2\mathbb{E}[\|g_k\|^2] \\
    & \leq \left(1-\frac{\alpha_k \mu}{e}\right) \mathbb{E}[\|\theta_k-\theta^*\|^2] + \alpha_k^2\mathbb{E}[\|e^{t(\wR_k-\doublewidetilde{R}_k)} \nabla \wR_k(\zeta_k)\|^2] \\
    & \leq \left(1-\frac{2}{k}\right) \mathbb{E}[\|\theta_k-\theta^*\|^2] + \frac{4e^2 B^2 e^{2t(\widetilde{F}_{\max}-\widetilde{F}_{\min})}}{\mu^2 k^2}. \label{eq:one_step}
\end{align}
When $k \leq k_t$ with $\alpha_k=\frac{1}{1+tLBe^{t(\wF_{\max}-\wF_{\min})} }$,
\begin{align}
    \mathbb{E}[\|\theta_{k}-\theta^*\|^2] 
    & \leq \left(1-\frac{ \mu}{e(tLBe^{t(\widetilde{F}_{\max}-\widetilde{F}_{\min})}+1)}\right) \mathbb{E}[\|\theta_{k-1}-\theta^*\|^2] + \frac{B^2 e^{2t(\widetilde{F}_{\max}-\widetilde{F}_{\min})}}{(1+tLBe^{t(\wF_{\max}-\wF_{\min})})^2}.\\
\end{align}
We can thus prove
\begin{align}
    \mathbb{E}[\|\theta_{k_t}-\theta^*\|^2] \leq \max\left\{\mathbb{E}[\|\theta_1-\theta^*\|^2], \frac{B^2e^{2t(\widetilde{F}_{\max}-\widetilde{F}_{\min})+1}}{\mu (1+tLBe^{t(\widetilde{F}_{\max}-\widetilde{F}_{\min})})}\right\}
\end{align}
Let 
\begin{align}
    V_t = \max\left\{k_t\mathbb{E}[\|\theta_{k_t}-\theta^*\|^2], \frac{4B^2e^{2+2t(\widetilde{F}_{\max}-\widetilde{F}_{\min})}}{\mu^2}\right\}.
\end{align}
We next prove for $k \geq k_t$, 
\begin{align}
    \mathbb{E}[\|\theta_k-\theta^*\|^2] \leq \frac{V_t}{k}.
\end{align}
Suppose $\mathbb{E}[\|\theta_k-\theta^*\|^2] \leq \frac{V_t}{k}$.  From \eqref{eq:one_step}, we have
\begin{align}
    \mathbb{E}[\|\theta_{k+1}-\theta^*\|^2] &\leq \left(1-\frac{2}{k}\right) \mathbb{E}[\|\theta_k-\theta^*\|^2] + \frac{4e^2B^2 e^{2t(\widetilde{F}_{\max}-\widetilde{F}_{\min})}}{k^2\mu^2} \\
    &\leq \left(1-\frac{2}{k}\right) \frac{V_t}{k} + \frac{V_t^2}{k^2} \\
    & \leq \frac{V_t}{k+1},
\end{align}
where $k \geq k_t = \left\lceil \frac{e+\sqrt{e^2+\mu tLBe^{t(\widetilde{F}_{\max}-\widetilde{F}_{\min})}+1}}{\mu}\right\rceil$. This completes the proof.\hfill \qedsymbol

\paragraph{Proof of Theorem~\ref{thm: convergence_stochastic_NCSM}.}
Assume $\wR(t;\theta)$ is non-convex and $\beta$-smooth, we have
\begin{align}
    \wR_{k+1} - \wR_k - \langle \nabla \wR_k, \theta_{k+1}-\theta_k \rangle \leq \frac{\beta}{2} \|\theta_{k+1}-\theta_k\|^2,
\end{align}
where $\wR_k := \wR(t;\theta_k)  = \frac{1}{t} \log\left(\frac{1}{N} \sum_{i \in [N]} e^{tf(x_i; \theta_k)}\right)$. Plugging in the updating rule
\begin{align}
    \theta_{k+1}-\theta_k = -\alpha_k \frac{e^{t(\zeta_k; \theta_k)}}{e^{t\doublewidetilde{R}_k}} \nabla f(\zeta_k; \theta_k) = -\alpha_k e^{t(\wR_k-\doublewidetilde{R}_k)} \nabla \wR_k(\zeta_k)
\end{align}
gives
\begin{align}
    \wR_{k+1}-\wR_k + \alpha_k \langle \nabla \wR_k, e^{t(\wR_k-\doublewidetilde{R}_k)} \nabla \wR_k(\zeta_k) \rangle \leq \frac{\beta}{2} \left\|\alpha_k e^{t(\wR_k-\doublewidetilde{R}_k)} \nabla \wR_k(\zeta_k) \right\|^2. \label{eq:smooth}
\end{align}
First, we note
\begin{align}
    \left\|\alpha_k^2 e^{t(\wR_k-\doublewidetilde{R}_k)} \nabla \wR_k(\zeta_k)\right\|^2 \leq \alpha_k^2 e^{2t(\wF_{\max}-\wF_{\min})} \|\nabla \wR_k(\zeta_k)\|^2.
\end{align}
Take expectation on both sides of~\eqref{eq:smooth},
\begin{align}
    \mathbb{E}[\wR_{k+1}] - \mathbb{E}[\wR_k] + \alpha_k \mathbb{E}[\langle \nabla \wR_k, e^{t(\wR_k-\doublewidetilde{R}_k)} \nabla \wR_k (\zeta_k) \rangle] \leq \frac{\beta \alpha_k^2 e^{2t(\wF_{\max}-\wF_{\min})} B^2}{2}. \label{eq:one_step_2}
\end{align}
Let
\begin{align}
    k_t := \left\lceil \frac{2(\wF_{\max}-\wF_{\min})t^2L^2}{\beta e^2} \right\rceil.
\end{align}
For any $k \geq k_t$, let 
\begin{align}
    \alpha_k = \frac{\sqrt{2(\wF_{\max}-\wF_{\min})}}{e^{t(\wF_{\max}-\wF_{\min})}\sqrt{\beta B^2K}}.
\end{align}
For $k < k_t$, let 
\begin{align}
    \alpha_k = \frac{1}{tLBe^{t(\wF_{\max}-\wF_{\min})}+1}.
\end{align}
We have for any $k \geq 1$,
\begin{align}
    \alpha_k tLBe^{t(\wF_{\max}-\wF_{\min})} \leq 1.
\end{align}
Therefore, for any $k \geq 1$,
\begin{align}
    \mathbb{E}[e^{t(\doublewidetilde{R}_k-\wR_k)} | \theta_1,\ldots, \theta_k] \leq 2e.
\end{align}
Thus, for any $k\geq 1$,
\begin{align}
    \mathbb{E}[\langle \nabla \wR_k, e^{t(\wR_k-\doublewidetilde{R}_k)} \nabla \wR_k (\zeta_k) \rangle] &=  \mathbb{E}[ \mathbb{E}[\langle \nabla \wR_k, e^{t(\wR_k-\doublewidetilde{R}_k)} \nabla \wR_k (\zeta_k) \rangle|\theta_1, \ldots, \theta_k]] \\
    &= \mathbb{E}[\langle \nabla \wR_k, \mathbb{E} [e^{t(\wR_k-\doublewidetilde{R}_k)} \nabla \wR_k (\zeta_k) |\theta_1, \ldots, \theta_k] \rangle]  \\
    &= \mathbb{E}[\langle \nabla \wR_k, \mathbb{E} [e^{t(\wR_k-\doublewidetilde{R}_k)} |\theta_1, \ldots, \theta_k] \mathbb{E}[\nabla \wR_k (\zeta_k) |\theta_1, \ldots, \theta_k] \rangle] \\
    &= \mathbb{E}[\langle \nabla \wR_k, \mathbb{E} [e^{t(\wR_k-\doublewidetilde{R}_k)} |\theta_1, \ldots, \theta_k] \nabla \wR_k \rangle] \\
    & \geq \frac{1}{2e} \mathbb{E}[\|\nabla \wR_k\|^2]. \label{eq:tilted_g}
\end{align}
Plug \eqref{eq:tilted_g} into~\eqref{eq:one_step_2},
\begin{align}
     \mathbb{E}[\|\nabla \wR_k\|^2] + \frac{2e}{\alpha_k}(\mathbb{E}[\wR_{k+1}]-\mathbb{E}[\wR_k]) \leq \beta \alpha_k e^{2t(\wF_{\max}-\wF_{\min})} eB^2.
\end{align}
Apply telescope sum from $k_t+1$ to $K$ and divide both sides by $K$,
\begin{align}
    \frac{1}{K}\sum_{k=k_t}^K  \mathbb{E}[\|\nabla \wR_k\|^2] + \frac{2e(\mathbb{E}[\wR_{K+1}]-\mathbb{E}[\wR_{k_t}])}{\alpha_k K} \leq \beta \alpha_k e^{2t(\wF_{\max}-\wF_{\min})} eB^2.
\end{align}
\begin{align}
    \frac{1}{K}\sum_{k=k_t}^K  \mathbb{E}[\|\nabla \wR_k\|^2] &\leq \beta \alpha_k e^{2t(\wF_{\max}-\wF_{\min})} eB^2 +\frac{2e(\mathbb{E}[\wR_{k_t}-\wR_{K+1}])}{\alpha_k K} \\
    &\leq \beta \alpha_k e^{2t(\wF_{\max}-\wF_{\min})} eB^2 +\frac{2e(\wF_{\max}-\wF_{\min})}{\alpha_k K} 
\end{align}
Consider that $\alpha_k=\frac{\sqrt{2 (\wF_{\max}-\wF_{\min}) }}{e^{t(\wF_{\max}-\wF_{\min})}\sqrt{\beta B^2K}}$, 
\begin{align}
    \frac{1}{K}\sum_{k=k_t}^K  \mathbb{E}[\|\nabla \wR_k\|^2] \leq \sqrt{8}B e^{t(\wF_{\max}-\wF_{\min})+1} \sqrt{\frac{\beta (\wF_{\max}-\wF_{\min})}{K}},
\end{align}
completing the proof. \hfill \qedsymbol

\paragraph{Proof of Theorem~\ref{thm: convergence_stochastic_NCSM_PL}.}
From the assumptions, we have $\wR(t;\theta)$ is $\frac{\mu}{2}$-PL, i.e.,
\begin{align}
    \mu (\wR(t;\theta)-\wR^*) \leq \|\nabla \wR(t;\theta)\|^2,
\end{align}
where $\wR^* := \wR(t;\breve{\theta}(t))$.
Let
\begin{align}
    k_t := \arg\max_k \left(k < \frac{4e}{\mu}+\frac{4etLBe^{t(\wF_{\max}-\wF_{\min})}}{\mu k}\right),
\end{align}
and
\begin{align}
\alpha_k = 
    \begin{cases}
    \frac{1}{tLB e^{t(\widetilde{F}_{\max}-\widetilde{F}_{\min})}+1},& \text{if } k \leq k_t\\
    \frac{4e}{\mu k},              & \text{otherwise}.
\end{cases}
\end{align}
Similarly, we can prove for any $k \geq 1$, 
\begin{align}
\mathbb{E}[e^{t(\doublewidetilde{R}_k-\wR_k)} | \theta_1, \ldots, \theta_k] \leq 2e.
\end{align}
Similarly,
\begin{align}
     \mathbb{E}[\wR_{k+1}]-\mathbb{E}[\wR_k] + \frac{\alpha_k}{2e} \mathbb{E}[\|\nabla \wR_k\|^2] \leq \frac{\beta \alpha_k^2 e^{2t(\wF_{\max}-\wF_{\min})}B^2}{2}.
\end{align}
Therefore,
\begin{align}
     \mathbb{E}[\wR_{k+1}]-\mathbb{E}[\wR_k] + \frac{\alpha_k}{2e} \mu \mathbb{E}[\wR_k-\wR^*] &\leq \frac{\beta \alpha_k^2 e^{2t(\wF_{\max}-\wF_{\min})}B^2}{2} \\
     \mathbb{E}[\wR_{k+1}-\wR^*]-\mathbb{E}[\wR_k-\wR^*] + \frac{\alpha_k}{2e} \mu  \mathbb{E}[\wR_k-\wR^*] &\leq \frac{\beta \alpha_k^2 e^{2t(\wF_{\max}-\wF_{\min})}B^2}{2} \\
     \mathbb{E}[\wR_{k+1}-\wR^*] &\leq \left(1- \frac{\alpha_k}{2e} \mu\right)  \mathbb{E}[\wR_k-\wR^*] +\frac{\beta \alpha_k^2 e^{2t(\wF_{\max}-\wF_{\min})}B^2}{2}
\end{align}
Let $\alpha_k = \frac{4e}{\mu k}$, and
\begin{align}
    V_t=\max\left\{k_t \mathbb{E}[\wR_{k_t}-\wR^*], \frac{8\beta B^2 e^{2t(\wF_{\max}-\wF_{\min})+2}}{\mu^2}\right\}.
\end{align}
We next prove $\mathbb{E}[\wR_k-\wR^*] \leq \frac{1}{k}$ ($k \geq k_t$) by induction. Suppose $\mathbb{E}[\wR_k-\wR^*]\leq \frac{V_t}{k}$, then
\begin{align}
    \mathbb{E}[\wR_{k+1}-\wR^*] &\leq \left(1-\frac{2}{k}\right) \mathbb{E}[\wR_k-\wR^*] + \frac{V_t}{k^2} \\
    &\leq \left(1-\frac{2}{k}\right)\frac{V_t}{k} + \frac{V_t}{k^2} \\
    &\leq \frac{V_t}{k+1},
\end{align}
which concludes the proof. \hfill \qedsymbol

\begin{figure}[h]
    \centering
    \begin{subfigure}[b]{0.95\textwidth}
        \centering
        \includegraphics[width=0.45\linewidth]{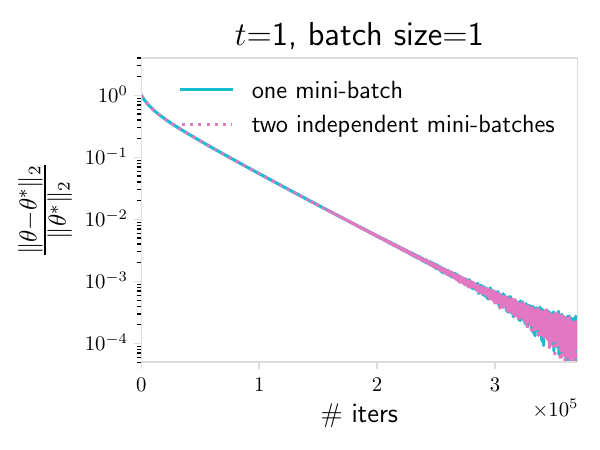}
        \hfill
        \includegraphics[width=0.45\linewidth]{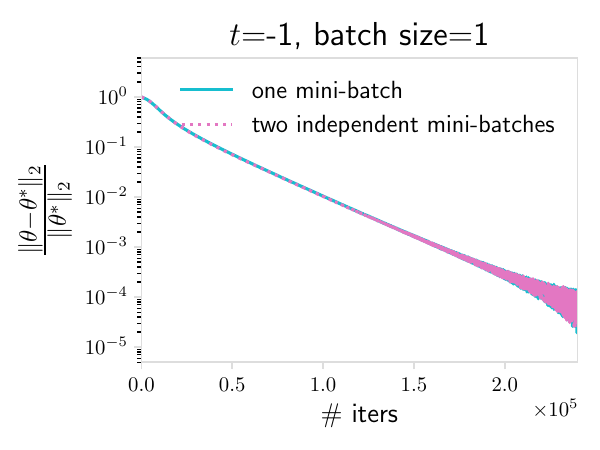}
    \end{subfigure}
    \caption{Convergence of Algorithm~\ref{alg:stochastic-non-h-TERM} using two independent mini-batches to update $\doublewidetilde{R}_{t}$ and calculate $e^{tf(x;\theta)} \nabla_{\theta}f(x;\theta)$ and a simpler variant using only one mini-batch to query $w_{t,x} \nabla_{\theta}f(x;\theta)$. We plot the optimality gap versus the number of iterations on the point estimation example (Figure~\ref{fig:example} (a)) with batch size being 1. While Algorithm~\ref{alg:stochastic-non-h-TERM} allows us to get better convergence guarantees theoretically, we find that these two variants perform similarly empirically.}
    \label{fig:convergence_compare}
\end{figure}

\begin{table}[h!]
\centering
\caption{TERM Applications and their corresponding solvers.}
\scalebox{1}{
    \begin{tabular}{l|l}
      Three toy examples (Figure~\ref{fig:example})   & ~~Algorithm~\ref{alg:batch-non-h-TERM} \\
       Robust regression (Table~\ref{table: outlier})  & ~~Algorithm~\ref{alg:batch-non-h-TERM} \\ 
       Robust classification (Table~\ref{table:label_noise_cnn})  & ~~Algorithm~\ref{alg:stochastic-non-h-TERM} \\ 
       Low-quality annotators (Figure~\ref{fig:noisy_annotator})~~  & ~~Algorithm~\ref{alg:stochastic-TERM} ($\tau=0$) \\ 
       Fair PCA (Figure~\ref{fig:pca}) & ~~Algorithm~\ref{alg:batch-TERM} ($\tau=0$) \\
       Class imbalance (Figure~\ref{fig:class_imabalance}) & ~~Algorithm~\ref{alg:stochastic-TERM} ($\tau=0$) \\
       Variance reduction (Table~\ref{table: dro}) & ~~Algorithm~\ref{alg:batch-TERM} ($\tau=0$) \\
       Hierarchical TERM (Table~\ref{table:hierarchical}) & ~~Algorithm~\ref{alg:batch-TERM} \\
    \end{tabular}}
    \label{table: alg_choice}
\vspace{0.15in}
\end{table}

\clearpage
\section{Additional Experiments and Experimental Details}
\label{appen:exp_full}
In Appendix~\ref{appen:complete_exp}, we provide complete experimental results on the properties or 
the use-cases of TERM. Details on how the experiments in Section~\ref{sec:experiments} were executed are provided in Appendix~\ref{appen:exp_detail}.



\subsection{Complete Results}\label{appen:complete_exp}

Recall that in Section~\ref{sec:term_properties},  Interpretation 1 is that TERM can be tuned to re-weight samples to magnify or suppress the influence of outliers. In Figure~\ref{fig:highlight_samples} below, we visually show this effect by highlighting the samples with the largest weight for $t\to +\infty$ and $t \to -\infty$ on the logistic regression example previously described in Figure~\ref{fig:example}. 

\begin{figure}[h!]
    \centering
    \begin{subfigure}{0.49\textwidth}
        \centering
        \includegraphics[width=0.82\textwidth]{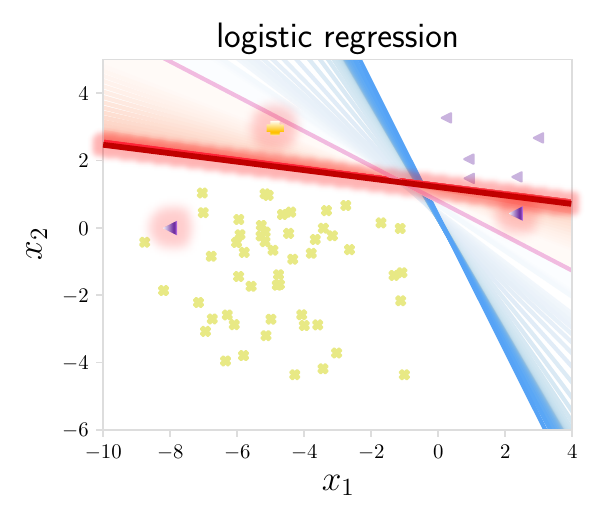}
        \caption{Samples with the largest weights as $t \to + \infty$.}
    \end{subfigure}
    \hfill
    \begin{subfigure}{0.49\textwidth}
        \centering
        \includegraphics[width=0.82\textwidth]{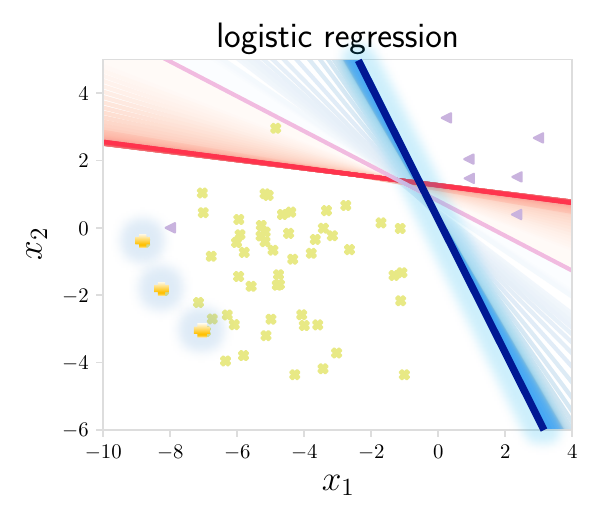}
        \caption{Samples with the largest weights as $t \to - \infty$.
        }
    \end{subfigure}
    \caption{For positive values of $t$, TERM focuses on the samples with relatively large losses (rare instances). When $t \to +\infty$ (left), a few misclassified samples have the largest weights and are highlighted. On the other hand, for negative values of $t$, TERM suppresses the effect of the outliers, and as $t \to -\infty$ (right), samples with the smallest losses hold the the largest weights. }
    \label{fig:highlight_samples}
\end{figure}




Next, we provide complete results of applying TERM to a diverse set of applications.

\paragraph{Robust classification.} Recall that in Section~\ref{sec:exp:robustness}, for classification in the presence of label noise, we only compare with baselines which do not require clean validation data. In Table~\ref{table:label_noise_cnn_full} below, we report the complete results of comparing TERM with all baselines, including MentorNet-DD~\citep{jiang2018mentornet} which needs additional clean data. In particular, in contrast to the other methods, MentorNet-DD uses 5,000 clean validation images. TERM is competitive with  the performance of MentorNet-DD, even though it does not have access to this clean data.

\begin{table}[h]
\caption{A complete comparison including two MentorNet variants. TERM is able to match the performance of MentorNet-DD, which needs additional clean labels.
}
\centering
\label{table:label_noise_cnn_full}
\scalebox{0.9}{
\begin{tabular}{lccc}
	   \toprule[\heavyrulewidth]
        \multicolumn{1}{l}{\multirow{2}{*}{\textbf{objectives}}} & \multicolumn{3}{c}{\textbf{test accuracy} ({\fontfamily{qcr}\selectfont{CIFAR-10}}, Inception)} \\
        \cmidrule(r){2-4}
          &  20\% noise & 40\% noise & 80\% noise\\
        \midrule
 \rule{0pt}{2ex}ERM & 0.775 {\tiny (.004)} &  0.719 {\tiny (.004)} & 0.284 {\tiny (.004)}  \\
 RandomRect~\citep{ren2018learning} & 0.744 {\tiny (.004)} & 0.699 {\tiny (.005)} & 0.384 {\tiny (.005)}  \\
SelfPaced~\citep{kumar2010self} & 0.784 {\tiny (.004)} & 0.733 {\tiny (.004)} & 0.272 {\tiny (.004)}  \\
 MentorNet-PD~\citep{jiang2018mentornet} & 0.798 {\tiny (.004)} & 0.731 {\tiny (.004)} & 0.312 {\tiny (.005)} \\
 GCE~\citep{zhang2018generalized} & \textbf{0.805} {\tiny (.004)}  & 0.750 {\tiny (.004)} & 0.433 {\tiny (.005)} \\
 MentorNet-DD~\citep{jiang2018mentornet} & \textbf{0.800} {\tiny (.004)} & \textbf{0.763} {\tiny (.004)} & \textbf{0.461}{\tiny (.005)} \\
 \rowcolor{myblue}
 TERM & 0.795 {\tiny (.004)}  & \textbf{0.768} {\tiny (.004)} & \textbf{0.455} {\tiny (.005)} \\
 \hline 
 \rule{0pt}{2ex}Genie ERM & 0.828 {\tiny (.004)} & 0.820 {\tiny (.004)} & 0.792 {\tiny (.004)} \\
\bottomrule[\heavyrulewidth]
\end{tabular}}
\end{table}

To interpret the noise more easily, we provide a toy logistic regression example with synthetic data here. In Figure~\ref{fig:robust_classification_synthetic}, we see that TERM with $t=-2$ (blue) can converge to the correct classifier under 20\%, 40\%, and 80\% noise.

\begin{figure}[h]
    \centering
    \includegraphics[width=\textwidth]{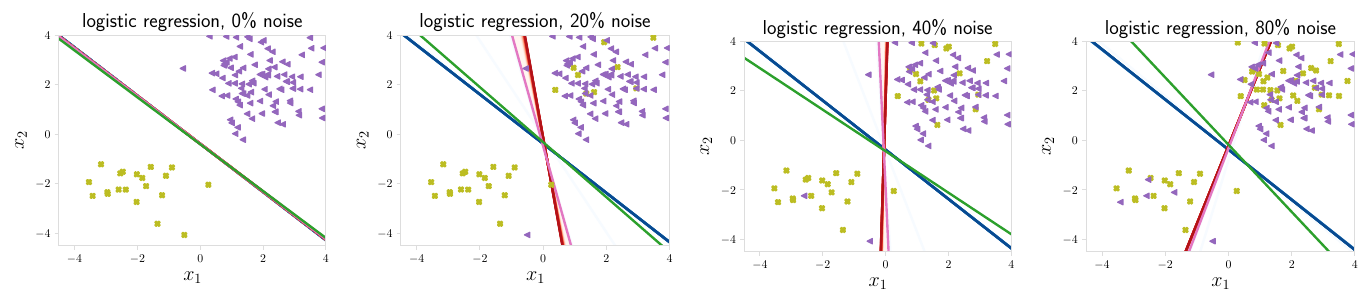}
    \caption{{Robust classification using synthetic data. On this toy problem, we show that TERM with negative $t$'s (blue) can be robust to random noisy samples. The green line corresponds to the solution of the generalized cross entropy (GCE) baseline~\citep{zhang2018generalized}. Note that on this toy problem, GCE is as good as TERM with negative $t$'s, despite its inferior performance on the real-world CIFAR10 dataset.}}
    \label{fig:robust_classification_synthetic}
\end{figure}


\begin{figure}[h]
\centering
  \includegraphics[width=0.5\linewidth]{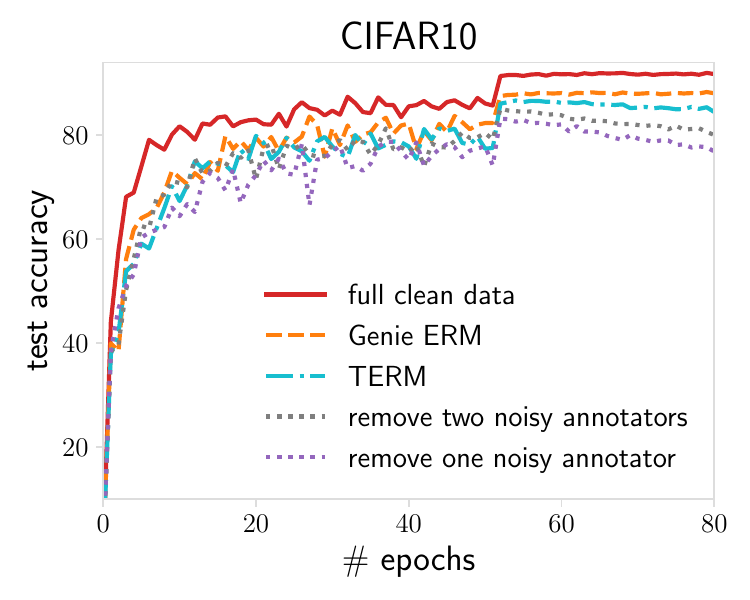}
  \caption{TERM achieves higher test accuracy  than the baselines, and can match the performance of Genie ERM (i.e., training on all the clean data combined).}
  \label{fig:noisy_annotator_2}
\end{figure}

\newpage
\paragraph{Low-quality annotators.} In Section~\ref{sec:exp:noisy_annotator}, we demonstrate that TERM can be used to mitigate the effect of noisy annotators, and we assume each annotator is either always correct, or always uniformly assigning random labels. 
Here, we explore a different and possibly more practical scenario where there are four noisy annotators who corrupt 0\%, 20\%, 40\%, and 100\% of their data by assigning labels uniformly at random, and there is one additional adversarial annotator who always assigns wrong labels. We assume the data points labeled by each annotator do not overlap, since~\citep{khetan2017learning} show that obtaining one label per sample is optimal for the data collectors under a fixed annotation budget. We compare TERM with several baselines: (a) training without the data coming from the adversarial annotator, (b) training without the data coming from the worst two annotators, and (c) training with all the clean data combined (Genie ERM). The results are shown in Figure~\ref{fig:noisy_annotator_2}. We see that TERM outperforms the strong baselines of removing one or two noisy annotators, and closely matches the performance of training with all the available clean data.

\subsection{Experimental Details} \label{appen:exp_detail}
We first describe the datasets and models used in each experiment presented in Section~\ref{sec:experiments}, and then provide a detailed setup including the choices of hyperparameters. All code and datasets are publicly available at \href{https://github.com/litian96/TERM}{github.com/litian96/TERM}.

\subsubsection{Datasets and models}


In Section~\ref{sec:exp:robustness}, for regression tasks, we use the drug discovery data extracted from~\citet{diakonikolas2019sever} which is originally curated from~\citet{olier2018meta} and train linear regression models with different losses. 
There are 4,085 samples in total with each having 411 features. We randomly split the dataset into 80\% training set, 10\% validation set, and 10\% testing set.
For mitigating noise on classification tasks, we use the standard CIFAR-10 data and their standard train/val/test partitions along with a standard inception network~\citep{szegedy2016rethinking}. For experiments regarding mitigating noisy annotators, we again use the CIFAR-10 data and their standard partitions with a ResNet20 model. The noise generation procedure is described in Section~\ref{sec:exp:noisy_annotator}. 

In Section~\ref{sec:exp:fairness}, for fair PCA experiments, we use the complete Default Credit data to learn low-dimensional approximations and the loss is computed on the full training set. We follow the exact data processing steps described in the work~\citep{samadi2018price} we compare with. There are 30,000 total data points with 21-dimensional features (after preprocessing). Among them, the high education group has 24,629 samples and the low education group has 5,371 samples. For meta-learning experiments, one the popular sine wave regression problem~\citep{finn2017model}, we generate 5,000 meta-training and 5,000 meta-testing tasks. Following~\citet{collins2020task}, there are 250 hard meta-training tasks with amplitudes drawn from $[4.95, 5]$ and 4,750 easy meta-training tasks with amplitudes drawn from $[0.01,1]$. The amplitudes of meta-testing tasks are drawn uniformly from $[0.1, 5]$. The phase values are drawn uniformly from $[0,\pi]$ for all tasks.
For class imbalance experiments, we directly take the unbalanced data extracted from MNIST~\citep{lecun1998gradient} used in~\citet{ren2018learning}.
When demonstrating the variance reduction of TERM, we use the HIV-1 dataset~\citep{hiv1} as in~\citep{duchi2019variance} and randomly split it into 80\% train, 10\% validation, and 10\% test set. There are 6,590 total samples and each has 160 features. We report results based on five such random partitions of the data. We train logistic regression models (without any regularization) for this binary classification task for TERM and the baseline methods. We also investigate the performance of a linear SVM. 

In Section~\ref{sec:exp:multi_obj}, the HIV-1 data are the same as that in Section~\ref{sec:exp:fairness}. We also manually subsample the data to make it more imbalanced, or inject random noise, as described in Section~\ref{sec:exp:multi_obj}. The CIFAR10 dataset used in this section is a standard benchmark, and we follow the same procedures in~\citet{cao2021heteroskedastic} to generate a noisy and imbalanced variant.

\subsubsection{Hyperparameters}

\paragraph{Selecting $t$.} In  Section~\ref{sec:exp:fairness} where we consider positive $t$'s, we select $t$ from a limited candidate set of $\{0.1, 1, 2, 5, 10, 50, 100, 200\}$ on the held-out validation set. 
For initial robust regression experiments, RMSE changed by only
0.08 on average across t; we thus used $t=-2$ for all experiments
involving noisy training samples (Section~\ref{sec:exp:robustness} and Section~\ref{sec:exp:multi_obj}).

\paragraph{Other parameters.} 
For all experiments, we tune all other hyperparameters (the learning rates, the regularization parameters, the decision threshold for ERM$_{+}$, $\rho$ for~\citep{duchi2019variance}, {the quantile value for CVaR (i.e., $\alpha$ in Eq.~\eqref{eq: cvar})~\citep{rockafellar2000optimization}}, $\alpha$ and $\gamma$ for focal loss~\citep{lin2017focal}) based on a validation set, and select the best one.  For experiments regarding focal loss~\citep{lin2017focal}, we select the class balancing parameter ($\alpha$ in the original focal loss paper) from $\texttt{range}(0.05, 0.95, 0.05)$  and select the main parameter $\gamma$ from $\{0.5, 1, 2, 3, 4, 5\}$. We tune $\rho$ in~\citep{duchi2019variance} such that $\frac{\rho}{n}$ is selected from $\{0.5, 1, 2, 3, 4, 5, 10\}$ where $n$ is the training set size. {We tune $\alpha$ for CVaR from $\{0.1, 0.3, 0.5, 0.7, 0.9\}$.} All regularization parameters including regularization for linear SVM are selected from $\{0.0001, 0.01, 0.1, 1, 2\}$.
For all experiments on the baseline methods, we use the default hyperparameters in the original paper (or the open-sourced code).

We summarize a complete list of main hyperparameter values as follows.

\noindent
\textit{Section~\ref{sec:exp:robustness}:}
    \begin{itemize}[leftmargin=*]
        \item Robust regression. The threshold parameter $\delta$ for Huber loss for all noisy levels is 1, the corruption parameter $k$ for CRR is: 500 (20\% noise), 1000 (40\% noise), and 3000 (80\% noise); and  TERM uses $t=-2$.
        \item Robust classification. The results are all based on the default hyperparameters provided by the open-sourced code of MentorNet~\citep{jiang2018mentornet}, if applicable. We tune the $q$ parameter for generalized cross entropy (GCE) from $\{0.4, 0.8, 1.0\}$ and select a best one for each noise level. For TERM, we scale $t$ linearly as the number of iterations from 0 to -2 for all noise levels.
        \item Low-quality annotators. For all methods, we use the same set of hyperparameters. The initial step-size is set to 0.1 and decayed to 0.01 at epoch 50. The batch size is 100. 
    \end{itemize}
    \textit{Section~\ref{sec:exp:fairness}:}
    \begin{itemize}[leftmargin=*]
        \item Fair PCA. We use the default hyperparameters and directly run the public code of~\citep{samadi2018price} to get the results on the min-max fairness baseline. We use a learning rate of 0.001 for our gradient-based solver for all target dimensions. 
        \item Fair meta-learning. We use a fixed learning rate of 0.01 for all methods, and tune a best learning rate for the task weights for the work of~\citet{collins2020task}. Similar as~\citet{collins2020task}, for all methods, we run one step of mini-batch SGD for inner optimization.
        \item Handling class imbalance. We take the open-sourced code of LearnReweight~\citep{ren2018learning} and use the default hyperparameters for the baselines of LearnReweight, HardMine, and ERM. We implement focal loss, and select $\alpha=0.05, \gamma=2$.
        \item Variance reduction. The regularization parameter for linear SVM is 1. $\gamma$ for focal loss is 2. We perform binary search on the decision thresholds for ERM$_{+}$ and RobustRegRisk$_{+}$, and choose 0.26 and 0.49, respectively.
    \end{itemize}
    \textit{Section~\ref{sec:exp:multi_obj}:}
    \begin{itemize}[leftmargin=*]
        \item Logistic regression on HIV. We tune the $q$ parameter for GCE based on validation data. We use $q=0,0,0.7, 0.3$ respectively for the four scenarios we consider. For RobustlyRegRisk, we use $\frac{\rho}{n}=10$ (where $n$ is the training sample size) and we find that the performance is not sensitive to the choice of $\rho$. {For CVaR, the tuned $\alpha$ value is 0.5 when the data imbalance ratio is 1:4, and 0.1 when the imbalance ratio is 1:20.} For focal loss, we tune the hyperparameters for best performance and select $\gamma=2$, $\alpha=$0.5, 0.1, 0.5, and 0.2 for four scenarios. For HAR, we tune the regularization parameter $\lambda$ via grid search from $\{0.1, 1, 2, 5, 10\}$ and select the best one. We use $t=-2$ for TERM in the presence of noise, and tune the positive $t$'s based on validation data. In particular, the values of tilts under four cases are: (0, 0.1), (0, 50), (-2, 5), and (-2, 10) for TERM$_{sc}$ and (0.1, 0), (50, 0), (1, -2) and (50, -2) for TERM$_{ca}$.
        \item ResNet32 on CIFAR10. We reproduce (and then directly take) the results from~\citep{cao2021heteroskedastic} for all baseline methods. For hierarchical TERM, we scale $t$ from 0 to 3 for group-level tilting, and scale $t$ from 0 to -2 for sample-level tilting within each group. $\lambda$ is set to 0.2. We use the default hyperparameters (batch size, learning rate, etc) in the open-sourced code of HAR~\citep{cao2021heteroskedastic} for TERM.
    \end{itemize}

\clearpage
\bibliography{ref}

\end{document}